\documentclass[11pt]{article}
\usepackage[utf8]{inputenc}
\usepackage{amsmath}
\usepackage{amsfonts}
\usepackage{amssymb}
\usepackage{graphicx}
\usepackage{hyperref}
\hypersetup{hidelinks}

\usepackage{subcaption}
\usepackage{amsmath,amsfonts,amsthm,amssymb,
amsopn,bm,color,mathtools,booktabs,multirow,appendix,caption}
\usepackage{xcolor}
\usepackage[square, numbers, sort&compress]{natbib}
\usepackage{thmtools}
\usepackage{thm-restate}
\usepackage{algorithm}
\usepackage{algpseudocode}
\usepackage{placeins}

\usepackage{enumitem}
\usepackage{cleveref}
\crefrangeformat{assumption}{Assumptions~#3#1#4--#5#2#6}

\newtheorem{assumption}{Assumption}

\newtheorem{definition}{Definition}
\newtheorem{lemma}{Lemma}
\newtheorem{theorem}{Theorem}
\newtheorem{corollary}{Corollary}
\newtheorem{remark}{Remark}
\newtheorem{example}{Example}

\DeclareMathOperator*{\E}{\mathbb{E}}
\DeclareMathOperator{\softmax}{softmax}
\DeclareMathOperator{\diag}{diag}
\DeclareMathOperator{\logsumexp}{logsumexp}

\usepackage[T1]{fontenc}
\usepackage{url}
\usepackage{nicefrac}
\usepackage{microtype}
\usepackage{tikz}
\usetikzlibrary{shapes.geometric}

\usepackage[vmargin=1.00in,hmargin=1.00in,centering,letterpaper]{geometry}

\newcommand{\Risk}{\mathcal{R}}            % population (squared) risk
\newcommand{\Pop}{\mathcal{P}}             % full mixture distribution
\newcommand{\PopOne}{\mathcal{P}_1}        % minority sub‑population
\newcommand{\PopTwo}{\mathcal{P}_2}        % majority sub‑population
\newcommand{\thetastar}{\theta^{\star}}    % globally optimal linear predictor
\newcommand{\thetabar}[1]{\bar{\theta}_{#1}}% ridge specialist for group #1

\newcommand{\algcomment}[1]{\textcolor{blue!70!black}{\footnotesize{\texttt{\textbf{//
          #1}}}}}

\newcommand{\R}{\mathbb{R}}

\newcommand{\frakD}{\mathfrak{D}}

\newcommand{\mc}{\mathcal}

\newcommand{\cP}{\mc{P}}
\newcommand{\cF}{\mc{F}}
\newcommand{\cG}{\mc{G}}
\newcommand{\cD}{\mc{D}}

\newcommand{\Var}{\operatorname{Var}}

\newcommand{\fNP}{\mc{R}}

\newcommand{\xt}{z^t}
\newcommand{\tildeztj}{\tilde{z}^t_j}
\newcommand{\zt}{z^t}

\newcommand{\etat}{\eta^t}

\newcommand{\yAgg}{y_{\text{agg}, i}}

\newcommand{\Ymax}{Y_{\max}}
\newcommand{\Ymaxprobe}{Y_{\max,\mathrm{probe}}}
\newcommand{\Lhat}{\widehat{L}_i}                % empirical risk (pseudo-labels)
\newcommand{\Ltrue}{\widehat{L}_{i, \mathrm{true}}} % empirical risk (true labels)
\newcommand{\Bth}{B_\theta}                    % radius for uniform deviation
\newcommand{\DeltaN}{\Delta_i}                 % empirical pseudo-true gap
\newcommand{\Lamu}{\Lambda_{\mathrm{u}}}       % uniform deviation (over \|\theta\|\le B_\theta)
\newcommand{\Lamst}{\Lambda_{\star}}           % one-point deviation (at \theta^\star)
\newcommand{\Rad}{\operatorname{Rad}}          % Rademacher complexity/average

\begin{document}
\begin{center}

    {\bf{\Large{Dynamics of Learning under User Choice: Overspecialization and Peer-Model Probing}}}
    
    \vspace*{.2in}
    
    {\large{
    \begin{tabular}{cccc}
      Adhyyan Narang$^{\dagger}$ & Sarah Dean$^{\ddagger}$
      & Lillian J. Ratliff$^{\dagger}$
      & Maryam Fazel$^{\dagger}$ \\
    \end{tabular}}}
    
    \vspace*{.2in}
    \begin{tabular}{c}
        Electrical and Computer Engineering, University of Washington$^{\dagger}$ \\
        Computer Science, Cornell University $^{\ddagger}$ \\
        \end{tabular}
    
    \vspace*{.08in}
    {\small Correspondence: \texttt{adhyyan@uw.edu}}
    \end{center}
\begin{abstract}
    In many economically relevant contexts where machine learning is deployed,
    multiple platforms obtain data from the same pool of users,
    each of whom selects the platform that best serves them.
    Prior work in this setting focuses exclusively on the “local” losses of
    learners on the distribution of data that they observe.
    We find that there exist instances where learners who use
    existing algorithms almost surely converge to models with
    arbitrarily poor global performance,
    even when models with low full-population loss exist.
    This happens through a feedback-induced mechanism, which we call
    the overspecialization trap:
    as learners optimize for users who already prefer them,
    they become less attractive to users outside this base,
    which further restricts the data they observe.
    Inspired by the recent use of knowledge distillation in modern ML,
    we propose an algorithm that allows learners to
    "probe" the predictions of peer models,
    enabling them to learn about users who do not select them.
    Our analysis characterizes when probing succeeds: this procedure converges almost surely to a stationary point with bounded full-population risk when probing sources are sufficiently informative, e.g., a known market leader or a majority of peers with good global performance.
    We verify our findings with semi-synthetic experiments on the MovieLens, Census, and Amazon Sentiment datasets.\footnote{Code for this paper is available at: \url{https://github.com/AdhyyanNarang/overspecialization-probing}.}
\end{abstract}

\section{Introduction}

Traditional supervised learning theory typically assumes a single learner observing data
drawn from a fixed distribution.
However, this assumption is increasingly violated in modern machine learning markets,
such as recommendation platforms and large language model (LLM) services.
In these ecosystems, multiple learners operate on the same pool of users,
and data is not assigned randomly.
Instead, users choose which platform to engage with based on how well that platform serves their
specific needs or preferences.
Consequently, the data distribution observed by a learner is
a function of the learner’s own performance and the choices available in the market.
This setting is increasingly garnering interest in the machine
learning community \cite{dean2022multi,Su2024-qw, ginart2021competing,
bose2023initializing,shekhtman2024strategic}.

This coupling between model performance and user selection creates a
feedback loop.
As a learner optimizes for its current user base,
it becomes increasingly specialized to that subpopulation.
While this minimizes "local" loss on observed users, it
often degrades performance on the unobserved population,
a phenomenon we term \emph{overspecialization}.
Once a learner is overspecialized, it gets caught in an informational trap:
it cannot learn to serve new users because it never observes them,
and it never observes them because it cannot serve them.
At a societal level, this dynamic fuels the formation of
algorithmic echo chambers \cite{cinus2022effect,beardow2021algorithmic,jiang2021social,interian2022network},
where platforms fragment the population rather than learning a robust, globally capable model.

Independently, another trend has become relevant in modern machine learning systems that has implications
for the overspecialization problem: techniques such as knowledge distillation and training on
synthetic data are becoming ubiquitous, particularly in the
training of Large Language Models \cite{hinton2015distilling, werner2025deepseek}.
While these methods are typically employed to improve reasoning capabilities or computational efficiency
(through compression of data), they introduce a structural change to the learning dynamic.
Models are no longer limited to learning from organic user data,
but can also "probe" other models to acquire synthetic labels.
This enables learners to observe signals
outside their siloed data distributions.
In this work, we study whether the use of probing mechanisms in
these machine learning markets can mitigate overspecialization.

\vspace{0.5em}

\noindent \textbf{Our Contributions.} We model a market where users select learners based on a combination of inherent preferences and predictive loss. We analyze the resulting dynamics through a game-theoretic lens to understand the impact of peer probing. Our main findings are as follows:

\begin{enumerate} \item \textbf{The Failure of Standard Learning:}
    We first analyze standard Multi-learner Streaming Gradient Descent (MSGD) in the absence of probing \citep{Su2024-qw}.
    We prove that due to the user selection mechanism,
    MSGD can converge to "bad" stationary points.
    In these equilibria,
    learners become overspecialized, achieving low loss on their niche but arbitrarily poor performance on the global population.
\item \textbf{Convergence of Peer Probing:} We propose a new algorithm,
\textbf{MSGD with Probing (MSGD-P)},
where learners mix gradient updates from organic users with updates
from pseudo-labeled queries sent to peer models.
We prove that this multi-agent dynamic converges to a stationary
point of a modified potential function (\Cref{thm:msgd_probe_convergence}).
To our knowledge, we are the first to observe and study the multi-agent dynamics that arise
from synthetic data training.

\item \textbf{Restoring Global Competence:} We show that probing effectively mitigates overspecialization.
By characterizing the stationary points of MSGD-P,
we derive bounds on the full-population loss under appropriate
informational conditions (e.g., probing a known market leader vs. aggregating diverse peers).

\item \textbf{Empirical Validation:} We validate our findings on the MovieLens, US Census and Amazon Sentiment datasets.
We observe that while standard learning leaves some models
trapped with poor accuracy,
introducing peer probing closes this performance gap.
\end{enumerate}

\begin{figure}[t]
    \centering
    \includegraphics[width=0.5\textwidth]{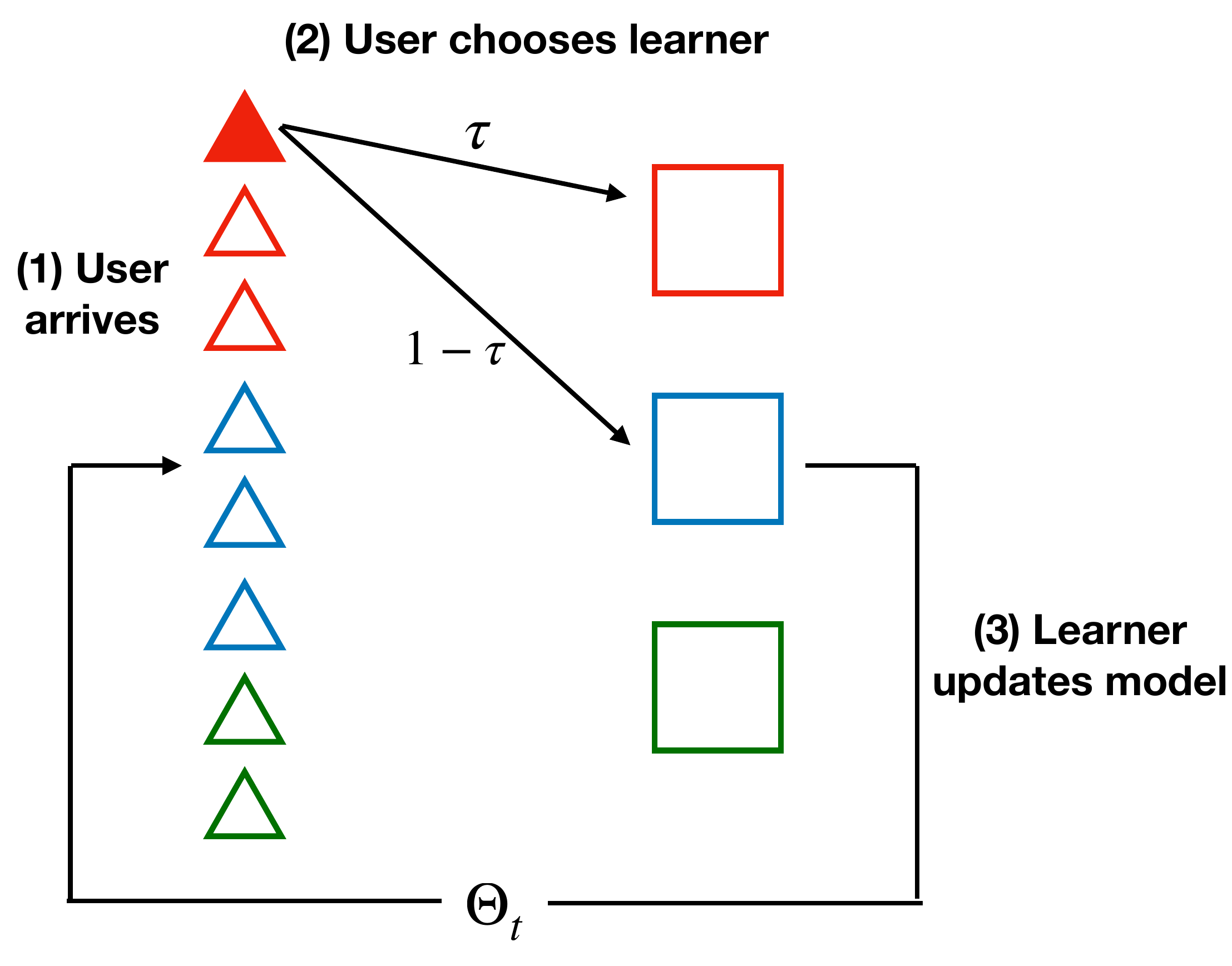}
    \caption{Illustration of our online multi-learner problem setting.
    The borders of users represent their highest ranked learner $\pi(z)$.
    For further details, see \Cref{sec:setup}.}
    \label{fig:setup}
\end{figure}
\section{Related Work}

Our work sits at the intersection of three lines of research.
We study a \emph{multi-learner} setting where users select among
competing platforms based on a combination of
\emph{inherent preferences and predictive quality}: a
model richer than pure loss-minimization or uniform-random selection.
Motivated by modern distillation practices, we analyze
\emph{peer-model probing} as a mechanism to mitigate overspecialization.

\paragraph{Performative Prediction.}
Our setting is an instance of performative prediction \citep{perdomo2020performative,
hardt2016strategic, miller2021accuracy}, where deploying a model influences
the data a learner subsequently observes.
Multi-learner extensions
\citep{piliouras2023multiagent, narang2023multiplayer, li2022multiagent,
wang2023network, zhu2023online, gois2025performative, wang2025lastiterate}
provide general tools for endogenous distribution shift.
We specialize to user-choice, where shift arises from selection
rather than manipulation, enabling precise characterization
of overspecialization and peer-model probing.

\paragraph{Learning under User Agency.}
Prior work on user agency focuses on single-learner settings
where users opt out or selectively provide data
\cite{hashimoto2018fairness, zhang2019group, james2023participatory,
ben2017best, Cherapanamjeri2023-Fisherman, Harris2023-mm, raab2024fair};
these analyses do not capture the inter-learner feedback dynamics
that arise in our multi-learner setting.
In the multi-learner setting,
\citet{ginart2021competing} characterize the \emph{existence} of
performance gaps due to competition via batch retraining;
we prove that streaming dynamics \emph{converge to}
overspecialized equilibria (Theorem~\ref{thm:bad_outcome_simple})
and propose probing as a mitigation.
Most closely related, \citet{dean2022multi} and \citet{Su2024-qw}
study gradient-based dynamics in choice-driven settings;
we extend their framework to analyze full-population risk
and introduce peer probing.
See Appendix~\ref{sec:extended_related_work} for detailed comparisons
with \citet{ginart2021competing}, \citet{kwon2022competition},
\citet{bose2023initializing}, and \citet{Su2024-qw}.

\paragraph{Knowledge Distillation.}
Our probing mechanism draws on knowledge distillation
\citep{hinton2015distilling} and self-training
\citep{scudder1965probability, yarowsky1995unsupervised};
unlike Deep Mutual Learning \citep{zhang2018deep}
and codistillation \citep{anil2018large},
which assume shared data,
our setting couples distillation with user-driven selection,
yielding novel multi-agent dynamics.
This also distinguishes our setting from domain generalization methods,
which train a centralized model from labeled source domains, and
selective-label/sample-selection work, which typically studies a single
learner under an exogenous censoring rule; here each learner's observed
distribution is endogenously determined by user choices among competing
models.

\section{Problem Setting}
\label{sec:setup}

We now formalize the machine learning market described above.

\subsection{Users and Learners}

Consider a market with $m$ service providers (learners) serving a population of users distributed according to $\mathcal{P}$ over $\mathcal{Z} = \mathcal{X} \times \mathcal{Y}$, where $\mathcal{X} \subseteq \mathbb{R}^d$ denotes covariates and $\mathcal{Y}$ denotes labels ($\mathcal{Y} \subseteq \mathbb{R}$ for regression, $\mathcal{Y} \subseteq \{1, \ldots, C\}$ for classification). We write $\mathcal{P}_X$ for the marginal distribution over covariates and $\mathcal{P}_{Y \mid X}(\cdot \mid x)$ for the conditional label distribution. When densities exist, we denote them by $p_X$ and $p_{Y \mid X}$.

Each learner $i \in [m] := \{1, \ldots, m\}$ maintains a model with parameters $\theta_i$, and we write $\Theta = (\theta_1, \ldots, \theta_m)$ for the joint parameter vector. The loss $\ell(x, y; \theta)$ measures the cost incurred by a model with parameters $\theta$ on a user with features $x$ and label $y$. For regression, we consider linear predictors $h_\theta(x) = x^\top \theta$ with squared loss
\[
\ell_{\mathrm{SQ}}(x, y; \theta) = (y - x^\top \theta)^2.
\]
For classification with $C$ classes, we use cross-entropy loss
\[
\ell_{\mathrm{CE}}(x, y; \theta) = -\sum_{c=1}^{C} y_c \log h_\theta(x)_c,
\]
where $h_\theta(x)_c = e^{x^\top \theta_c} / \sum_{j=1}^{C} e^{x^\top \theta_j}$ is the softmax output. Our framework can handle nonlinear relationships as well, as long as a suitable nonlinear feature transformation is known i.e if $h_\theta(x) = \phi(x)^\top \theta$.

\begin{assumption}
    \label{ass:bounded_support}
    The distribution $\mathcal{P}$ has continuous density $p_Z(z) = p_X(x) p_{Y|X}(y|x)$ with $p_X$ supported on $\{x : \|x\| \leq R\}$ for some $R > 0$. For regression, $p_{Y|X}$ is supported on $[-Y_{\max}, Y_{\max}]$.
\end{assumption}

Under this assumption, both loss functions satisfy standard regularity conditions (see Appendix~\ref{app:loss_regularity}).

\begin{restatable}{lemma}{LemLossRegularity}
    \label{lem:loss_regularity}
    Under Assumption~\ref{ass:bounded_support}, for all $z \in \mathcal{Z}$, the loss $\ell(z, \cdot)$ is non-negative, convex, differentiable, locally Lipschitz, and $\beta_\ell$-smooth.
\end{restatable}

\subsection{User Preferences and Platform Choice}

A central feature of our framework is that users have
\emph{inherent preferences} over platforms that exist
independently of current model quality.
These preferences capture factors such as brand loyalty,
familiarity, network effects, or historical habits.
We encode them via a function
$\pi: \mathcal{Z} \to [m]$, where $\pi(z)$ denotes the platform that user
$z$ intrinsically prefers.
This function is exogenous and fixed throughout learning:
it represents pre-existing affinities that platforms cannot directly
influence through model updates.

The preference function $\pi$ induces a partition of the user space.
Let $S_i = \{z \in \mathcal{Z} : \pi(z) = i\}$ denote users who prefer
platform $i$, with $\mathcal{P}_i = \mathcal{P}|_{S_i}$
the corresponding conditional distribution and
$\alpha_i = \Pr_{z \sim \mathcal{P}}[\pi(z) = i]$
the fraction of users preferring platform $i$.
Users do not only follow their inherent preferences;
they also consider predictive quality. We model this tradeoff as follows:

\begin{definition}[User Selection Rule]
    \label{def:user_selection_rule}
    Given models $\Theta$, user $z$ selects platform
    \[
    M(z; \Theta) =
    \begin{cases}
    \pi(z) & \text{with probability } \tau, \\
    \arg\min_{i \in [m]} \ell(z; \theta_i) & \text{with probability } 1 - \tau.
    \end{cases}
    \]
\end{definition}
The parameter $\tau \in [0, 1]$ governs how strongly
inherent preferences influence user behavior.
When $\tau = 1$, users ignore model quality entirely and
follow their intrinsic preferences.
When $\tau = 0$, users select the platform minimizing their loss.
For intermediate values,
users often default to familiar platforms but sometimes
shop based on quality.
This model generalizes the setting of \citet{Su2024-qw},
who assume users either minimize loss or choose uniformly at random.
We make this choice because user preferences in practice
are neither purely quality-driven nor purely random:
the $\pi(z)$ function captures persistent individual affinities that
create systematic heterogeneity in platform selection.

Model quality also induces a partition of users.
Let $Z_i(\Theta) = \{z : i = \arg\min_{j \in [m]} \ell(z; \theta_j)\}$
denote users for whom platform $i$ achieves minimal loss,
with $\mathcal{D}_i(\Theta) = \mathcal{P}|_{Z_i(\Theta)}$
the conditional distribution and
$a_i(\Theta) = \Pr_{z \sim \mathcal{P}}[z \in Z_i(\Theta)]$ the population fraction.
Under Definition~\ref{def:user_selection_rule}, learner $i$ observes users from the mixture
\[
\mathcal{O}_i(\Theta) = \tau \alpha_i \mathcal{P}_i + (1 - \tau) a_i(\Theta) \mathcal{D}_i(\Theta),
\]
where $\mathcal{O}_i(\Theta)$ is a sub-probability measure with total mass
$w_i(\Theta) = \tau \alpha_i + (1 - \tau) a_i(\Theta)$.

\subsection{Dynamics and Learning Objective}

We consider an online setting where learners interact with users over $T$ timesteps,
illustrated in Figure~\ref{fig:setup}. At each step $t$:
\begin{enumerate}
    \item A user $z^t \sim \mathcal{P}$ arrives and selects platform $M(z^t; \Theta^t)$.
    \item The selected learner observes $z^t$, incurs loss $\ell(z^t; \theta_{M(z^t; \Theta^t)}^t)$, and updates its parameters.
\end{enumerate}

We denote an instance of this market as
$\mathcal{G} = (\mathcal{P}, \ell, \pi, \tau)$.

\vspace{0.5em}

\noindent\textit{Learning Objective.} Each learner's goal is to
minimize the \emph{full-population risk}
\[
\mathcal{R}(\theta) = \mathbb{E}_{z \sim \mathcal{P}}[\ell(z; \theta)].
\]
We write $\theta^\star = \arg\min_\theta \mathcal{R}(\theta)$ for the
population-optimal model and
$\epsilon = \mathcal{R}(\theta^\star)$ for the Bayes risk.
This objective differs from prior work \citep{Su2024-qw, dean2022multi},
which focuses on the ``local'' loss over each learner's observed
distribution $\mathcal{O}_i(\Theta)$.
We focus on full-population risk because the signature
of overspecialization is precisely the gap between local
and global performance: a learner may achieve
low loss on $\mathcal{O}_i(\Theta)$ while performing
poorly on users outside its observed base.
Understanding when this gap emerges
and how to prevent it is the central question of this paper.

\section{The Failure of Standard Learning Dynamics}
\label{sec:standard_learner_rules}

We now analyze Multi-learner Streaming Gradient Descent (MSGD),
the standard algorithm for this setting introduced by
\citet{Su2024-qw}.
First, we reproduce the result from \cite{Su2024-qw} that MSGD converges to stationary points
(Theorem~\ref{lem:msgd_convergence}); however,
we use a different proof technique that we believe is more insightful, which
we describe below.
Despite convergence, we show that these equilibria can
exhibit severe overspecialization:
learners achieve low loss on observed users while performing arbitrarily poorly on the full population (Theorem~\ref{thm:bad_outcome_simple}).

\subsection{Algorithm and Assumptions}

Algorithm~\ref{alg:msgd} presents MSGD,
studied in prior work by \citet{Su2024-qw}.
When a user arrives and selects learner $i$,
that learner performs a stochastic gradient update on the observed loss, and other
learners remain unchanged.

\begin{algorithm}[t]
    \caption{Multi-learner Streaming Gradient Descent (MSGD) \citep{Su2024-qw}}
    \begin{algorithmic}[1]
    \Require Loss function $\ell(\cdot, \cdot) \geq 0$; initial models $\Theta^0 = (\theta^0_1, \ldots, \theta^0_m)$;
    learning rate $\{\eta^t\}_{t \geq 1}$
    \For{$t = 0, 1, 2, \ldots, T$}
        \State Sample user $z^t \sim \mathcal{P}$
        \State User selects learner $i = M(z^t; \Theta^t)$
        \State $\theta^{t+1}_{i} \gets \theta^t_{i} - \eta^t \nabla_\theta \ell(z^t; \theta^t_{i})$
    \EndFor
    \State \Return $\Theta^T$
    \end{algorithmic}
    \label{alg:msgd}
\end{algorithm}

In order to study
the convergence behavior of the algorithm,
we make the following standard assumptions on learning rates and loss geometry,
which are the same as in \citet{Su2024-qw}.
\begin{assumption}
    \label{ass:learning_rate}
    The learning rates satisfy $\sum_{t=1}^\infty \eta^t = \infty$ and $\sum_{t=1}^\infty (\eta^t)^2 < \infty$.
\end{assumption}

\begin{assumption}
    \label{ass:loss_measure}
    For any $\theta \neq \theta'$, there exists $d_0 > 0$ such that for all $d < d_0$, the set $\{z : |\ell(z; \theta) - \ell(z; \theta')| < d\}$ has Lebesgue measure at most $d$.
\end{assumption}
Assumption~\ref{ass:loss_measure}, from \citet{Su2024-qw}, rules out
pathological cases where a large mass of users is nearly indifferent between
two learners. Intuitively, it ensures that small parameter perturbations move
only a small mass of users across loss-induced decision boundaries, so the
partition-dependent potential $f$ can be differentiated without boundary terms.
Assumption~\ref{ass:bounded_iterates} below is the standard boundedness
condition used in stochastic-approximation analyses.

\subsection{Convergence to Stationary Points}

In MSGD, each learner $i$ optimizes its expected loss over observed users.
Define the potential function $f(\Theta)$ as the sum of the expected losses of all learners
on the distributions that they observe.
\begin{equation}
    \label{eq:potential_function_msgd}
    f(\Theta) = \sum_{i=1}^m \mathbb{E}_{\mathcal{O}_i(\Theta)}[\ell(z; \theta_i)],
\end{equation}
where we recall that
$
\mathcal{O}_i(\Theta) = \tau \alpha_i \mathcal{P}_i + (1 - \tau) a_i(\Theta) \mathcal{D}_i(\Theta),
$
and $w_i(\Theta) = \tau \alpha_i + (1 - \tau) a_i(\Theta)$.
Note that this is the sum of the ``local" losses of each of the learners on their observed distribution, unlike the global risk $\mc{R}(\cdot)$ that was introduced above.

In order
to show convergence, we make the same boundedness assumptions as in
\citet{Su2024-qw}.
\begin{assumption}
    \label{ass:bounded_iterates}
    The parameter sequence $\{\Theta^t\}_{t \geq 0}$
    is almost surely bounded: $\sup_{t \geq 0} \|\Theta^t\| < \infty$.
    Moreover, the set $\{\Theta : \nabla f(\Theta) = 0 \}$ is compact.
\end{assumption}

Our approach uses stochastic approximation \citep{borkar2008stochastic}
and differs from that of \citet{Su2024-qw}.
The key insight is that $f$ serves as a Lyapunov function:
despite each learner optimizing over a different,
endogenously-determined user distribution,
the aggregate of local losses forms a coherent potential.
This is surprising because multi-agent gradient dynamics
often cycle or diverge \citep{mazumdar2020On};
the Lyapunov structure explains why convergence occurs here.
Concretely, we show that the MSGD iterates track the ODE $\dot{\Theta} = -\nabla f(\Theta)$. The following lemma, proved using standard stochastic approximation arguments, establishes this connection.
\begin{restatable}{lemma}{LemMSGDConvergenceODE}
    \label{lem:msgd_convergence_ode}
    Let Assumptions~\ref{ass:bounded_support}--\ref{ass:bounded_iterates} hold. Then the iterates $\{\Theta^t\}$ of Algorithm~\ref{alg:msgd} converge to a compact connected internally chain transitive invariant set of the ODE $\dot{\Theta} = -\nabla f(\Theta)$.
\end{restatable}
See Appendix~\ref{sec:app_ict_defs} for definitions of the dynamical-systems terms used in Lemma~\ref{lem:msgd_convergence_ode}.
The lemma shows that the discrete stochastic dynamics behave, in the limit, like the continuous gradient flow on $f$.
Since $f$ decreases along trajectories of this
ODE i.e $\tfrac{d}{dt} f(\Theta(t)) = -\|\nabla f\|^2 \leq 0$, the only invariant sets are stationary points. This yields our main convergence result.
\begin{restatable}{theorem}{LemMSGDConvergence}
    \label{lem:msgd_convergence}
    \label{lem:msgd_idealized_convergence}
    Let Assumptions~\ref{ass:bounded_support}--\ref{ass:bounded_iterates} hold. Then the iterates $\{\Theta^t\}$ of Algorithm~\ref{alg:msgd} converge to the set of stationary points $\{\Theta : \nabla f(\Theta) = 0\}$ almost surely.
\end{restatable}
Hence, the stochastic approximation framework reveals that MSGD implicitly minimizes a single global objective $f$, providing a clean explanation for why convergence occurs despite the multi-agent structure.

The same potential argument extends to mini-batch updates, matching the
implementation used in the experiments. If a batch $\mathcal{B}^t$ is sampled
i.i.d.\ and learner $i$ updates using the average gradient over the selected
subset $S_i^t=\{z\in\mathcal{B}^t:M(z;\Theta^t)=i\}$, then conditional on
$S_i^t\neq\emptyset$ this average is an unbiased gradient over
$\mathcal{O}_i(\Theta^t)/w_i(\Theta^t)$. Since
$w_i(\Theta)\ge \tau\alpha_i>0$ for $\tau>0$, the batch dynamics track the
same potential with a bounded state-dependent rescaling of the step size.

\subsection{The Overspecialization Trap}

While Theorem~\ref{lem:msgd_convergence} guarantees convergence,
the stationary points may be highly undesirable.
A learner cannot improve on users it never observes,
and it never observes users it cannot serve well.
This feedback loop is the overspecialization trap.

The following theorem shows that this trap can be severe:
MSGD can converge to equilibria where some learners have
arbitrarily poor global performance,
even when models with low full-population loss exist.
\begin{restatable}{theorem}{badoutcomesimple}
    \label{thm:bad_outcome_simple}
    Let Assumptions~\ref{ass:bounded_support}--\ref{ass:bounded_iterates} hold. For any $\tau \geq \tfrac{1}{2}$ and any choice of $\epsilon, \Gamma$ with $0 < \epsilon < \Gamma$, there exists an instance $\mathcal{G}$ such that:
    \begin{enumerate}
        \item There exists $\theta^\star$ with $\mathcal{R}(\theta^\star) \leq \epsilon$.
        \item The MSGD iterates converge to a unique stationary point $\bar{\Theta}$ where $\mathcal{R}(\bar{\theta}_i) \geq \Gamma$ for some learner $i \in [m]$.
    \end{enumerate}
\end{restatable}

The key mechanism is as follows. When $\tau \geq \tfrac{1}{2}$, inherent preferences dominate at equilibrium: the loss-induced partition collapses to the ranking partition, so that $Z_i(\Theta) = S_i$ for all $i$. Each learner optimizes exclusively for users who intrinsically prefer it, arriving at
\[
\bar{\theta}_i = \arg\min_{\theta} \mathbb{E}_{z \sim \mathcal{P}_i}[\ell(z; \theta)].
\]
This specialization occurs regardless of whether a better global model exists.
In the constructed instance
(detailed in Appendix~\ref{sec:app_bad_outcome}),
learner 1 achieves \emph{zero} loss on its observed population
$\mathcal{P}_1$ while its full-population loss exceeds $\Gamma$.
The learner has perfectly fit its niche while
becoming arbitrarily poor globally.
This dynamic formalizes the echo chamber phenomenon that
platforms become increasingly specialized to
their existing audience, unable to learn models
that serve the broader population.

The condition $\tau \geq \tfrac{1}{2}$ delineates the regime
where inherent preferences dominate user behavior.
A key contribution of the analysis is showing that
in this regime
the equilibrium is unique and obtainable in closed form, which enables us to exactly characterize the risk.
However, when $\tau < \tfrac{1}{2}$ and
quality-based selection dominates,
there may be multiple equilibria. Now, the limit
point becomes initialization-dependent, which presents
a technical obstacle to characterizing the limiting
risk in closed form;
however, experiments in Section~\ref{sec:numerical}
demonstrate similar phenomena across different values of $\tau$.

\section{Mitigating Overspecialization through Peer Probing}
\label{sec:probing}

The standard MSGD dynamics from
\cite{Su2024-qw} fail to converge to models that perform well on the full population
because they are ``blocked'' from observing users outside the ones who choose them, and
are limited to learning from a restricted portion of the feature space.

However, in many practical settings, learners can \emph{probe} the predictions
of other learners. For example, a platform could create an account on a
competitor's service to observe their recommendations. This concept has gained
particular prominence in the context of Large Language Models (LLMs) through
\emph{knowledge distillation} \cite{hinton2015distilling}, where one model learns
from another model's outputs. Recent work has extensively explored this area
\cite{werner2025deepseek, xu2024survey, yang2024survey, li2024direct, tan2023gkd, minillm2023}.
The release of the DeepSeek LLM \cite{werner2025deepseek} has brought knowledge
distillation to the forefront of both policy discussions and mainstream media attention
\cite{werner2025deepseek}. Despite this growing practical importance, the theoretical
implications of these multi-agent learning interactions—where models train on each
other's outputs—remain largely unexplored. This section asks the question:
under what circumstances can probing the predictions of others help
overcome the overspecialization trap?

\subsection{Algorithm}

We propose MSGD with Probing (MSGD-P), shown in Algorithm~\ref{alg:msgd_probe}.
The algorithm has two phases. In the \emph{offline phase},
each probing learner $j \in U$ collects a dataset $\mathfrak{D}_j$
of pseudo-labeled examples by querying peer models on sampled covariates.
In the \emph{online phase}, learners interleave standard MSGD updates
(on organic users) with gradient steps on their probing datasets.

More concretely, for each probing learner $i \in U$, we sample covariates
$(\tilde{x}^1_i, \ldots, \tilde{x}^n_i) \sim \mathcal{P}_X^n$.
Given a query covariate $x$, the learner selects a subset of peers
$T_i(x) \subseteq [m]$ to consult, and forms a pseudo-label via median aggregation
\[
\yAgg(x, \Theta) = \operatorname{median}\{h_{\theta_j}(x) : j \in T_i(x)\}.
\]
We then define $\tilde{y}^q_i := \yAgg(\tilde{x}^q_i, \Theta^0)$ and
the pseudo-labeled examples $\tilde{z}^q_i := (\tilde{x}^q_i, \tilde{y}^q_i)$,
and collect the probing dataset $\mathfrak{D}_i := \{\tilde{z}^q_i\}_{q=1}^n$.
The choice of $T_i(x)$ determines which peers are consulted;
we discuss this in Section~\ref{sec:when_probing_helps}.

For a probing learner $i \in U$, the update can be interpreted as
a stochastic gradient step on the following instantaneous loss:
\begin{equation}
\label{eq:probing_loss}
L_i^t(\theta_i) = \tau \alpha_i \mathbb{E}_{z \sim \mathcal{P}_i}[\ell(z; \theta_i)]
+ (1-\tau) a_i(\Theta^t) \mathbb{E}_{z \sim \mathcal{D}_i(\Theta^t)}[\ell(z; \theta_i)]
+ \frac{p}{n}\sum_{q=1}^n \ell(\tilde{z}^q_i; \theta_i)
+ \frac{\lambda p}{2}\|\theta_i\|^2.
\end{equation}
The first two terms capture organic learning from users who select learner $i$
(via inherent preference $\mathcal{P}_i$ or quality-based choice $\mathcal{D}_i$),
while the third term captures learning from probing data.
The parameter $p > 0$ controls the relative weight of probing gradients:
larger $p$ emphasizes pseudo-labels, smaller $p$ prioritizes organic data.

\begin{algorithm}[t]
    \caption{Multi-learner Streaming Gradient Descent with Probing (MSGD-P)}
    \begin{algorithmic}[1]
    \Require loss function $\ell(\cdot, \cdot) \geq 0$;
    Initial models $\Theta^0 = (\theta^0_1, \ldots, \theta^0_m)$;
    Learning rate $\{\etat\}_{t=1}^{T+1}$,
    probing weight $p > 0$, set of probing learners $U \subseteq [m]$, regularization weight $\lambda \ge 0$
    \vspace{0.25em}
    \State \algcomment{Offline probing data collection}
    \For{$j \in U:$}
        \State Sample covariates $(\tilde{x}^1_j, \ldots, \tilde{x}^n_j) \sim \mathcal{P}_X^n$.
        \State Collect pseudo-labels and store dataset $\frakD_j = \{\left(\tilde{x}^1_j,
        \yAgg(\tilde{x}^1_j, \Theta^0)\right) \ldots \left(\tilde{x}^n_j,
        \yAgg(\tilde{x}^n_j, \Theta^0)\right)\}$
    \EndFor
    \State \algcomment{Online updates}
    \For{$t = 0, 1, 2, \dots, T$}
        \State Sample data point $\xt \sim \mathcal{P}$
        \State User selects model $i = M(\xt; \Theta^t)$
        \State $\theta^{t+1}_{i} \gets \theta^t_{i} -
        \etat \nabla \ell(\xt, \theta^t_{i})$
        \For{$j \in U:$}
            \State Sample $\tilde{z}^t_j$ uniformly from $\frakD_j$
            \State $\theta^{t+1}_j \gets \theta^t_j -
            \etat p \left(\nabla \ell(\tildeztj, \theta^t_j)
            + \lambda \theta^t_j\right)$
        \EndFor
    \EndFor
    \State \Return $\Theta^T$
    \end{algorithmic}
    \label{alg:msgd_probe}
\end{algorithm}

We assume that probing learners can sample covariates from the
full distribution $\mathcal{P}_X$, but do not have access to true labels.
This asymmetry is natural in practice: covariates are often publicly available
or easy to generate, while labels require costly human annotation or reveal
private user behavior.
For example, a streaming service knows the metadata of all movies
(genre, cast, runtime) and general user demographics,
but not which movies a non-subscriber would rate highly.
In the LLM setting, covariates are simply text prompts,
which can be generated synthetically or scraped from public sources
(e.g., Reddit, StackOverflow),
whereas ground-truth responses require expensive human evaluation.

We focus on \emph{offline probing}, where pseudo-labels are collected once
at initialization from a fixed snapshot of peer models.
This design choice mirrors practical knowledge distillation workflows:
in settings like DeepSeek distilling from Claude or GPT-4,
teachers are queried to create a fixed dataset,
and student training then proceeds independently \citep{werner2025deepseek}.
Offline probing also ensures reproducibility by capturing a specific
model version's behavior,
avoiding inconsistencies from querying peers at different stages of adaptation.
We discuss extensions to online probing in Section~\ref{sec:discussion}.

\subsection{Convergence}

With probing, each learner $i \in U$ now optimizes a blend of two objectives: the loss on observed users (as in standard MSGD) and the loss on probing data. This leads to a modified potential function:
\begin{equation}
    \label{eq:potential_function_msgd_probe}
  \widetilde{f}(\Theta) = f(\Theta) + p \sum_{i \in U}  \left( \frac{1}{n} \sum_{q = 1}^n \ell(\tilde{z}^q_i, \theta_i)
  + \frac{\lambda}{2} \|\theta_i\|^2 \right),
\end{equation}
where the second term captures the probing loss (with regularization) weighted by $p$.

The same stochastic approximation analysis from Section~\ref{sec:standard_learner_rules} extends to this setting: $\widetilde{f}$ serves as a Lyapunov function for the modified dynamics, yielding the following convergence guarantee.
\begin{restatable}{theorem}{ThmMSGDProbeConvergence}
    \label{thm:msgd_probe_convergence}
    Let Assumptions~\ref{ass:bounded_support}-\ref{ass:loss_measure}, \Cref{ass:bounded_iterates}
    (as applied to $\widetilde{f}$), and \Cref{ass:accurate_probing} hold.
    Then, the iterates $\{\Theta^t\}$ of \Cref{alg:msgd_probe}
    converge to the set of
    stationary points $\{\Theta: \nabla \widetilde{f}(\Theta) = 0\}$ almost surely.
\end{restatable}
Crucially, the stationary points of $\widetilde{f}$ differ from those of $f$: probing changes \emph{where} learners converge, not \emph{whether} they converge. The probing term pulls learners toward models that perform well on the probed distribution, potentially escaping the overspecialization trap. However, this benefit depends on the quality of the pseudo-labels—a question we address next.

\subsection{When Does Probing Help?}
\label{sec:when_probing_helps}

Note that the convergence guarantee above holds for any choice of
$T_i(x)$.
However, in order for the probing data to be \emph{helpful},
the pseudo-labels must be a good proxy for the ground-truth labels.

\begin{assumption}[Accurate Probing]
\label{ass:accurate_probing}
We say that the ``accurate probing'' condition
holds for probing learner $i \in [m]$ if there exists $B \geq 0$ such that
\[
\E_{(x,y) \sim \mathcal{P}}
\left[
\bigl(\yAgg(x, \Theta_{-i}) - y\bigr)^2
\right]
\leq B.
\]
\end{assumption}
This assumption is stated for a single probing learner $i$.
Different learners may satisfy it via different scenarios (or not at all);
performance guarantees in Section~\ref{sec:performance_guarantees} apply
to any learner for whom the assumption holds.

Below, we identify scenarios under which
Assumption~\ref{ass:accurate_probing} holds.
These scenarios differ along two axes:
\emph{what learner $i$ must know} about the market,
and \emph{what must be true} about peer models at initialization.
Table~\ref{tab:probing_scenarios} summarizes these scenarios;
the second and third columns display the knowledge-vs-peer-requirement tradeoff.

\begin{table}[t]
    \centering
    \small
    \begin{tabular}{p{0.16\textwidth}p{0.14\textwidth}p{0.22\textwidth}p{0.12\textwidth}p{0.18\textwidth}}
    \toprule
    Scenario & What $i$ knows & Peer requirement & $T_i(x)$ & $B$ \\
    \midrule
    Majority-good & Nothing & $>50\%$ in $B_r(\theta^\ast)$ & $[m]$ & $R^2 r^2 + 2\epsilon$ \\
    Market-leader & Identity of $j^\ast$ & $\mathcal{R}(\theta^0_{j^\ast}) \le \xi$ & $\{j^\ast\}$ & $\xi$ \\
    Partial knowledge & Subset $G$ & $>50\%$ of $G$ in $B_r(\theta^\ast)$ & $G$ & $R^2 r^2 + 2\epsilon$ \\
    Preference-aware & $\pi(x)$ & \textbf{Nothing} & $\{\pi(x)\}$ & $\epsilon$ \\
    \bottomrule
    \end{tabular}
    \vspace{0.5em}
    \caption{Probing scenarios with corresponding rules and accuracy bounds.
    The preference-aware scenario is notable: it requires no assumption
    on peer quality, only knowledge of user preferences.}
    \label{tab:probing_scenarios}
\end{table}

\begin{definition}[Globally good peers]
\label{def:globally_good_scenarios}
We say learner $i$ can achieve accurate probing via \emph{globally good peers}
if the following instance properties and knowledge conditions hold:
\begin{enumerate}[label=(\roman*), leftmargin=2em]
    \item \textbf{Majority-good.}\label{itm:majority_good}
    More than half of the learners satisfy
    $\theta_j^0 \in B_r(\theta^\ast)$ for a given proximity parameter $r > 0$.

    \item \textbf{Market-leader.}\label{itm:market_leader}
    There exists a single learner $j^\ast \in [m]$ such that
    \[
        \E_{x \sim \mathcal{P}}[\ell(x, \theta^0_{j^\ast})] \le \xi,
    \]
    and the identity of $j^\ast$ is known to learner~$i$.

    \item \textbf{Partial knowledge.}\label{itm:partial_knowledge}
    There exists a subset $G \subseteq [m] \setminus \{i\}$ such that
    more than half of the learners in $G$ satisfy
    $\theta_j^0 \in B_r(\theta^\ast)$ for a given $r > 0$,
    and learner~$i$ has knowledge of the subset $G$.
\end{enumerate}
\end{definition}
Above, the parameter $r > 0$ in the majority-good and partial-knowledge scenarios
controls how close peers must be to $\theta^\ast$;
smaller $r$ yields tighter bounds. When no globally good learners exist or can be identified,
probing may still be effective if the learner has access to ranking information.

\begin{definition}[Preference-aware probing]
\label{def:ranking_based_scenarios}
Suppose all learners initialize at the parameters
$\bar{\Theta} = (\bar{\theta}_1, \ldots, \bar{\theta}_m)$.
If learner $i$ has knowledge of the inherent preference function
$\pi(z)$, which identifies each user's preferred platform,
we call this the \emph{preference-aware}\label{itm:preference_aware} scenario.
\end{definition}

These definitions capture different approaches to probing:
\emph{Definition~\ref{def:globally_good_scenarios}} covers settings where
learners can identify or rely on peers with strong global performance,
while \emph{Definition~\ref{def:ranking_based_scenarios}} addresses settings
where learners must instead leverage knowledge of user preferences.
The scenarios exhibit a fundamental tradeoff:
when peer models are favorable (e.g., many are globally good),
learner $i$ needs little knowledge to achieve accurate probing;
conversely, with stronger knowledge (e.g., knowing $\pi(x)$), the learner
can probe in a more targeted way, and
probing succeeds even when no peer is globally competent.

In realistic markets, learners often naturally gain access to information
enabling one of these approaches. For Scenarios~(\ref{itm:majority_good}) or
(\ref{itm:partial_knowledge}),
platforms may observe broad industry benchmarks
\citep{liang2023helm,chiang2024chatbot, russakovsky2015imagenet,
bennett2007netflix, olson2017pmlb}. For \Cref{itm:preference_aware},
it is natural to maintain knowledge of user preference patterns
\citep{guadagni1983logit, abdullah2021eliciting, huang2018cross,
iranikermani2023brand}. For Scenario~\ref{itm:market_leader}, there are examples in the
LLM literature of this explicitly happening: for instance, the
Alpaca model \citep{taori2023alpaca}
was admittedly trained on data generated by text-davinci-003 (GPT-3.5),
the Vicuna model \citep{chiang2023vicuna}
was trained on data generated by GPT-4, and the survey paper
\cite{gudibande2023false} documents the prevalence of this practice.

The probing rule $T_i(x)$ (fourth column of Table~\ref{tab:probing_scenarios})
in each scenario is chosen so that median aggregation is robust:
probing all peers when the majority are good,
targeting the known leader when one exists,
or routing to the locally-expert peer in the preference-aware case.
The preference-aware scenario is particularly notable:
it requires \emph{no assumption} on peer quality,
only knowledge of user preferences $\pi(x)$,
enabling learner $i$ to aggregate specialized knowledge into
global competence even when every peer suffers from overspecialization.

The next lemma shows that in each scenario,
the pseudo-labels obtained via the corresponding $T_i(x)$
are uniformly bounded in mean-squared error,
ensuring that Assumption~\ref{ass:accurate_probing} holds.

\begin{restatable}{lemma}{goodprobinglemma}
    \label{lem:good_probing}
    For each scenario in Definitions~\ref{def:globally_good_scenarios}--\ref{def:ranking_based_scenarios},
    the probing rule $T_i(x)$ in Table~\ref{tab:probing_scenarios} satisfies
    Assumption~\ref{ass:accurate_probing} with the stated accuracy parameter $B$.
\end{restatable}

\subsection{Performance Guarantees}
\label{sec:performance_guarantees}

We now characterize the full-population risk of MSGD-P stationary points.
The bound decomposes into four interpretable components:
an irreducible Bayes error term,
a probing bias term capturing pseudo-label inaccuracy,
and two regularization-dependent terms reflecting a bias-variance tradeoff.

\begin{restatable}{theorem}{sqperformancebigO}
    \label{cor:sq-performance-bigO}
    Let Assumptions~\ref{ass:bounded_support}--\ref{ass:accurate_probing} hold.
    For any $\kappa\in(0,1)$, with probability at least $1-\kappa$,
    we have for probing learner $i \in [m]$, for the squared loss:
    $$
    \Risk(\tilde\theta_i)
    \;\leq\;
    O\!\left(\Bigl(\frac{p+1}{p}\Bigr)\epsilon
    + B
    + \lambda \|\theta^\star\|^2
    \;+\;\frac{(p+1) C_{\text{gen}} }{p\lambda}
    \sqrt{\frac{\log(1/\kappa)}{n}} \right).
    $$
    where
    $C_{\text{gen}}\left(R, \Ymax, \|\theta^\star\|,
    \max_{j\ne i}\|\theta_j^0\|\right)$
    is stated explicitly in the Appendix.
  \end{restatable}

The four terms admit natural interpretations:
(i) $\frac{p+1}{p}\epsilon$ is the irreducible Bayes error,
scaled by the ratio of total to probing gradient weight;
(ii) $B$ is the probing bias from pseudo-label inaccuracy
(see Table~\ref{tab:probing_scenarios});
(iii) $\lambda\|\theta^\star\|^2$ is the regularization bias; and
(iv) the final term captures finite-sample generalization error
from $n$ probing queries.

The probing weight $p$ controls the balance between
the organic distribution (users who choose learner $i$)
and the probing distribution (the full population).
A learner who prioritizes performance on their existing user base
may prefer smaller $p$,
while a learner seeking global competence benefits from larger $p$.

The regularization parameter $\lambda$ exhibits a classical bias-variance tradeoff.
Larger $\lambda$ increases the regularization bias ($\lambda\|\theta^\star\|^2$)
but improves generalization by keeping parameter norms bounded,
reducing the $O(1/\sqrt{\lambda n})$ term.
Conversely, smaller $\lambda$ reduces bias but worsens the generalization bound.

\begin{restatable}{corollary}{sqperformancebias}
    \label{cor:sq-performance-bias}
    Let Assumptions~\ref{ass:bounded_support}--\ref{ass:accurate_probing} hold.
    Fix $\lambda = \epsilon/\|\theta^\star\|^2$ and any $\kappa\in(0,1)$.
    Define $M_0 := \max_{j\ne i}\|\theta_j^0\|$.
    If there are sufficiently many probing samples $n \ge \underline n(p,\kappa,\epsilon,R,\Ymax,M_0,\|\theta^\star\|)$,
    then with probability at least $1-\kappa$,
    every stationary point $\tilde\Theta$ of MSGD-P satisfies, for each probing learner $i$,
    \[
        \Risk(\tilde\theta_i)
        \;\le\;
        O\!\left(\Bigl(\frac{p+1}{p}\Bigr)\epsilon + B\right).
    \]
\end{restatable}
See Appendix~\ref{sec:app_sq_performance} for the explicit sample complexity $\underline n$ and proof.
We note that this sample complexity bound is not tight in $n$:
we show in Figure~\ref{fig:census_probe_vs_n_main} that strong empirical recovery can occur with very small probing datasets.

\begin{remark}[Cross-entropy loss]
An analogous performance guarantee holds for cross-entropy loss;
see Assumption~\ref{ass:accurate_probing_ce} and Appendix~\ref{sec:ce_performance} for the bound,
and Table~\ref{tab:accurate_probing_ce} for the corresponding accuracy parameters.
\end{remark}

Hence, probing breaks the information barrier
created by user-choice dynamics:
a learner who is overspecialized cannot observe users outside its niche and thus cannot learn to serve them.
While \Cref{thm:bad_outcome_simple} shows that
the loss of a learner may be arbitrarily worse than $\epsilon$,
for sufficiently large $n$, the bound above presents a ceiling on the risk
of any probing learner.

\section{Numerical Experiments}
\label{sec:numerical}

We evaluate our approach on three real-world datasets: MovieLens-10M,
the ACS Employment dataset from the US Census (Alabama, 2018),
and the Amazon Reviews 2023 corpus.

\begin{figure}[t]
    \centering
        \includegraphics[width=\textwidth]{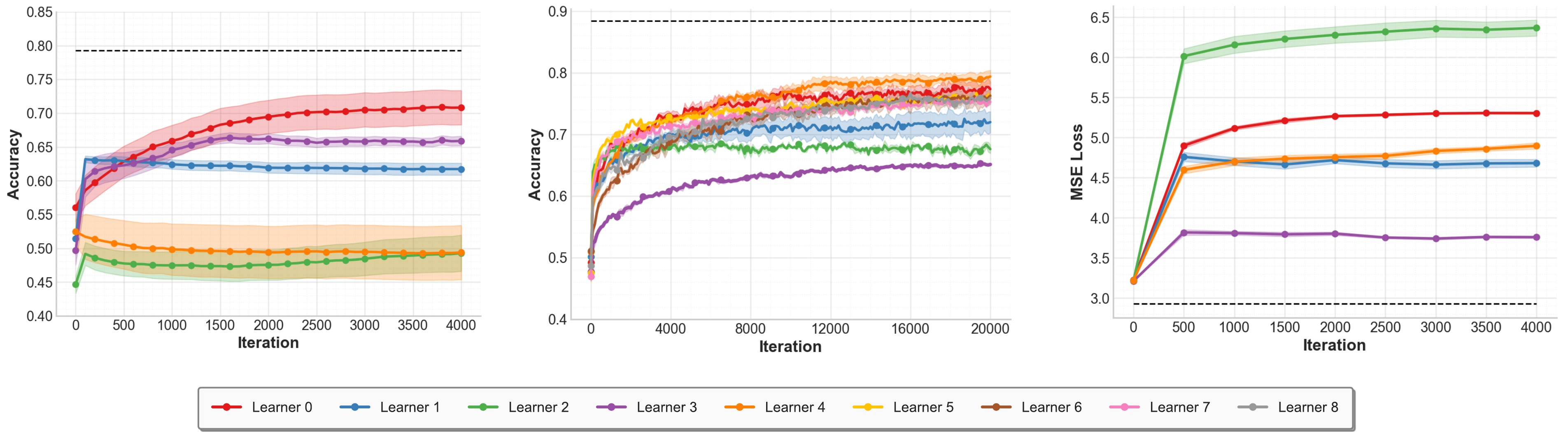}
    \caption{\textbf{MSGD full-population performance with random initialization
    (Preference-aware scenario).}
     Left: Census test accuracy. Mid: Amazon sentiment test accuracy
     Right: MovieLens test loss. The
     dashed black line represents the performance of a baseline $\theta^\ast$
     trained on the full dataset.
     In all cases, $\tau = 0.3$; other dataset-specific hyperparameters are
     given in Table~\ref{tab:hyperparameters}.}
	\label{fig:bad_outcome}
\end{figure}

\vspace{0.5em}

\subsection{Experimental Setup}

We evaluate on three datasets; in each case, every data point corresponds to an individual user.

\paragraph{Movie Recommendation with Squared Loss.}
We use the MovieLens-10M dataset \citep{harper2015movielens},
which contains 10 million movie ratings from 70k users across 10k movies---a
natural testbed for multi-learner competition in recommendation.
Following \citet{bose2023initializing} and \citet{Su2024-qw},
we extract $d=16$ dimensional user embeddings via matrix factorization
and retain ratings for the top 200 most-rated movies,
yielding a population of 69,474 users.
Each user's data consists of $z = (x, r)$ where $x \in \mathbb{R}^d$
is the embedding and $r$ contains their ratings.
Let $\Omega_x$ denote the set of movies rated by user $x$ with $|\Omega_x|$ movies.
Each learner fits a linear model $\theta \in \mathbb{R}^{d \times 200}$ using squared loss:
\begin{align*}
\ell(z; \theta) = \frac{1}{|\Omega_x|} \sum_{i \in \Omega_x} (\theta_i^\top x - r_i)^2.
\end{align*}

\paragraph{Census Data with Logistic Loss.}
We use the ACSEmployment task from \texttt{folktables} \citep{ding2021retiring},
where the goal is to predict employment status from demographic features.
The population consists of 38,221 individuals from the 2018 Alabama census (ages 16--90),
with $d=16$ features describing age, education, marital status, etc.
Each user's data is $z = (x, y)$ where $x \in \mathbb{R}^d$
(standardized to zero mean, unit variance) and $y \in \{0,1\}$.
Each learner uses logistic regression:
\begin{align*}
\ell(z; \theta) = -y \log(\sigma(\theta^\top x)) - (1-y) \log(1 - \sigma(\theta^\top x)),
\end{align*}
where $\sigma$ is the sigmoid function. The model predicts $\hat{y} = \mathbf{1}[\theta^\top x > 0]$.

\paragraph{Amazon Reviews 2023 with Logistic Loss.}
We use the Amazon Reviews 2023 corpus and sample up to 30{,}000 reviews from nine product categories.
Each review is represented by a $d=384$ sentence embedding using \texttt{all-MiniLM-L6-v2}, and each learner performs binary sentiment classification with logistic regression.
As in the Census setting, labels are binary and predictions are thresholded logistic outputs.
Full preprocessing details are provided in Appendix~\ref{sec:app_numerical}.

\paragraph{User Preferences.}
For Census and MovieLens ($m=5$), we induce inherent preferences $\pi(z)$ via K-means clustering on user features, assigning each user to a preferred platform based on cluster membership.
For Amazon ($m=9$), preferences are determined by product category, with all reviews from a given category preferring the same learner.

\paragraph{Evaluation.}
For each dataset, the train/test split is fixed once and shared across all
learners; there is no validation split, since no per-run hyperparameter tuning
is performed. Census and Amazon report standard binary classification accuracy
on the held-out test set, while MovieLens reports masked test MSE.

\subsection{Experimental Results}

We present results for the preference-aware scenario from
\Cref{def:ranking_based_scenarios}. The results are qualitatively
similar in the other scenarios from
\Cref{def:globally_good_scenarios} as well,
and are presented in \Cref{sec:app_numerical}.

\paragraph{Expt 1: MSGD converges to equilibria with poor global performance.}
Our first set of experiments validates Theorem~\ref{thm:bad_outcome_simple}.
Figure~\ref{fig:bad_outcome} shows the full-population performance trajectories of
individual learners on all three datasets without probing ($p=0$) over $T = 4000$ rounds,
when initialized randomly.
In all cases, the results reveal large overspecialization gaps relative to the dashed
black baseline in Figure~\ref{fig:bad_outcome}.

\paragraph{Expt 2: Peer Model Probing Mitigates Overspecialization.}
\begin{figure}[t]
    \centering
    \includegraphics[width=\textwidth]{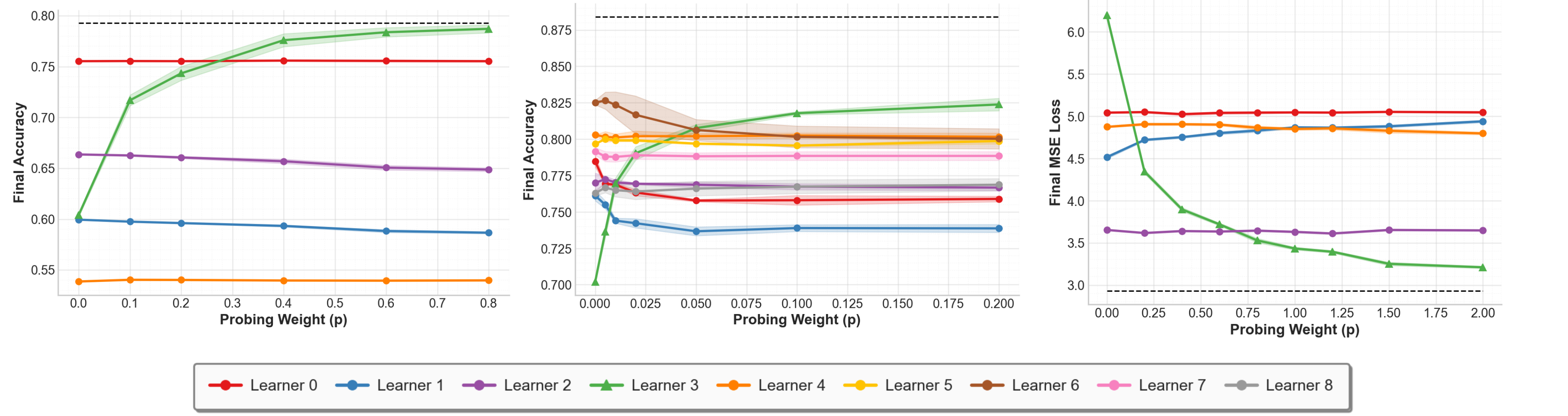}
    \caption{\textbf{Effect of probing on full-population performance (Preference-aware scenario).}
    Left: Census final accuracy vs probing weight $p$.
    Mid: Amazon sentiment final accuracy vs probing weight $p$
    Right: MovieLens final loss vs $p$. In all cases, triangle markers
    indicate the probing learner. Here $\tau = 0.7$; other dataset-specific
    hyperparameters are given in Table~\ref{tab:hyperparameters}.}
    \label{fig:good_ranking}
\end{figure}

Figure~\ref{fig:good_ranking} shows learner final performance at timestep $T$
as a function of probing weight $p$ using the offline probing budgets in
Table~\ref{tab:hyperparameters}; here, one learner (indicated by triangle markers)
uses preference-aware probing while other learners use standard MSGD.

On Census (left), the probing learner's accuracy improves from approximately 60\% at $p=0$
to about 78\% at $p=0.8$, shrinking its baseline gap from roughly 18 percentage points to about 1 percentage point.
The improvement is monotonic in $p$: even modest probing ($p=0.2$)
yields noticeable gains.
On MovieLens (right), the effect is equally pronounced:
the probing learner's MSE loss decreases from approximately 6.2 to 3.5.

\paragraph{Expt 3: Probing is robust to noisy source selection.}
\begin{figure}[t]
    \centering
    \includegraphics[width=\textwidth]{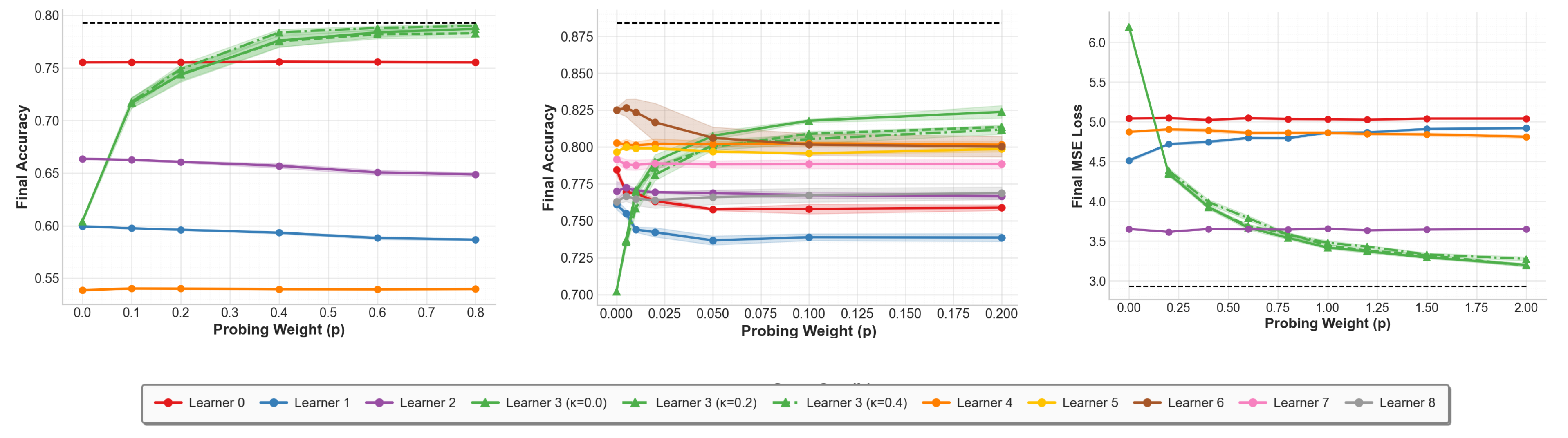}
    \caption{\textbf{Effect of noisy probing-source selection.}
    Preference-aware probing routes each probe query to the preferred peer
    $\pi(x)$ with probability $1-\kappa$ and to a random peer with probability
    $\kappa$. Left: Census final test accuracy vs.\ probing weight $p$.
    Middle: Amazon sentiment final test accuracy vs.\ $p$. Right: MovieLens
    final test MSE vs.\ $p$. The green learner is the probing learner; its
    multiple curves show different $\kappa$ values.}
    \label{fig:kappa}
\end{figure}
Figure~\ref{fig:kappa} tests a noisy version of preference-aware probing:
for each probe query, the learner routes to $\pi(x)$ with probability
$1-\kappa$ and to a random peer with probability $\kappa$. Across datasets,
probing remains beneficial under moderate routing noise, showing that the
method does not require perfect knowledge of user preferences.

\paragraph{Expt 4: How much probing data is needed?}
Figure~\ref{fig:census_probe_vs_n_main} studies sample efficiency by sweeping the probing dataset size $n$ on Census, averaged over 10 random seeds.
We observe substantial gains even with very small probing sets:
for the probing learner with $p \in \{0.5,1.0\}$,
final accuracy rises from about $0.68$ at $n=5$ to about $0.78$ by $n=50$,
and then saturates near $0.79$ at $n=100$;
this is a tiny fraction of the full dataset size of $38{,}221$ examples.
As $n$ increases, mean performance improves and variability across runs decreases, consistent with the finite-sample term in Theorem~\ref{cor:sq-performance-bigO}, which scales as $1/\sqrt{n}$.

\begin{figure}[t]
    \centering
    \includegraphics[width=0.6\textwidth]{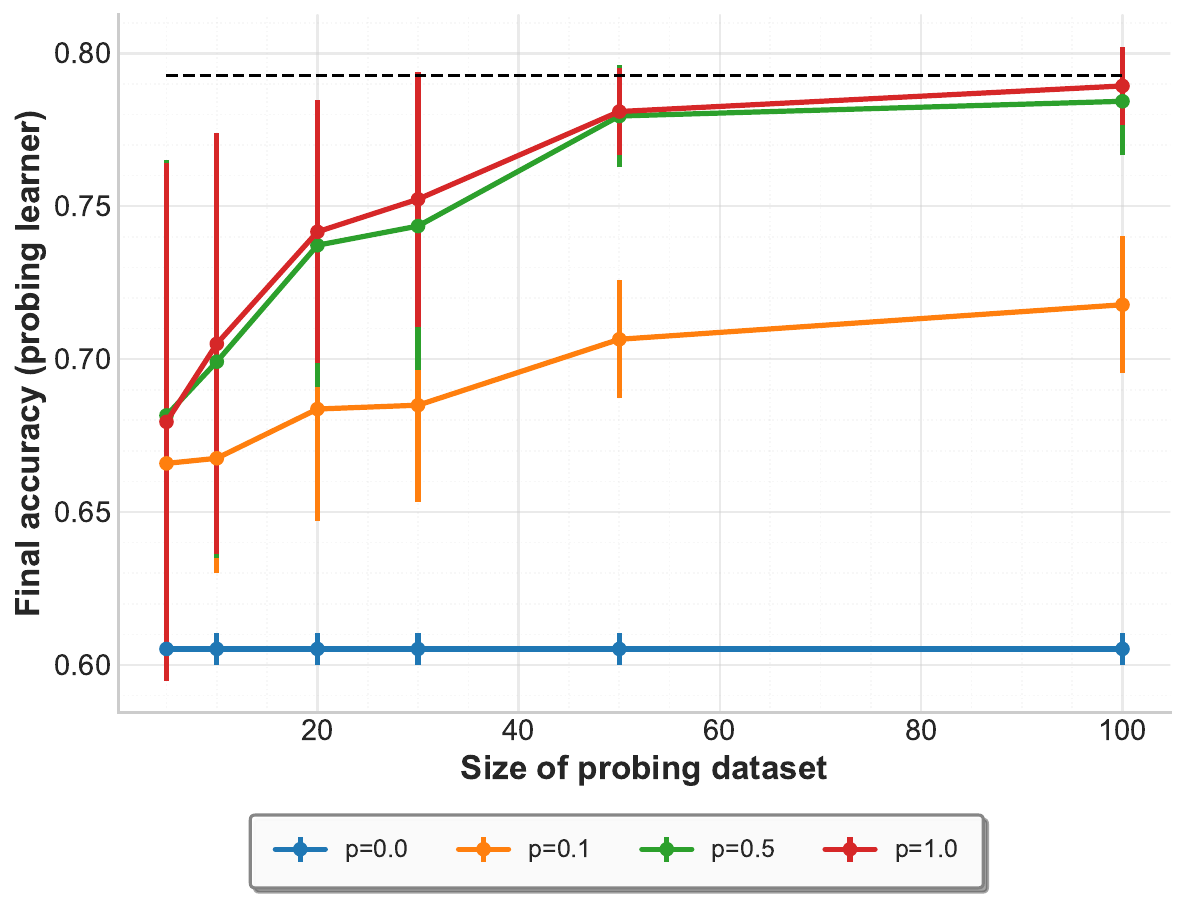}
    \caption{\textbf{Performance of probing learner on Census as a function of $n$.} Error bars show one standard deviation over 10 random seeds.}
    \label{fig:census_probe_vs_n_main}
\end{figure}

\paragraph{Expt 5: Two-layer neural learners.}
\begin{figure}[t]
    \centering
    \includegraphics[width=\textwidth]{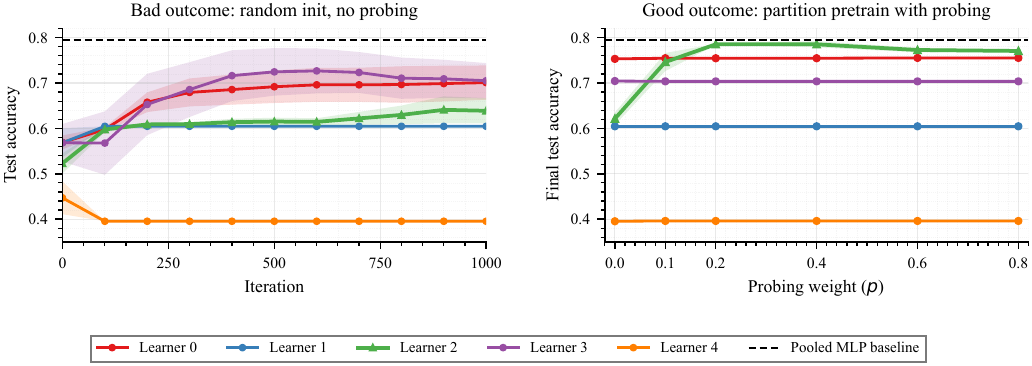}
    \caption{\textbf{Census experiments with two-layer neural learners.}
    We replace each linear logistic learner with a two-layer ReLU MLP
    ($d\!\to\!64\!\to\!1$) trained by SGD with binary cross-entropy, while
    keeping the Census train/test split, K-means preference clustering, induced
    learner rankings, and offline pseudo-label construction fixed. Left:
    random initialization with standard MSGD and no probing ($\tau=0.3,p=0$).
    Right: partition-pretrained initialization with preference-aware probing by
    Learner 2 ($\tau=0.7,\kappa=0$). Shaded bands show standard error over
    three seeds.}
    \label{fig:census_nn}
\end{figure}
Figure~\ref{fig:census_nn} repeats the Census preference-aware experiments with
two-layer ReLU MLP learners. The same qualitative pattern persists: standard
MSGD overspecializes, while probing improves the underperforming learner toward
the pooled MLP baseline.

\section{Discussion}
\label{sec:discussion}

We studied competitive dynamics in machine learning markets where users
choose platforms based on inherent preferences and predictive quality.
We showed that standard learning dynamics in competitive ML markets
converge to overspecialized equilibria,
and that peer model probing provably
mitigates this failure under identifiable informational conditions.
The key insight is that optimizing for observed users creates
information barriers that standard dynamics cannot overcome;
probing breaks these barriers,
with guarantees that degrade gracefully with pseudo-label quality.
Our experiments confirm that even small
probing datasets suffice to close most of the gap.

\paragraph{User Preference Modeling.}
Our experiments simulate user preferences via K-means clustering,
which provides a controlled environment but may not capture
the full complexity of real-world platform choice.
Settings with explicit preference data or richer choice models
(e.g., multinomial logit with heterogeneous coefficients) merit exploration.

\paragraph{Online Probing.}
Our analysis considers offline probing, where pseudo-labels are collected once at initialization;
this mirrors fixed-teacher distillation workflows and avoids querying peers at different stages of adaptation.
Extending to \emph{online probing}, where learners continuously query adapting peers throughout training, introduces co-adaptation dynamics that may lead to instabilities reminiscent of model collapse \citep{shumailov2024model}.
Characterizing when online probing converges—and whether it outperforms offline probing—is an interesting direction for future work.

\paragraph{Convex Losses and Linear Models.}
Our theoretical analysis is restricted to convex losses (squared and cross-entropy)
with linear predictors. The Census MLP experiment in Figure~\ref{fig:census_nn}
suggests that the qualitative overspecialization and probing-recovery
phenomena persist for simple non-convex learners, but extending the
convergence and performance guarantees to deep networks remains an important
open direction.

\section*{Acknowledgements}
Ratliff, Fazel and Narang are supported in part by NSF Award 2312775. Fazel and Narang are supported by NSF TRIPODS II
DMS-2023166. Fazel is also supported by NSF CCF-2212261. Dean is supported in part by NSF CCF 2312774.

\newpage
\bibliography{neurips_2025/neurips_refs}
\bibliographystyle{plainnat}

\appendix

\newpage

\section{Extended Related Work}
\label{sec:extended_related_work}

This section provides detailed comparisons with related work
summarized in Section~1.1 of the main paper.

\paragraph{Single-learner user agency.}
A body of work studies learning dynamics when users are treated as
independent agents rather than passive data points.
\citet{hashimoto2018fairness} show that empirical risk minimization
can cause minority groups to opt out, creating a feedback loop
that further degrades minority performance.
\citet{zhang2019group} study a related dynamic where
underrepresented groups receive worse predictions,
leading to further data scarcity.
\citet{james2023participatory} analyze participatory data collection
where users choose whether to contribute data.
\citet{ben2017best} study best-response dynamics in strategic classification.
\citet{Cherapanamjeri2023-Fisherman} and \citet{Harris2023-mm}
study settings where data quality or availability depends
on the learner's past performance.
We build on this perspective in the multi-learner setting,
where inter-learner interactions create additional feedback dynamics.

\paragraph{Multi-learner competition and user choice.}
Several works study multi-learner user-choice settings with explicitly
strategic users who optimize their own utility functions
\citep{shekhtman2024strategic, ben2017best, ben2019regression,
jagadeesan2023competition, jagadeesan2023improved}.
\citet{ginart2021competing} provide both empirical and theoretical analysis
of how competition drives specialization:
their Theorems 4.1--4.3 establish risk ratio bounds showing
that competing predictors perform worse on the general population
than a single predictor would.
Their analysis considers batch retraining dynamics and characterizes
the \emph{existence} of performance gaps due to competition.
We complement this by analyzing streaming gradient-based dynamics,
proving that such dynamics \emph{converge to} overspecialized equilibria
(Theorem~\ref{thm:bad_outcome_simple}),
and proposing probing as a mitigation.
\citet{kwon2022competition} study a related setting where learners may
purchase user data, providing primarily empirical analysis of
the resulting market dynamics.

\paragraph{Gradient-based dynamics in choice-driven settings.}
Most closely related to our work,
\citet{dean2022multi} and \citet{Su2024-qw}
analyze gradient-based dynamics in choice-driven settings.
\citet{Su2024-qw} introduce the MSGD algorithm and prove convergence to
stationary points of an aggregate loss across learners.
We build directly on their framework:
our Algorithm~\ref{alg:msgd} adapts MSGD to our user selection rule
(Definition~\ref{def:user_selection_rule}),
and our potential function (Equation~\ref{eq:potential_function_msgd})
extends theirs. Our Theorem~\ref{lem:msgd_convergence}
provides an alternate proof of their convergence result using
stochastic approximation techniques,
which we believe better illuminates the underlying dynamics.
Beyond convergence, we analyze generalization to users outside each
learner's observed population---a consideration absent from prior
work---and introduce peer probing as a mechanism to restore
global competence.
\citet{bose2023initializing} propose intelligent initialization schemes
to improve outcomes in similar settings, but also focus on the sum of
all local losses on the observed distribution.

\section{Terminology and Preliminary Results}
\label{app:loss_regularity}

\subsection{Dynamical systems terminology}
\label{sec:app_ict_defs}

This subsection records standard definitions underlying terms such as
``invariant set'' and ``internally chain transitive'' in \Cref{lem:borkar_track}.
Let $\dot{x}(t) = h(x(t))$ be an ODE with locally Lipschitz $h$, so that solutions are unique.
Let $\phi_t(x)$ denote the state at time $t \ge 0$ of the solution initialized at $x$ at time $0$.

\begin{definition}[Invariant set]
A set $A \subseteq \mathbb{R}^d$ is (forward) \emph{invariant} for the ODE if
$\phi_t(x) \in A$ for all $x \in A$ and all $t \ge 0$.
\end{definition}

\begin{definition}[$(\varepsilon,T)$-chain]
Fix $\varepsilon > 0$ and $T > 0$. An \emph{$(\varepsilon,T)$-chain in $A$ from $x$ to $y$}
is a finite sequence of points $x=x_0, x_1,\ldots, x_k=y$ in $A$ and times
$t_1,\ldots,t_k \ge T$ such that
\[
\|\phi_{t_i}(x_{i-1}) - x_i\| < \varepsilon \qquad \text{for all } i \in \{1,\ldots,k\}.
\]
\end{definition}

\begin{definition}[Internally chain transitive]
A compact invariant set $A$ is \emph{internally chain transitive} if for every
$x,y \in A$ and every $\varepsilon > 0$, $T > 0$, there exists an $(\varepsilon,T)$-chain in $A$ from $x$ to $y$.
\end{definition}

\subsection{Regularity of Squared and Cross Entropy Losses}
\begin{lemma}[Regularity of the squared‐error loss]
  \label{lem:squared_loss_regularity}
  Let \(\mathcal X\subset\R^d\), \(\mathcal Y\subset\R\), \(W\subset\R^d\) be compact. Write
  \[
    \ell(x,y;\theta) \;=\;\bigl(y - x^\top\theta\bigr)^2,
    \qquad
    B =\max_{\theta\in W}\|\theta\|.
  \]
  Then for all \((x,y)\in\mathcal X\times\mathcal Y\) and all \(\theta_1,\theta_2\in W\):
  \begin{enumerate}[label=(\roman*)]
    \item \(\ell\) is nonnegative, \(C^\infty\) in \(\theta\), convex, and \(\beta\)–smooth with
      \(\beta = 2R^2\).
    \item \(\ell\) is locally Lipschitz on \(W\):
      \[
        \bigl|\ell(x,y;\theta_1)-\ell(x,y;\theta_2)\bigr|
        \le L_W\,\|\theta_1-\theta_2\|,
        \quad
        L_W = R\bigl(2|y| + 2RB\bigr).
      \]
    \item The Hessian is constant in
    \(\theta\), so
      \(\|\nabla^2\ell(\theta_1)-\nabla^2\ell(\theta_2)\|=0\)
      (i.e.\ \(\gamma_\ell=0\)).
      \end{enumerate}
  \end{lemma}

  \begin{proof}
  \begin{enumerate}[label=(\roman*)]
    \item[(i)]  One computes
    \(\nabla_\theta\ell = -2\,x\bigl(y - x^\top\theta\bigr)\)
    and
    \(\nabla^2_\theta\ell = 2\,x\,x^\top\succeq0\).
    Hence \(\ell\ge0\), is \(C^\infty\), convex, and
    \(\|\nabla\ell(\theta_1)-\nabla\ell(\theta_2)\|\le2R^2\|\theta_1-\theta_2\|\).

    \item[(ii)]  Observe
    \[
      \ell(\theta_1)-\ell(\theta_2)
      = \bigl(y - x^\top\theta_1\bigr)^2 - \bigl(y - x^\top\theta_2\bigr)^2
      = x^\top(\theta_2-\theta_1)\;\bigl(2y - x^\top(\theta_1+\theta_2)\bigr).
    \]
    Since \(\|\theta_i\|\le B\),
    \(\lvert x^\top(\theta_2-\theta_1)\rvert\le R\|\theta_1-\theta_2\|\)
    and
    \(\bigl|2y - x^\top(\theta_1+\theta_2)\bigr|\le2|y|+2RB\),
    giving the stated bound.

    \item[(iii)]  Because \(\nabla^2_\theta\ell\equiv2\,x\,x^\top\) is independent of \(\theta\), its difference vanishes.
  \end{enumerate}
  \end{proof}

\begin{lemma}
  \label{lem:loss_bound}
  Under Assumption~\ref{ass:bounded_support}, for all $z\in\mathcal Z$ the loss $\ell(z,\cdot)$ is non-negative, convex, differentiable, locally Lipschitz and $\beta_\ell$-smooth with
  \[  \beta_\ell = 2R^2 \quad\text{for both the squared loss and cross-entropy loss.}\]
  Moreover, let
  \[
    \fNP(\theta) = \E_{z\sim\mathcal P}[\ell(z,\theta)],
    \quad
    \theta^* = \arg\min_\theta \fNP(\theta),
    \quad
    \epsilon = \fNP(\theta^*).
  \]
  Then for any $\theta$ with $\|\theta - \theta^*\| \le \gamma$, the following quadratic upper bound holds:
  \[
    \fNP(\theta) \;\le\; \epsilon + \tfrac{\beta_\ell}{2}\,\|\theta - \theta^*\|^2
    \;\le\; \epsilon + R^2\,\gamma^2.
  \]
\end{lemma}

\begin{proof}
The first part (non-negativity, convexity, differentiability, local Lipschitzness, and smoothness) follows directly by applying Lemmas~\ref{lem:squared_loss_regularity} and~\ref{lem:crossent_regularity} under Assumption~\ref{ass:bounded_support}, which together show each $\ell(z,\theta)$ is $\beta_\ell$-smooth with $\beta_\ell=2R^2$.

Since $f(\theta)=\E_{z}[\ell(z,\theta)]$, differentiability and smoothness carry over to $f$, and we have
\[
\|\nabla f(\theta) - \nabla f(\theta^*)\| \le \beta_\ell\,\|\theta - \theta^*\|.
\]
At the minimizer $\theta^*$, $\nabla f(\theta^*)=0$.  Hence for any $\theta$ with $\|\theta-\theta^*\|\le\gamma$, the gradient satisfies
\[
\|\nabla f(\theta)\| \le \beta_\ell\,\|\theta-\theta^*\| \le \beta_\ell\,\gamma.
\]

Applying the standard smoothness inequality (second-order Taylor upper bound) around $\theta^*$,
\[
f(\theta) \le f(\theta^*) + \langle\nabla f(\theta^*),\,\theta-\theta^*\rangle + \tfrac{\beta_\ell}{2}\,\|\theta-\theta^*\|^2
= \epsilon + \tfrac{\beta_\ell}{2}\,\|\theta-\theta^*\|^2,
\]
where we used $\nabla f(\theta^*)=0$ and $f(\theta^*)=\epsilon$.  Finally, substituting $\beta_\ell=2R^2$ gives
\[
f(\theta) \le \epsilon + R^2\,\|\theta-\theta^*\|^2 \le \epsilon + R^2\,\gamma^2,
\]
completing the proof.
\end{proof}

  \begin{lemma}[Regularity of the multiclass cross‐entropy loss]
    \label{lem:crossent_regularity}
    Let $\mathcal X\subset\R^d$, $\mathcal Y=\{1,\dots,K\}$, $W\subset\R^{K\times d}$ compact, and let $\|x\|\le R$ for all $x\in\mathcal X$. Define for $(x,y)\in\mc{X}\times\mc{Y}$ and $\Theta=(\theta_1,\dots,\theta_K)\in W$
    \[
      \ell(x,y;\Theta)
      \;=\;
      -\log\frac{\exp(\theta_y^\top x)}
                 {\sum_{k=1}^K \exp(\theta_k^\top x)}.
    \]
    Then for all $x,y,\Theta_1,\Theta_2$ as above:
    \begin{enumerate}[label=(\roman*)]
      \item $\ell$ is nonnegative, $C^\infty$ in $\Theta$, convex, and $\beta$–smooth with
        $
          \beta \;=\; R^2.
        $
      \item $\ell$ is (locally) Lipschitz on $W$:
        \[
          \bigl|\ell(x,y;\Theta_1)-\ell(x,y;\Theta_2)\bigr|
          \;\le\;L\,\|\Theta_1-\Theta_2\|,
          \qquad
          L = \sqrt{2}\,R.
        \]
    \end{enumerate}
    \end{lemma}

    \begin{proof}
    The proofs of each subpart are as follows:
    \begin{enumerate}[label=(\roman*)]
      \item[\textit{(i)}]
      Let $p = \softmax(\Theta\,x)\in\Delta^{K-1}$.  One checks
      \[
        \nabla_\Theta\ell
        = (p - e_y)\,x^\top,
        \qquad
        \nabla^2_\Theta\ell
        = \bigl[\diag(p)-p\,p^\top\bigr]\;\otimes\;\bigl(x\,x^\top\bigr).
      \]
      Since $\diag(p)-p\,p^\top\succeq0$
      $\ell\ge0$, is $C^\infty$ and convex, and
      \[
        \|\nabla^2_\Theta\ell\|
        \;\le\;
        \|x\|^2\;\lambda_{\max}\bigl(\diag(p)-p\,p^\top\bigr)
        \;\le\;R^2,
      \]
      it follows that $\ell$ is $\beta$–smooth with $\beta=R^2$.

      \item[\textit{(ii)}]
      By the mean‐value theorem,
      \[
        \bigl|\ell(\Theta_1)-\ell(\Theta_2)\bigr|
        \;\le\;
        \sup_{\Theta\in W}\|\nabla_\Theta\ell\|\;\|\Theta_1-\Theta_2\|.
      \]
      But
      \[
        \|\nabla_\Theta\ell\|
        = \|p - e_y\|\;\|x\|
        \;\le\;\sqrt{2}\,R,
      \]
      giving the claimed Lipschitz constant $L=\sqrt2\,R$.
    \end{enumerate}
    \end{proof}

\begin{restatable}{lemma}{LemSoftmaxLipschitz}
  \label{lem:softmax_lipschitz}
  Let $\{h_\theta\}_{\theta\in\Theta}$ be the family of softmax classifiers mapping $x\in\R^d$ to a distribution over $C$ classes,
  \[  h_\theta(x)_c = \frac{e^{x^\top\theta_c}}{\sum_{j=1}^C e^{x^\top\theta_j}}.  \]
  Assume:
  \begin{enumerate}[label=(\roman*)]
    \item $\|x\|_2 \le R$ for all $x$ in the support of $\mc P$.
    \item There exists $\alpha>0$ such that for all $x,\theta,c$, we have $h_\theta(x)_c\ge\alpha$.
    \item $\theta,\theta^*$ lie in a convex set $\Theta\subset\R^{C\times d}$.
  \end{enumerate}
  Then the mapping $\theta\mapsto h_\theta(x)$ is Lipschitz in the $\ell_1$–norm: for all $\theta,\theta^*\in\Theta$,
  \[
    \|h_\theta(x)-h_{\theta^*}(x)\|_1
    \;\le\;
    L_h\,\|\theta-\theta^*\|_2,
    \quad
    L_h = R\sqrt{2\,(1-\alpha)}.
  \]
  \end{restatable}

  \begin{proof}
  By the multivariate mean‐value theorem, there exists $\bar\theta$ on the line segment between $\theta^*$ and $\theta$ such that
  \[
    h_\theta(x)-h_{\theta^*}(x)
    = D_\theta h(\bar\theta;x)\;\bigl(\theta-\theta^*\bigr),
  \]
  where $D_\theta h(\bar\theta;x)\in\R^{C\times(Cd)}$ is the Jacobian w.r.t. all parameters.  Write its $c$th row as the gradient vector $\nabla_\theta h_c(\bar\theta;x)$.  A direct calculation shows for each class $c$ and each block $k$:
  \[
    \nabla_{\theta_k}\,h_c(\bar\theta;x)
    = h_c\bigl(\mathbf{1}\{c=k\}-h_k\bigr)\,x
    \quad\Longrightarrow\quad
    \|\nabla_\theta h_c(\bar\theta;x)\|_2
    \le
    \|x\|_2\;h_c\;\sqrt{(1-h_c)^2 + \sum_{k\neq c}h_k^2}.
  \]
  Using $\sum_kh_k=1$, $h_k\le 1$ for each $k$, and the lower‐bound $h_c\ge\alpha$, one checks
  \[
    (1-h_c)^2 + \sum_{k\neq c}h_k^2
    \le (1-h_c)^2 + (1-h_c) = (1-h_c)(2-h_c) \le 2\,(1-h_c) \le 2\,(1-\alpha),
  \]
  where the first inequality uses $\sum_{k\neq c}h_k^2 \le \sum_{k\neq c}h_k = 1-h_c$.
  Hence
  \[
    \|\nabla_\theta h_c(\bar\theta;x)\|_2 \le R\,h_c\,\sqrt{2\,(1-\alpha)}.
  \]
  Therefore the operator norm from $\ell_2$ to $\ell_1$ satisfies
  \[
    \|D_\theta h(\bar\theta;x)\|_{2\to1}
    \;=\;
    \sum_{c=1}^C \|\nabla_\theta h_c(\bar\theta;x)\|_2
    \;\le\;
    R\sqrt{2(1-\alpha)}\sum_{c=1}^C h_c
    = R\sqrt{2(1-\alpha)}.
  \]
  Putting these together gives
  \[
    \|h_\theta(x)-h_{\theta^*}(x)\|_1
    \le
    \|D_\theta h(\bar\theta;x)\|_{2\to1}\,\|\theta-\theta^*\|_2
    \le L_h\,\|\theta-\theta^*\|_2,
  \]
  as claimed.
  \end{proof}

    \begin{lemma}
    \label{lem:loss_regularity}
    Under Assumption~\ref{ass:bounded_support}, for all $z \in \mathcal{Z}$, the loss $\ell(z, \cdot)$ is non-negative, convex, differentiable, locally Lipschitz, and $\beta_\ell$-smooth.
    \end{lemma}
    \begin{proof}
      The result follows by applying
      Lemmas \ref{lem:squared_loss_regularity} and \ref{lem:crossent_regularity}
      with $|x|\le R$ and choosing
      $\beta_\ell = \max{2R^2, R^2} = 2R^2$.
      Each loss is nonnegative, convex, differentiable, locally Lipschitz,
      and smooth under these bounds.
    \end{proof}

  \section{Bad Outcome}
  \label{sec:app_bad_outcome}

    \paragraph{Notation}
    Recall the potential function from the main text:
    \begin{equation}
      \label{eq:potential_function_msgd_app}
      f(\Theta) = \sum_{i=1}^m \left(
      \tau \alpha_i \E_{z \sim \mathcal{P}_i}[\ell(z, \theta_i)]
      + (1 - \tau) a_i(\Theta) \E_{z \sim \mathcal{D}_i(\Theta)}[\ell(z, \theta_i)]\right)
    \end{equation}
    From Lemma~\ref{lem:msgd_idealized_convergence}, we have that the
    MSGD iterates converge almost surely to the set
    $\{\nabla f(\Theta) = 0\}$.
    Additionally, for learner $i \in [m]$, define $\bar{\theta}_i$ as the solution
    that each learner would choose if they trained independently on
    the subpopulation of users who rank them highest:
    \begin{equation}
        \label{eq:bar_theta_app}
    \bar{\theta}_i
    = \arg\min_{\theta_i} \E_{x \sim \mc{P}_i} \ell(x, \theta_i)
    \end{equation}

      \begin{example}
        \label{ex:bad_outcome_1d}
        Specify the family of instances $\mathcal{G}_\text{bad}(\tau, C)$ as follows.

      \begin{enumerate}
      \item Distribution $\mathcal{P}$: is defined as a mixture of subpopulations.
      $\mathcal{P} = \alpha\mathcal{P}_1 + (1-\alpha)\mathcal{P}_2$. Here,
      $\{\mathcal{P}_1, \mathcal{P}_2\}$ are $2$ subpopulations.
      For either subpopulation, the covariates are generated from the
      zero mean and unit-variance uniform distribution:
      \begin{align}
        x \sim \text{Unif}([-\sqrt{3}, \sqrt{3}]) \quad \text{for} \quad (x,y) \sim \mathcal{P}_i,
      \end{align}
      For each subpopulation, the response variable is generated as:
      \begin{align}
          y &= C x \quad \text{for} \quad (x,y) \sim \mathcal{P}_1 \\
          y &= - x \quad \text{for} \quad (x,y) \sim \mathcal{P}_2
      \end{align}
      \item Loss function: $\ell(x, y, \theta) = (y - \theta^Tx)^2$ is the squared loss
      \item Ranking $\pi(z) = i$ for
      $(x,y) \sim \mathcal{P}_i$.
      \end{enumerate}
      \end{example}

    \begin{lemma}
    \label{lem:specialists_vs_compromise_1d}
    Consider the 1‑D bad‑outcome family in Example~\ref{ex:bad_outcome_1d} with
    mixture weight $\alpha\in(0,1)$ and slope parameter $C>1$,
    and let
    $\Risk(\theta)=\E_{(x,y)\sim\Pop}[(y-\theta x)^2]$.

    \begin{enumerate}[label=(\roman*)]
        \item The least‑squares predictor on the full mixture,
              $\thetastar=\alpha C-(1-\alpha)$, satisfies
              \[
                  \Risk(\thetastar)=\alpha(1-\alpha)(C+1)^2.
              \]

        \item The specialist trained on $\PopOne$ is
              $\thetabar1 = C$, and its mixture risk is
              \[
                  \Risk(\thetabar1)
                  \,=\, (1-\alpha)(C+1)^2.
              \]

        \item The specialist trained on $\PopTwo$ is
              $\thetabar2 = -1$, and its mixture risk is
              \[
                  \Risk(\thetabar2)
                  \,=\, \alpha\,(C+1)^2.\qedhere
              \]
    \end{enumerate}
    \end{lemma}

    \begin{proof}
    We have that $\E[x] = 0$ and $\E[x^2]=1$,
    so that the squared risk reduces to $(\beta-\theta)^2$ when the
    true slope is $\beta$.

    \paragraph{(i) Global compromise.}
    On the mixture, the conditional label is linear: $y = \beta x$ with
    $\beta = \alpha C + (1-\alpha)(-1) = \alpha C - (1-\alpha)$.
    For squared loss with $\E[x]=0$ and $\E[x^2]=1$, the population least‑squares
    minimizer is the regression slope, so $\thetastar=\beta$. The corresponding
    risk is
    \[
      \Risk(\thetastar)
      = \alpha\,(C-\beta)^2 + (1-\alpha)\,(-1-\beta)^2
      = \alpha(1-\alpha)(C+1)^2.
    \]

    \paragraph{(ii) Specialist for $\PopOne$.}
    Training least squares on $\PopOne$ alone yields $\thetabar1=C$ since
    $\E[x\,y]=C$ and $\E[x^2]=1$. The mixture risk is
    \[
      \Risk(\thetabar1)
      = \alpha\,(C-C)^2 + (1-\alpha)\,(-1-C)^2
      = (1-\alpha)(C+1)^2.
    \]

    \paragraph{(iii) Specialist for $\PopTwo$.}
    Training least squares on $\PopTwo$ gives $\thetabar2=-1$ (its true slope).
    The mixture risk is
    \[
      \Risk(\thetabar2)
      = \alpha\,(C+1)^2 + (1-\alpha)\,(-1-(-1))^2
      = \alpha\,(C+1)^2.
    \]
    \end{proof}

    \begin{lemma}
      \label{lem:stationary_point_combinations}
      Let $(\tilde{\theta}_1, \tilde{\theta}_2) \in \{\nabla f(\Theta) = 0\}$ be a stationary point
      of $f(\Theta)$ (Equation~\ref{eq:potential_function_msgd}).
      Then, under the assumption that $\tau \geq \frac{1}{2}$, it must be true that $\mc{D}_1 (\tilde{\theta}_1, \tilde{\theta}_2) = \mc{P}_1$
      and $\mc{D}_2 (\tilde{\theta}_1, \tilde{\theta}_2) = \mc{P}_2$.
    \end{lemma}
    \begin{proof}
    The proof proceeds by considering the decisions of the loss-minimizing users,
    and shows that the stationary point condition implies the conditions on
    $\mc{D}_i(\Theta)$ in the lemma statement. Consider any $(x, y)$ from $\mc{P}_1$. For any $\theta \in \mathbb{R}$,
    we have that
    \begin{equation}
      \ell(x, y, \theta) = (Cx - \theta x)^2 = (C - \theta)^2 x^2.
    \end{equation}
    Hence,
    \begin{equation*}
    \arg \min_{i \in \{1, 2\}} \ell(x, y, \theta_i)
    = \arg \min_{i \in \{1, 2\}} (C - \theta_i)^2 x^2
    = \arg \min_{i \in \{1, 2\}} |C - \theta_i|.
    \end{equation*}
    Thus every loss-minimizing user in $\mc{P}_1$ picks the learner whose
    parameter is closer to $C$:
    \[
      M(z;\Theta) \,=\, \arg\min_{i\in\{1,2\}} |C - \theta_i|, \quad z\in\mc{P}_1.
    \]
    Likewise, every loss-minimizing user in $\mc{P}_2$ picks the learner whose
    parameter is closer to $-1$:
    \[
      M(z;\Theta) \,=\, \arg\min_{i\in\{1,2\}} |-1 - \theta_i|, \quad z\in\mc{P}_2.
    \]
    Hence, within each subpopulation $\mc{P}_i$, all loss-minimizing users make the
    same choice. Consequently, for any $\Theta$, the observed-data distributions
    $\mc{D}_i(\Theta)$ can only take one of the following four forms:
    \begin{enumerate}
      \item $\mc{D}_1(\Theta) = \mc{P}$ and $\mc{D}_2(\Theta)$ has zero mass under $\mc{P}$
      \item $\mc{D}_1(\Theta)$ has zero mass under $\mc{P}$ and $\mc{D}_2(\Theta) = \mc{P}$
      \item $\mc{D}_1(\Theta) = \mc{P}_2$ and $\mc{D}_2(\Theta) = \mc{P}_1$
      \item $\mc{D}_1(\Theta) = \mc{P}_1$ and $\mc{D}_2(\Theta) = \mc{P}_2$
    \end{enumerate}
    For any stationary point $(\tilde{\theta}_1, \tilde{\theta}_2)$, we consider each case separately and show that all cases other than the last one lead
    to a contradiction.

    \vspace{0.5em}

    \noindent \textbf{Case 1:} In this case, the second learner observes only negative labels,
    and the first observes a mix of both positive and negative labels. Hence, it should
    be intuitively true that $\tilde{\theta}_1 > \tilde{\theta}_2 > -1$.
    This would lead to a contradiction since this would imply that users from
    $\mc{P}_2$ would strictly prefer the second learner over the first one. Now, we can verify this formally. Define the total probability masses seen by each learner
    \begin{align}
      M_1 &= \tau\,\alpha \;+\;(1-\tau)\,\alpha \;+\;(1-\tau)(1-\alpha)
            \;=\;\alpha \;+\;(1-\tau)(1-\alpha),\\
      M_2 &= \tau\,(1-\alpha).
      \end{align}
      The stationary conditions are given by
      \begin{align}
      \tilde{\theta}_1
      &= \arg \min_{\theta_1} \tau\,\alpha
        \E_{P_1}\bigl[(y-\theta_1\,x)^2\bigr]
      + (1-\tau)\,\alpha
        \E_{P_1}\bigl[(y-\theta_1\,x)^2\bigr]
      + (1-\tau)(1-\alpha)
        \E_{P_2}\bigl[(y-\theta_1\,x)^2\bigr],\\[6pt]
      \tilde{\theta}_2
      &= \arg \min_{\theta_2} \tau\,(1-\alpha)
        \E_{P_2}\bigl[(y-\theta_2\,x)^2\bigr].
      \end{align}
      We can solve these optimization problems in closed form:
      \begin{align}
      \tilde{\theta}_1
      &= \frac{\alpha\,C \;-\;(1-\tau)(1-\alpha)}{M_1},\\
      \tilde{\theta}_2
      &= -1.
      \end{align}
      Now,
      \[
      \tilde{\theta}_1 - \tilde{\theta}_2
      = \tilde{\theta}_1 + 1
      = \frac{\alpha\,(C+1)}{M_1}
      > 0,
      \]
      which holds for every \(\tau\in[0,1]\) since \(\alpha>0\) and \(C>0\).
      Since $\tilde{\theta}_2 = -1$ and $\tilde{\theta}_1>-1$, $\mc{P}_2$ users
      would strictly prefer learner 2 over learner 1, leading to a contradiction.

    \vspace{0.5em}

    \noindent \textbf{Case 2:}
    In this case, the first learner observes only positive labels,
    and the second observes a mix of both positive and negative labels.
    Hence, it is easy to verify using a similar procedure as
    the previous case that $C > \tilde{\theta}_1 > \tilde{\theta}_2$, which
    would lead to a contradiction since $\mc{P}_1$ users would strictly prefer learner 1.

      \vspace{0.5em}
    \noindent \textbf{Case 3:}
    We follow a similar argument as in Case 1.
    Define the total probability masses seen by each learner
    \begin{align}
    M_1 &= \tau\,\alpha \;+\;(1-\tau)(1-\alpha),\\
    M_2 &= \tau\,(1-\alpha)\;+\;(1-\tau)\,\alpha.
    \end{align}
    The stationary points in this case are given by \[
    \tilde{\theta}_1 \;=\;
    \frac{\tau\,\alpha\,C \;-\;(1-\tau)(1-\alpha)}
         {\tau\,\alpha \;+\;(1-\tau)(1-\alpha)},
    \qquad
    \tilde{\theta}_2\;=\;
    \frac{(1-\tau)\,\alpha\,C \;-\;\tau(1-\alpha)}
         {(1-\tau)\,\alpha \;+\;\tau(1-\alpha)}.
    \]

    It is easy to verify that both $\tilde{\theta}_1$ and $\tilde{\theta}_2$ are greater than $-1$.
    Hence, the distance from $-1$ is given by
    \begin{align*}
    d_1 &= \bigl|\tilde{\theta}_1 - (-1)\bigr|
         = \frac{\tau\,\alpha\,(C+1)}
                {\tau\,\alpha \;+\;(1-\tau)(1-\alpha)},\\
    d_2 &= \bigl|\tilde{\theta}_2- (-1)\bigr|
         = \frac{(1-\tau)\,\alpha\,(C+1)}
                {(1-\tau)\,\alpha \;+\;\tau(1-\alpha)}.
    \end{align*}
    A straightforward comparison gives
    \[
    d_1 > d_2
    \;\Longleftrightarrow\;
    \tau\,M_2 > (1-\tau)\,M_1
    \;\Longleftrightarrow\;
    (1-\alpha)\,(2\tau-1) > 0
    \;\Longleftrightarrow\;\tau>\tfrac12.
    \]
    Therefore, whenever $\tau>\tfrac12$, the free users from $\mc{P}_2$ strictly
    prefer learner 2 over learner 1. This contradicts the Case 3 hypothesis
    that $\mc{D}_1(\Theta)=\mc{P}_2$. Hence Case 3 cannot be a stationary point
    when $\tau>1/2$. Having ruled out Cases 1–3, only the fourth configuration
    remains, completing the proof.
    \end{proof}

      \badoutcomesimple*
      \begin{proof}
      The proof follows by considering Example~\ref{ex:bad_outcome_1d}.
      \vspace{0.5em}

      \noindent \textbf{Part (i)}
      The proof follows from Lemma~\ref{lem:specialists_vs_compromise_1d}.
      Choosing $\alpha = \frac{\epsilon}{(C+1)^2}$ satisfies the condition.

      \vspace{0.5em}

      \noindent \textbf{Characterizing the stationary points}
      From Lemma~\ref{lem:msgd_idealized_convergence}, we have that the
      MSGD iterates converge almost surely to the set
      $\{\nabla f(\Theta) = 0\}$.
      By Lemma~\ref{lem:stationary_point_combinations}, we have that
      $(\tilde{\theta}_1, \tilde{\theta}_2) \in \{\nabla f(\Theta) = 0\}$
      if and only if
      $\mc{D}_1 (\tilde{\theta}_1, \tilde{\theta}_2) = \mc{P}_1$
      and $\mc{D}_2 (\tilde{\theta}_1, \tilde{\theta}_2) = \mc{P}_2$.
      Hence, any stationary point must satisfy
      \[
      \tilde{\theta}_1 = \arg \min_{\theta \in \mathbb{R}}
      \alpha \E_{z \sim \mathcal{P}_1}[\ell(z, \theta)]
      \]
      and
      \[
      \tilde{\theta}_2 = \arg \min_{\theta \in \mathbb{R}}
      (1 - \alpha) \E_{z \sim \mathcal{P}_2}[\ell(z, \theta)]
      \]
      Clearly, $(\bar{\theta}_1, \bar{\theta}_2)$ is the unique solution to these equations, and MSGD
      must converge to this point almost surely.

      \vspace{0.5em}

      \noindent \textbf{Characterizing the loss}
      From Lemma~\ref{lem:specialists_vs_compromise_1d}, at the stationary point we have
      $\bar{\theta}_1 = C$ and $\bar{\theta}_2 = -1$, so the mixture risks are
      \[
      \Risk(\bar{\theta}_1) = (1-\alpha)(C+1)^2,
      \qquad
      \Risk(\bar{\theta}_2) = \alpha(C+1)^2.
      \]
      Moreover, the optimal global (compromise) risk satisfies
      $\Risk(\theta^*) \le \epsilon$ by our choice in Part (i).
      To ensure learner 1's risk exceeds $\Gamma$, choose
      \[
      C \,=\, \sqrt{\Gamma + \epsilon}\; - 1,
      \qquad
      \alpha \,=\, \frac{\epsilon}{(C+1)^2}.
      \]
      Then $\alpha \in (0,1)$ and
      \[
      \Risk(\theta^*) = \alpha(1-\alpha)(C+1)^2 \le \alpha(C+1)^2 = \epsilon,
      \]
      while
      \[
      \Risk(\bar{\theta}_1) = (1-\alpha)(C+1)^2 = (C+1)^2 - \epsilon \ge \Gamma.
      \]
      Thus the two claims in the theorem are simultaneously satisfied.
      \end{proof}

\section{Convergence of \Cref{alg:msgd} and \Cref{alg:msgd_probe}}
\label{sec:app_msgd_p_converge}

\subsection{Convergence of Algorithm~\ref{alg:msgd}}
\label{sec:app_msgd_converge}

The convergence results for MSGD (Algorithm~\ref{alg:msgd}) follow as special cases of the MSGD-P analysis when the probing set $U = \emptyset$. In this case, all probing terms vanish and the augmented potential $\widetilde{f}$ reduces to the original potential $f$.

\LemMSGDConvergenceODE*
\begin{proof}
This follows from the proof of \Cref{thm:msgd_probe_convergence} with $U = \emptyset$. When $U = \emptyset$:
\begin{itemize}
    \item The indicator $\mathbf{1}_{i \in U} = 0$ for all $i \in [m]$, so all probing terms vanish.
    \item The augmented potential reduces to $\widetilde{f}(\Theta) = f(\Theta)$.
    \item The ODE~\eqref{eq:def_tildeF_probe} becomes $\dot{\Theta} = -\nabla f(\Theta)$.
\end{itemize}
The stochastic approximation argument proceeds identically: by \Cref{lem:borkar_track}, verifying local Lipschitzness of $-\nabla f$ (\Cref{lem:loss_regularity}, \Cref{lem:ai_local_lipschitz}), the step-size conditions (Assumption~\ref{ass:learning_rate}), the martingale variance bound (\Cref{lem:martingale_variance} with $U = \emptyset$, so $K_{\text{probe}} = 0$), and bounded iterates (Assumption~\ref{ass:bounded_iterates}), the iterates converge almost surely to a compact connected internally chain transitive invariant set of $\dot{\Theta} = -\nabla f(\Theta)$.
\end{proof}

\LemMSGDConvergence*
\begin{proof}
Follows directly from \Cref{thm:msgd_probe_convergence}
with $U = \emptyset$.
\end{proof}

\subsection{Convergence of MSGD-P (Algorithm~\ref{alg:msgd_probe})}

\begin{algorithm}[t]
  \caption{Multi-learner Streaming Gradient Descent}
  \begin{algorithmic}[1]
  \Require loss function $\ell(\cdot, \cdot) \geq 0$; Initial models $\Theta^0 = (\theta^0_1, \ldots, \theta^0_m)$;
  Learning rate $\{\etat\}_{t=1}^{T+1}$
  \For{$t = 0, 1, 2, \dots, T$}
      \State Sample data point $\zt \sim \mathcal{P}$
      \State User selects model $i = M(z^t; \Theta^t)$
      \State $\theta^{t+1}_{i} \gets \theta^t_{i} -
      \etat \nabla \ell(\zt, \theta^t_{i})$
  \EndFor
  \State \Return $\Theta^T$
  \end{algorithmic}
  \label{alg:msgd_probe_ideal}
\end{algorithm}
Define the empirical probing loss and augmented potential:
\begin{align}
  \hat L_i(\theta_i) &= \tfrac{1}{n}\sum_{q=1}^n \ell(\tilde z_i^q,\theta_i), \quad i\in U,\\
  \widetilde f(\Theta) &= f(\Theta) + p\sum_{i\in U} \left( \hat L_i(\theta_i) + \frac{\lambda}{2}\|\theta_i\|^2 \right).
\end{align}

Define the ordinary differential equation (ODE) $\dot{\Theta} = \widetilde{F}(\Theta)$ with:
\begin{equation}
  \label{eq:def_tildeF_probe}
  \widetilde{F}_i(\Theta) = - \Big(
    \tau\alpha_i\,\E_{z\sim\cP_i}[\nabla\ell(z,\theta_i)]
    +(1-\tau)a_i(\Theta)\E_{z\sim\cD_i(\Theta)}[\nabla\ell(z,\theta_i)]
    + p\,\mathbf{1}_{i\in U}\,\bigl(\nabla \hat L_i(\theta_i) + \lambda\,\theta_i\bigr)
  \Big)
\end{equation}

\ThmMSGDProbeConvergence*
\begin{proof}
  The proof follows the stochastic approximation template by showing the iterates track the ODE $\dot\Theta=\widetilde{F}(\Theta)$ and then using a Lyapunov argument for $\widetilde f$.

  Let $z_1 \sim \cP$, $z_2 \sim \cP_i$, and $z_3 \sim \cD_i(\Theta^t)$. For the on-platform update, define for each coordinate $i$ the random direction
  \[
    g_{i,\text{plat}}^t(\Theta^t)=
    \begin{cases}
      \nabla\ell(z_2,\theta_i^t), & \text{w.p. } \tau\alpha_i,\\
      \nabla\ell(z_3,\theta_i^t), & \text{w.p. } (1-\tau)\,a_i(\Theta^t),\\
      0, & \text{otherwise.}
    \end{cases}
  \]
  For the probing update, for each $j\in U$ sample $\tilde z_j^t$ uniformly from the fixed dataset $\frakD_j$ and set
  \[
    g_{i,\text{probe}}^t(\Theta^t)= p\,\mathbf{1}_{i\in U}\,\bigl(\nabla\ell(\tilde z_i^t,\theta_i^t)+\lambda\,\theta_i^t\bigr).
  \]
  Let $\tilde g_i^t(\Theta^t)=g_{i,\text{plat}}^t(\Theta^t)+g_{i,\text{probe}}^t(\Theta^t)$ and $\tilde g^t=(\tilde g_1^t,\dots,\tilde g_m^t)$. The algorithmic iterate satisfies
  \[
    \Theta^{t+1}=\Theta^t-\eta_t\,\tilde g^t(\Theta^t)=\Theta^t-\eta_t\,(\widetilde{F}(\Theta^t)+v^t),
  \]
  where $v^t=\tilde g^t(\Theta^t)-\E[\tilde g^t(\Theta^t)\mid\cF_t]$.

  Drift identity: By construction and by uniform sampling from $\frakD_i$,
  \[
    \begin{aligned}
    \E[\tilde g_i^t(\Theta^t)\mid\cF_t]
    &= \tau\alpha_i\,\E_{z\sim\cP_i}[\nabla\ell(z,\theta_i^t)]
     + (1-\tau)\,a_i(\Theta^t)\,\E_{z\sim\cD_i(\Theta^t)}[\nabla\ell(z,\theta_i^t)]\\
    &\quad + p\,\mathbf{1}_{i\in U}\,\bigl(\nabla\hat L_i(\theta_i^t) + \lambda\,\theta_i^t\bigr) \\
    &= -\widetilde{F}_{i}(\Theta^t).
    \end{aligned}
  \]

  Below, we verify the assumptions of \Cref{lem:borkar_track}:
  \begin{itemize}
    \item From \Cref{lem:loss_regularity} and \Cref{lem:ai_local_lipschitz}, the drift $\widetilde{F}(\Theta)$ is locally Lipschitz; the probe part $\theta_i\mapsto p\,(\nabla\hat L_i(\theta_i) + \lambda\,\theta_i)$ is a finite average of locally Lipschitz gradients plus a globally Lipschitz linear term.
    \item From Assumption~\ref{ass:learning_rate}, the step sizes satisfy $\sum_t \eta_t=\infty$ and $\sum_t \eta_t^2<\infty$.
    \item From \Cref{lem:martingale_variance}, augmented to include the probe term, there exists $K>0$ such that $\E[\|v^t\|^2\mid\cF_t]\le K(1+\|\Theta^t\|^2)$.
    \item From Assumption~\ref{ass:bounded_iterates}, we have $\sup_t\|\Theta^t\|<\infty$ almost surely.
  \end{itemize}

  By \Cref{lem:borkar_track}, the iterates converge almost surely to a compact, connected, internally chain transitive invariant set of $\dot\Theta=\widetilde{F}(\Theta)$. Since $\widetilde{F}(\Theta)=-\nabla\widetilde f(\Theta)$, along ODE trajectories
  \[
    \tfrac{d}{dt}\,\widetilde f(\Theta(t))=\langle\nabla\widetilde f,\dot\Theta\rangle=-\|\nabla\widetilde f\|^2\le0,
  \]
  with equality iff $\nabla\widetilde f=0$. Thus $\widetilde f$ is a strict Lyapunov function and the only invariant sets are stationary points, yielding the claim.
\end{proof}

\begin{lemma}[\citet{borkar2008stochastic}, Chapter 2, Theorem 2]
  \label{lem:borkar_track}
  Let $\{x_n\}$ be a sequence generated by the stochastic approximation algorithm
  \[
  x_{n+1} = x_n + a(n)[h(x_n) + M_{n+1}], \quad n \geq 0,
  \]
  where:
  \begin{enumerate}
  \item $h: \mathbb{R}^d \to \mathbb{R}^d$ is Lipschitz continuous
  \item The step sizes $\{a(n)\}$ satisfy $\sum_n a(n) = \infty$ and $\sum_n a(n)^2 < \infty$
  \item $\{M_n\}$ is a martingale difference sequence satisfying $E[||M_{n+1}||^2|\mathcal{F}_n] \leq K(1 + ||x_n||^2)$ for some $K > 0$
  \item $\sup_n ||x_n|| < \infty$ almost surely
  \end{enumerate}
  Then almost surely, the sequence $\{x_n\}$ converges to a
  (possibly sample path dependent)
  compact connected internally chain transitive invariant set of the ODE
  \[
  \dot{x}(t) = h(x(t)).
  \]
  \end{lemma}

\begin{lemma}[Martingale variance bound]
  \label{lem:martingale_variance}
  Suppose that Assumption~\ref{ass:bounded_support} holds and that the probing datasets $\{\frakD_i\}_{i\in U}$ are fixed finite sets.
  Let $g^t(\Theta^t) = (g^t_1(\Theta^t), \dots, g^t_m(\Theta^t))$ be the stochastic update vector defined by
  \[
  g^t_i(\Theta^t) \;=\; g^t_{i,\mathrm{plat}}(\Theta^t)\; +\; g^t_{i,\mathrm{probe}}(\Theta^t),
  \]
  where
  \[
  g^t_{i,\mathrm{plat}}(\Theta^t)
  = \begin{cases}
      \nabla\ell(z^t,\theta^t_i), & \text{if } i = i_t, \\
      0, & \text{if } i \neq i_t,
    \end{cases}
  \qquad
  g^t_{i,\mathrm{probe}}(\Theta^t)
  = p\,\mathbf{1}_{i\in U}\,\bigl(\nabla\ell(\tilde z_i^t,\theta^t_i) + \lambda\,\theta_i^t\bigr),
  \]
  with $\tilde z_i^t$ sampled uniformly from $\frakD_i$ independently of $i_t$.
  Define the filtration $\mathcal{F}_t = \sigma(\Theta^0, z^1, \dots, z^{t-1}, \{\frakD_i\}_{i\in U})$, which contains all information available prior to time $t$. Define the martingale difference sequence:
  \[
  v^t = g^t(\Theta^t) - \mathbb{E}[g^t(\Theta^t) \mid \mathcal{F}_t].
  \]
  Then there exists a constant $K > 0$ (made explicit below) such that:
  \[
  \mathbb{E}[\|v^t\|^2 \mid \mathcal{F}_t] \leq K\,(1 + \|\Theta^t\|^2).
  \]
\end{lemma}

\begin{proof}
  Gradient bounds. From Lemmas~\ref{lem:squared_loss_regularity} and~\ref{lem:crossent_regularity}, for all $(x,y)$ with $\|x\|\le R$ and all $\theta$ we have bounds of the form
  \[
    \|\nabla_\theta\ell(x,y;\theta)\| \le A_0 + A_1\,\|\theta\|,
  \]
  with
  \[
    (A_0,A_1)=
    \begin{cases}
      (2R\,Y_{\max},\;2R^2), & \text{(squared loss)},\\
      (\sqrt2\,R,\;0), & \text{(cross-entropy)}.
    \end{cases}
  \]
  For the probe term under squared loss, labels are pseudo-labels
  \(
    \hat y=\operatorname{median}\{\langle x,\theta_0^j\rangle: j\in[m]\}
  \)
  so that $|\hat y|\le R\,\max_{j\in[m]}\|\theta_0^j\|=:Y_{\max,\mathrm{probe}}$, giving
  $\|\nabla_\theta\ell(x,\hat y;\theta)\| \le \widetilde A_0 + A_1\,\|\theta\|$ with $\widetilde A_0:=2R\,Y_{\max,\mathrm{probe}}$.
  Including the L2 regularization term (which appears only in the probe update) and using the triangle inequality,
  \[
    \|p\,(\nabla_\theta\ell(x,\hat y;\theta)+\lambda\,\theta)\| \le p\,\widetilde A_0 + p\,(A_1+\lambda)\,\|\theta\|.
  \]
  Define the block constants
  \[
    K_{\text{plat}} := 2\max\bigl(A_0^2,A_1^2\bigr),\qquad
    K_{\text{probe}} := 2\max\bigl((p\,\widetilde A_0)^2,\bigl(p\,(A_1+\lambda)\bigr)^2\bigr),
  \]
  where for cross-entropy we take $A_1=0$ and use the same $A_0=\sqrt2 R$ for both platform and probe.

  Variance decomposition. By definition of $v^t$ and $\Var(Y)=\E[\|Y\|^2]-\|\E[Y]\|^2$,
  \[
    \E[\|v^t\|^2\mid\cF_t]\;\le\;\E[\|g^t(\Theta^t)\|^2\mid\cF_t].
  \]
  We bound the RHS. Since $g^t$ has at most one nonzero platform block and up to $|U|$ nonzero probe blocks,
  \begin{align*}
    \|g^t(\Theta^t)\|^2
    &= \sum_{i=1}^m \|g^t_{i,\mathrm{plat}}+g^t_{i,\mathrm{probe}}\|^2
    \;\le\; 2\sum_{i=1}^m \|g^t_{i,\mathrm{plat}}\|^2 + 2\sum_{i=1}^m \|g^t_{i,\mathrm{probe}}\|^2.
  \end{align*}
  Taking conditional expectations and using the selection probabilities for the platform block as in the MSGD lemma,
  \begin{align*}
    \E[\|g^t(\Theta^t)\|^2\mid\cF_t]
    &\le 2\sum_{i=1}^m \mathbb{P}(i_t=i\mid\cF_t)\;\E[\|\nabla\ell(z^t,\theta_i^t)\|^2\mid\cF_t, i_t=i] \\
    &\quad + 2\sum_{j\in U} \E[\|p\,\nabla\ell(\tilde z_j^t,\theta_j^t)\|^2\mid\cF_t].
  \end{align*}
  Applying the block bounds gives
  \begin{align*}
    \E[\|g^t(\Theta^t)\|^2\mid\cF_t]
    &\le 2K_{\text{plat}}\sum_{i=1}^m \mathbb{P}(i_t=i\mid\cF_t)\,(1+\|\theta_i^t\|^2) \\
    &\quad + 2K_{\text{probe}}\sum_{j\in U} (1+\|\theta_j^t\|^2) \\
    &\le 2K_{\text{plat}}\,(1+\|\Theta^t\|^2)\; +\; 2|U|K_{\text{probe}}\; +\; 2K_{\text{probe}}\,\|\Theta^t\|^2 \\
    &= 2\bigl(K_{\text{plat}}+|U|K_{\text{probe}}\bigr)\; +\; 2\bigl(K_{\text{plat}}+K_{\text{probe}}\bigr)\,\|\Theta^t\|^2.
  \end{align*}
  Therefore, setting
  \[
    K := 2\max\bigl\{K_{\text{plat}}+|U|K_{\text{probe}},\; K_{\text{plat}}+K_{\text{probe}}\bigr\}
  \]
  yields
  \[
    \E[\|v^t\|^2\mid\cF_t] \le K\,(1+\|\Theta^t\|^2).
  \]
  This completes the proof.
\end{proof}

    \begin{lemma}[Gradient of the augmented objective]
      \label{lem:gradient_stream_probe}
  For every \(i\in[m]\) the gradient of \(\widetilde f\) with respect to~\(\theta_i\) is
  \begin{align*}
      \nabla_{\theta_i} \,\widetilde f(\Theta)
      &= \tau\alpha_i\,\E_{z\sim\cP_i}[\nabla\ell(z,\theta_i)]
      + (1-\tau)\,a_i(\Theta)\,\E_{z\sim\cD_i(\Theta)}[\nabla\ell(z,\theta_i)]\\
      &\quad + p\,\mathbf{1}_{i\in U}\,\bigl(\nabla\hat L_i(\theta_i) + \lambda\,\theta_i\bigr).
  \end{align*}
      \end{lemma}

      \begin{proof}
      We treat a single index \(i\); all other coordinates of \(\Theta\) are
      held fixed. The term
      \((1-\tau)a_i(\Theta)\E_{z\sim\cD_i(\Theta)}[\ell(z,\theta_i)]\)
      depends on \(\theta_i\) both explicitly (inside~\(\ell\)) and implicitly
      through \(a_i(\Theta)\) and \(\cD_i(\Theta)\).
      Lemma 4.3 of \citet{Su2024-qw} proves that for any differentiable
      \(\ell\) satisfying the stated regularity,
      \[
        \nabla_{\theta_i}\!\Bigl(a_i(\Theta)\,\E_{z\sim\cD_i(\Theta)}[\ell(z,\theta_i)]\Bigr)
        = a_i(\Theta)\,\E_{z\sim\cD_i(\Theta)}[\nabla\ell(z,\theta_i)].
      \]
      Multiplying by \((1-\tau)\) yields the corresponding contribution, while the \(\tau\alpha_i\)-term is immediate.
      Local L-Lipschitzness of \(\ell\) ensures the required directional limits.

      The probing part depends only on \(\theta_i\) through the finite average
      \(\hat L_i(\theta_i)=\tfrac1n\sum_{q=1}^n \ell(\tilde z_i^q,\theta_i)\) and the regularization term \(\tfrac{\lambda}{2}\|\theta_i\|^2\),
      so
      \(\nabla_{\theta_i}\bigl(p\,(\hat L_i(\theta_i) + \tfrac{\lambda}{2}\|\theta_i\|^2)\bigr)=p\,(\nabla\hat L_i(\theta_i) + \lambda\,\theta_i)\) if \(i\in U\) and zero otherwise.
      Summing the contributions gives the stated expression.
    \end{proof}

\begin{lemma}[Local Lipschitz of \(a_i(\Theta)\)]
  \label{lem:ai_local_lipschitz}
  Let the standard regularity and boundedness assumptions for the idealized MSGD dynamics hold.
  Then \(a_i(\Theta)\) is locally Lipschitz in \(\Theta\).
  \end{lemma}

  \begin{proof}

  From \Cref{lem:loss_regularity}, we know that the loss function $\ell$ is locally Lipschitz.
  In other words: for every compact set \(\mathcal{K}\subset\R^{k\times d}\) there is a constant \(L_{\mathcal{K}}<\infty\) such that
  \[
    \bigl|\ell(x,\theta)-\ell(x,\theta')\bigr|\;\le\;L_{\mathcal{K}}\,\|\theta-\theta'\|
    \quad\forall\,x\in B(0,R),\;\theta,\theta'\in \mathcal{K}.
  \]
  Fix any compact neighborhood \(\mathcal{K}\) containing both \(\Theta\) and \(\Theta'\).  Let
  \[
    p_{\max} \;=\; \sup_{x\in B(0,R)}p(x),
    \qquad
    L_{\mathcal{K}}\;\text{as above.}
  \]
  We will show
  \(\bigl|a_i(\Theta)-a_i(\Theta')\bigr|\le C_{\mathcal{K}}\|\Theta-\Theta'\|\)
  for some \(C_{\mathcal{K}}\).

  \medskip\noindent\textbf{Case \(m=2\).}  With services \(i=1,2\),
  \[
    a_1(\Theta)-a_1(\Theta')
    = \int_{X_1(\Theta)\setminus X_1(\Theta')}p(x)\,dx
      \;-\;\int_{X_1(\Theta')\setminus X_1(\Theta)}p(x)\,dx,
  \]
  so
  \[
    \bigl|a_1(\Theta)-a_1(\Theta')\bigr|
    \;\le\; p_{\max}\,\Bigl[\lambda\bigl(X_1(\Theta)\setminus X_1(\Theta')\bigr)
                       +\lambda\bigl(X_1(\Theta')\setminus X_1(\Theta)\bigr)\Bigr].
  \]
  For any \(x\in X_1(\Theta')\setminus X_1(\Theta)\) we have
  \(\ell(x,\theta'_1)<\ell(x,\theta'_2)\) and \(\ell(x,\theta_2)<\ell(x,\theta_1)\).
  Hence
  \[
    0 \;<\; \ell(x,\theta_1)-\ell(x,\theta_2)
    = \bigl[\ell(x,\theta_1)-\ell(x,\theta'_1)\bigr]
    + \bigl[\ell(x,\theta'_1)-\ell(x,\theta'_2)\bigr]
    + \bigl[\ell(x,\theta'_2)-\ell(x,\theta_2)\bigr].
  \]
  Since \(x\in X_1(\Theta')\), the middle term satisfies \(\ell(x,\theta'_1)-\ell(x,\theta'_2)\le 0\).
  The first and third terms are each bounded by \(L_{\mathcal{K}}\|\Theta-\Theta'\|\) by Lipschitzness.  Thus
  \(\ell(x,\theta_1)-\ell(x,\theta_2)\le 2L_{\mathcal{K}}\|\Theta-\Theta'\|\).  Define
  \[
    S \;=\; \bigl\{\,x:\bigl|\ell(x,\theta_1)-\ell(x,\theta_2)\bigr|\le 2L_{\mathcal{K}}\|\Theta-\Theta'\|\bigr\}.
  \]
  Since \(\lambda(S)\le (2L_{\mathcal{K}}/C)\|\Theta-\Theta'\|\) for some constant \(C\)
  from \Cref{ass:loss_measure}
  and \(X_1(\Theta')\setminus X_1(\Theta)\subset S\), we get
  \(\lambda(X_1(\Theta')\setminus X_1(\Theta))\le C'\|\Theta-\Theta'\|\).
  The same argument applies to the other set difference, so altogether
  \[
    \bigl|a_1(\Theta)-a_1(\Theta')\bigr|
    \;\le\;4\,p_{\max}\,L_{\mathcal{K}}\;\|\Theta-\Theta'\|.
  \]

  \medskip\noindent\textbf{General \(m\).}  The same pairwise argument shows
  for each \(i\) and \(j\neq i\),
  \(\lambda\bigl(X_i(\Theta)\cap X_j(\Theta')\bigr)\le C_{\mathcal{K}}\|\Theta-\Theta'\|\).
  Summing over all \(j\neq i\) gives
  \(\lambda\bigl(X_i(\Theta)\triangle X_i(\Theta')\bigr)\le C''_{\mathcal{K}}\|\Theta-\Theta'\|\),
  and hence
  \(\bigl|a_i(\Theta)-a_i(\Theta')\bigr|\le p_{\max}\,C''_{\mathcal{K}}\,\|\Theta-\Theta'\|\).

  \medskip
  Since all constants depend only on the compact set \(\mathcal{K}\), this proves that
  \(a_i(\Theta)\) is Lipschitz on \(\mathcal{K}\), i.e.\ locally Lipschitz in \(\Theta\).
  \end{proof}

  \section{Squared Loss: Performance Guarantee}
  \label{sec:app_sq_performance}

  \noindent \textbf{Notation.}
  Let $\{(x_i^q,y_i^q)\}_{q=1}^n$ be the probing sample with true labels $y_i^q$
  and pseudo-labels $\tilde y_i^q$.
  Here, the true labels $y_i^q$ are hidden from the learner, and
  the pseudo-labels $\tilde y_i^q$ are observed.
  For any $\theta\in\mathbb{R}^d$, define
  \[
  \widehat L_{i}(\theta)
  := \frac{1}{n}\sum_{q=1}^n \big(\langle x_i^q, \theta \rangle - \tilde y_i^q\big)^2,
  \]
  and
  \[
    \widehat L_{i, \mathrm{true}}(\theta)
    := \frac{1}{n}\sum_{q=1}^n \big(\langle x_i^q, \theta \rangle - y_i^q\big)^2.
  \]
  Additionally, define the empirical pseudo--true discrepancy
  \[
  \DeltaN^2 := \frac{1}{n}\sum_{q=1}^n \big(\tilde y_i^q - y_i^q\big)^2.
  \]

\subsection{Proof of \Cref{lem:good_probing}}

We restate the lemma for ease of reference, and then provide the proof.

\goodprobinglemma*
  \begin{proof}
We treat each scenario in turn.

\textbf{(i) Majority-good.}
Fix any $x$ with $\|x\|\le R$. Write the peer deviations
$u_j:=\langle x,\theta_j^0\rangle - \langle x,\theta^\star\rangle$.
If strictly more than half of the peers satisfy
$\|\theta_j^0-\theta^\star\|\le r$,
then at least half of the $\{u_j\}$ lie in $[-Rr,Rr]$,
so the median obeys $|\tilde y(x) - \langle x,\theta^\star\rangle|\le Rr$ and hence
\[
(\tilde y - y)^2
\;=\;\big(\tilde y - \langle x,\theta^\star\rangle + \langle x,\theta^\star\rangle - y\big)^2
\;\le\; 2\,(\tilde y - \langle x,\theta^\star\rangle)^2 \;+\; 2\,(\langle x,\theta^\star\rangle - y)^2
\;\le\; 2R^2 r^2 \;+\; 2\,(\langle x,\theta^\star\rangle - y)^2.
\]
Taking expectation over $(x,y)\sim\cP$ yields
$\E[(\tilde y - y)^2] \le 2R^2 r^2 + 2\,\E[(\langle x,\theta^\star\rangle - y)^2] = 2R^2 r^2 + 2\epsilon$.

\vspace{0.5em}

\textbf{(ii) Market-leader.}
Here $\tilde y(x)=x^\top\theta_{j^\ast}$ and by assumption
$\E\big[(\tilde y - y)^2\big]=\E\big[(x^\top\theta_{j^\ast}-y)^2\big]\le \xi$, which directly verifies Accurate Probing with $B=\xi$.

\vspace{0.5em}

\textbf{(iii) Partial knowledge.}
Fix any $x$ with $\|x\|\le R$.
The probing rule $T_i(x)=G$ uses median aggregation over the subset $G \subseteq [m]\setminus\{i\}$.
Write the peer deviations for $j\in G$:
$u_j:=\langle x,\theta_j^0\rangle - \langle x,\theta^\star\rangle$.
Since all learners in $G$ satisfy $\|\theta_j^0-\theta^\star\|\le r$,
all deviations $\{u_j\}_{j\in G}$ lie in $[-Rr,Rr]$.
Because $|G| > (m-1)/2$, the set $G$ contains more than half of the peers (excluding $i$),
and thus the median over $G$ obeys
$|\tilde y(x) - \langle x,\theta^\star\rangle|\le Rr$.
The remainder of the proof follows identically to case~(i):
\[
(\tilde y - y)^2
\;=\;\big(\tilde y - \langle x,\theta^\star\rangle + \langle x,\theta^\star\rangle - y\big)^2
\;\le\; 2\,(\tilde y - \langle x,\theta^\star\rangle)^2 \;+\; 2\,(\langle x,\theta^\star\rangle - y)^2
\;\le\; 2R^2 r^2 \;+\; 2\,(\langle x,\theta^\star\rangle - y)^2.
\]
Taking expectation over $(x,y)\sim\cP$ yields
$\E[(\tilde y - y)^2] \le 2R^2 r^2 + 2\,\E[(\langle x,\theta^\star\rangle - y)^2] = 2R^2 r^2 + 2\epsilon$.

\vspace{0.5em}

\textbf{(iv) Preference-aware.}
By the probing rule $T_i(x)=\{\pi(x)\}$, we have $\tilde y(x)=x^\top\bar\theta_{\pi(x)}$, so
\[
\E\big[(\tilde y - y)^2\big]
\;=\; \sum_{i=1}^m \alpha_i\,\E_{(x,y)\sim \cP_i}\!\big[(y - x^\top\bar\theta_i)^2\big].
\]
By optimality of $\bar\theta_i$ for the ERM objective on $\cP_i$, for every $i$ and any $\theta$ (in particular $\theta^\star$),
\[
\E_{\cP_i}\!\big[(y - x^\top\bar\theta_i)^2\big]
\;\le\;
\E_{\cP_i}\!\big[(y - x^\top\theta^\star)^2\big].
\]
Summing over $i$ with weights $\alpha_i$ gives
\[
\E\big[(\tilde y - y)^2\big]
\;\le\;
\sum_{i=1}^m \alpha_i\,\E_{\cP_i}\!\big[(y - x^\top\theta^\star)^2\big]
\;=\; \epsilon.
\]
\end{proof}

\subsection{Proof of \Cref{cor:sq-performance-bigO}}

\begin{proof}[Proof of Corollary~\ref{cor:sq-performance-bias}]
      Set $\lambda = \epsilon/\|\theta^\star\|^2$ in
      \Cref{thm:sq-performance-median-explicit}.
      Define
      \[
      S(\kappa,\lambda)
      :=
      (4b_\star + 6C_0^2)\sqrt{2\log(2/\kappa)}
      + 4(\Ymax+\Bth R)\Bth R
      + b_{\mathrm{u}}\sqrt{2\log(2/\kappa)},
      \]
      where $\Bth$ and $b_{\mathrm{u}}$ are the $\lambda$-dependent quantities from
      \Cref{thm:sq-performance-median-explicit}.
      If
      \[
      n \;\ge\; \underline n
        \;:=\;
        \frac{S(\kappa,\lambda)^2}{\epsilon^2},
      \]
      then the concentration terms in \Cref{thm:sq-performance-median-explicit}
      satisfy $S(\kappa,\lambda)/\sqrt{n}\le \epsilon$.
      With this choice and $\lambda\|\theta^\star\|^2=\epsilon$,
      the explicit bound yields
      \[
      \Risk(\tilde\theta_i)
      \;\le\;
      6B + \Bigl(4+\frac{2}{p}\Bigr)\epsilon + \epsilon + \epsilon
      \;=\; 6B + \Bigl(6+\frac{2}{p}\Bigr)\epsilon,
      \]
      which is $O\!\left(\bigl(\frac{p+1}{p}\bigr)\epsilon + B\right)$.

      To see the stated scaling of $\underline n$, note that
      $\Bth = \Theta(1/\sqrt{\lambda})$ and
      $b_{\mathrm{u}} = (\Ymax+\Bth R)^2 = O(1/\lambda)$
      with constants depending on $(R,\Ymax,M_0,p)$.
      To surface the dominant dependence on $(R,\Ymax,M_0,p)$, write
      $A := \Ymax^2 + p R^2 M_0^2$ so that
      \[
      \Bth R = R\sqrt{\frac{2A}{\lambda p}}
      \quad\text{and}\quad
      (\Bth R)^2 = \frac{2R^2A}{\lambda p}.
      \]
      The leading terms (in terms of $\epsilon$) in $S(\kappa,\lambda)$ scale as $(\Bth R)^2$,
      so one can take
      \[
      \underline n
      = O\!\left(
      \frac{R^4\|\theta^\star\|^4}{\epsilon^4}
      \Bigl(\frac{\Ymax^2}{p}+R^2M_0^2\Bigr)^2
      \log\frac{1}{\kappa}
      \right),
      \]
      where we suppress lower-order terms in $\epsilon$ coming from $b_\star$ and $C_0$.
      \end{proof}

\sqperformancebigO*
\begin{proof}
By \Cref{thm:sq-performance-median-explicit}, for any $\kappa\in(0,1)$ and
with probability at least $1-\kappa$,
\[
\Risk(\tilde\theta_i)
\;\le\;
6\,B
+ \Bigl(4+\frac{2}{p}\Bigr)\epsilon
+ \lambda\,\|\theta^\star\|^2
+ \frac{T(\kappa)}{\sqrt{n}},
\]
where every $n^{-1/2}$ contribution has been grouped into
\begin{align}
T(\kappa)
&:= \bigl(4b_\star + 6C_0^2\bigr)\sqrt{2\log(2/\kappa)}
   + b_{\mathrm{u}}\sqrt{2\log(2/\kappa)}
   + 4(\Ymax+\Bth R)\Bth R. \label{eq:Tkappa-def}
\end{align}
The terms that depend on $\lambda$ and $p$ are through $\Bth = \sqrt{2(\Ymax^2 + p\,\Ymaxprobe^2)/(\lambda\,p)}$, which also appears inside
$b_{\mathrm{u}}=(\Ymax+\Bth R)^2$.
Expanding the $\Bth$-dependent pieces, from \eqref{eq:Tkappa-def},
\[
T(\kappa)
= \bigl(4b_\star + 6C_0^2\bigr)\sqrt{2\log(2/\kappa)}
+ \underbrace{\Bigl[(\Ymax+\Bth R)^2\sqrt{2\log(2/\kappa)}
+ 4(\Ymax+\Bth R)\Bth R\Bigr]}_{\text{$\Bth$-dependent}}.
\]
Thus the $\lambda$-independent part is already $O\!\big(\sqrt{\log(1/\kappa)}\big)$,
so it remains to control the bracketed term.  Expanding gives
\begin{align*}
(\Ymax+\Bth R)^2\sqrt{2\log(2/\kappa)}
+ 4(\Ymax+\Bth R)\Bth R
= O\!\left(
\sqrt{\log(1/\kappa)}\,\bigl[\Ymax^2 + \Ymax\,\Bth R + \Bth^2 R^2\bigr]
\right).
\end{align*}
Plugging in the value of $\Bth$
from Lemma~\ref{lem:norm-bound-Btheta}
and assuming $\lambda < 1$, we can
combine with the $(\lambda,p)$-independent
part of $T(\kappa)$ to get
\begin{align*}
C_{\text{gen}} = R (\Ymax^2 + \,\Ymaxprobe^2) + 4b_\star + 6C_0^2
\end{align*}
\end{proof}

\begin{lemma}[Squared-loss performance Full Statement]
  \label{thm:sq-performance-median-explicit}
  Let \Cref{ass:accurate_probing} hold with parameter $B$.
  Then, for any $\kappa\in(0,1)$,
  with probability at least $1-\kappa$
  over the probing sample,
  every stationary point
  $\tilde\Theta$ of MSGD-P satisfies,
  for probing learner $i \in [m]$,
\begin{align*}
  \Risk(\tilde\theta_i)
  &\;\le\;
  6\,B
  + \Bigl(4+\frac{2}{p}\Bigr)\,\epsilon
  + \lambda\,\|\theta^\star\|^2 \\
  &\qquad + \frac{(4\,b_\star + 6\,C_0^2)\,\sqrt{2\log(2/\kappa)}}{\sqrt{n}}
  + \frac{4\,(\Ymax+\Bth R)\,\Bth R}{\sqrt{n}}
  + b_{\mathrm{u}}\,\sqrt{\frac{2\log(2/\kappa)}{n}},
\end{align*}
  where the constants are defined as follows:
  \begin{itemize}[leftmargin=2em]
  \item $\displaystyle C_0 \;:=\; \Ymax + \Ymaxprobe$.
  \item $\displaystyle \Bth \;:=\; \sqrt{\frac{2(\Ymax^2+p\,\Ymaxprobe^2)}{\lambda\,p}}$ is a radius (independent of $\theta^\star$) with
    $\displaystyle \Ymaxprobe \;:=\; R\,M_0,\ \ M_0\;:=\; \max_{j\ne i}\|\theta_j^0\|,$
    which bounds the pseudo-label magnitudes via $|\tilde y_i^q|\le \Ymaxprobe$.
  \item $\displaystyle b_\star \;:=\; \big(\Ymax + R\,\|\theta^\star\|\big)^2.$
  \item $\displaystyle b_{\mathrm{u}} \;:=\; \big(\Ymax + \Bth R\big)^2.$
  \end{itemize}
\end{lemma}

\begin{proof}
We work on the event $\mathcal{E}_{\mathrm{u}}\cap\mathcal{E}_\star$ where $\mathcal{E}_{\mathrm{u}}$ is the event of Lemma~\ref{lem:uniform-ball} (probability $\ge 1-\kappa/2$ with $\Bth$ from Lemma~\ref{lem:norm-bound-Btheta}) and $\mathcal{E}_\star$ is the event of Lemma~\ref{lem:one-point-theta-star} (probability $\ge 1-\kappa/2$).
By a union bound,
\[
\mathbb{P}(\mathcal{E}_{\mathrm{u}}\cap\mathcal{E}_\star)\ge 1-\kappa.
\]

\medskip
\noindent\textit{Bound on $\Lhat(\tilde\theta_i)$.}
By stationarity, $\tilde\theta_i$ minimizes
\[
\hat\Phi_i^{\mathrm{probe}}(\theta;\tilde\Theta)
=\tau\alpha_i\,\E_{\cP_i}[\ell(z,\theta)]
+(1-\tau)a_i(\tilde\Theta)\,\E_{\cD_i(\tilde\Theta)}[\ell(z,\theta)]
+p\,\Lhat(\theta)+\tfrac{\lambda\,p}{2}\|\theta\|^2.
\]
Thus, comparing the objective at $\tilde\theta_i$ and $\theta^\star$,
\[
p\,\Lhat(\tilde\theta_i)
\;\le\;
\tau\alpha_i\,\E_{\cP_i}[\ell(z,\theta^\star)]
+(1-\tau)a_i(\tilde\Theta)\,\E_{\cD_i(\tilde\Theta)}[\ell(z,\theta^\star)]
+p\,\Lhat(\theta^\star)+ \tfrac{\lambda\,p}{2}\|\theta^\star\|^2.
\]

Since
\[
\tau\alpha_i \E_{\cP_i}\ell(\theta^\star) + (1-\tau)a_i(\tilde\Theta)\E_{\cD_i(\tilde\Theta)}\ell(\theta^\star)\le \epsilon,
\]
dividing by $p$, we obtain
\[
\Lhat(\tilde\theta_i)
\;\le\; \frac{\epsilon}{p}
\;+\; \Lhat(\theta^\star)
\;+\; \frac{\lambda}{2}\,\|\theta^\star\|^2.
\]

By Lemma~\ref{lem:pseudo-to-true-sq} and Lemma~\ref{lem:one-point-theta-star},
\[
\Lhat(\theta^\star)\le 2\,\Ltrue(\theta^\star)+2\,\DeltaN^2 \le 2\,\epsilon + 2\,\Lamst + 2\,\DeltaN^2.
\]
Hence
\[
\Lhat(\tilde\theta_i)
\;\le\;
\frac{\epsilon}{p} + 2\,\epsilon + 2\,\Lamst + 2\,\DeltaN^2
\;+\; \frac{\lambda}{2}\,\|\theta^\star\|^2.
\]

\medskip
\noindent\textit{Bound on $\Ltrue(\tilde\theta_i)$.}
Applying Lemma~\ref{lem:pseudo-to-true-sq} again gives
\[
\Ltrue(\tilde\theta_i)
\;\le\; 2\,\Lhat(\tilde\theta_i) + 2\,\DeltaN^2.
\]
Substituting the bound from Step 1,
\[
\Ltrue(\tilde\theta_i)
      \;\le\;
\Bigl(\tfrac{2}{p}+4\Bigr)\epsilon \;+\; 4\,\Lamst \;+\; 6\,\DeltaN^2
\;+\; \lambda\,\|\theta^\star\|^2.
      \]

By Lemma~\ref{lem:deltaN-accurate-probing}, on $\mathcal{E}_\star$,
      \[
\DeltaN^2 \le B + C_0^2\,\sqrt{\frac{2\log(2/\kappa)}{n}}.
      \]
Therefore,
      \[
\Ltrue(\tilde\theta_i)
\;\le\;
6\,B \;+\; \Bigl(\tfrac{2}{p}+4\Bigr)\epsilon \;+\; 4\,\Lamst \;+\; 6\,C_0^2\,\sqrt{\frac{2\log(2/\kappa)}{n}}
\;+\; \lambda\,\|\theta^\star\|^2.
      \]

\medskip
\noindent\textit{Bound on $\Risk(\tilde\theta_i)$.}
Finally, on $\mathcal{E}_{\mathrm{u}}$ (with $\|\tilde\theta_i\|\le \Bth$ from Lemma~\ref{lem:norm-bound-Btheta}),
      \[
\Risk(\tilde\theta_i)\le \Ltrue(\tilde\theta_i) + \Lamu.
      \]
    \end{proof}

\begin{proof}[Alternative scaling with explicit $\lambda$]
If we keep $\lambda$ explicit and only choose $n$ to control the
concentration terms in \Cref{thm:sq-performance-median-explicit},
then for any target tolerance $\delta>0$ it suffices to take
\[
n \;\ge\; \bar n(\lambda,\delta)
  \;:=\;
  \frac{S(\kappa,\lambda)^2}{\delta^2},
\]
which yields
\[
\Risk(\tilde\theta_i)
\;\le\;
6B + \Bigl(4+\frac{2}{p}\Bigr)\epsilon
+ \lambda\|\theta^\star\|^2 + \delta.
\]
Writing $\Ymaxprobe = R M_0$ and $A := \Ymax^2 + p R^2 M_0^2$,
we have
\[
\Bth = \sqrt{\frac{2A}{\lambda p}}
\quad\text{and}\quad
(\Ymax+\Bth R)^2
\le 2\Ymax^2 + \frac{4R^2A}{\lambda p}.
\]
Using $b_\star = (\Ymax+R\|\theta^\star\|)^2$ and $C_0 = \Ymax + R M_0$,
this implies
\[
S(\kappa,\lambda)^2
= O\!\left(\log\frac{1}{\kappa}\right)
\left[
(\Ymax+R\|\theta^\star\|)^4
+(\Ymax+R M_0)^4
+\frac{R^4A^2}{\lambda^2 p^2}
\right],
\]
and therefore a sufficient choice is
\[
n
= O\!\left(
\frac{\log(1/\kappa)}{\delta^2}
\Bigl[(\Ymax+R\|\theta^\star\|)^4
+(\Ymax+R M_0)^4
+\frac{R^4(\Ymax^2+pR^2M_0^2)^2}{\lambda^2 p^2}\Bigr]
\right).
\]
Setting $\delta=\epsilon$ yields the same bias expression as above,
but now with $n$ explicit in $(R,\Ymax,M_0,p,\lambda,\epsilon,\kappa)$.
\end{proof}

      \begin{lemma}
\label{lem:pseudo-to-true-sq}
We have
      \[
\Ltrue(\theta)\;\le\;\Big(\sqrt{\Lhat(\theta)}+\DeltaN\Big)^2
\quad\text{and}\quad
\Lhat(\theta)\;\le\;\Big(\sqrt{\Ltrue(\theta)}+\DeltaN\Big)^2.
      \]
In particular, $\Ltrue(\theta)\le 2\,\Lhat(\theta)+2\,\DeltaN^2$ and
$\Lhat(\theta)\le 2\,\Ltrue(\theta)+2\,\DeltaN^2$.
      \end{lemma}

\begin{proof}[Proof of Lemma~\ref{lem:pseudo-to-true-sq}]
Let $X\in\mathbb{R}^{n\times d}$ be the design matrix, $\mathbf a:=X\theta-\tilde{\mathbf y}$ and $\mathbf b:=\tilde{\mathbf y}-\mathbf y$.
Then
\[
X\theta-\mathbf y=\mathbf a+\mathbf b.
\]
By the triangle inequality,
\[
\sqrt{\Ltrue(\theta)}
=\frac{\|X\theta-\mathbf y\|}{\sqrt{n}}
\le \frac{\|\mathbf a\|}{\sqrt{n}}+\frac{\|\mathbf b\|}{\sqrt{n}}
=\sqrt{\Lhat(\theta)}+\DeltaN.
\]
Squaring yields the first inequality; the second is analogous.
        \end{proof}

\begin{lemma}[One–point deviation at $\theta^\star$]
\label{lem:one-point-theta-star}
For any $\kappa\in(0,1)$, with probability at least $1-\kappa/2$,
\[
\big|\Ltrue(\theta^\star) - \Risk(\theta^\star)\big|
\;\le\;
\Lamst
\quad\text{where}\quad
\Lamst \;:=\; b_\star\,\sqrt{\frac{2\log(2/\kappa)}{n}},
\ \ b_\star \;:=\; \big(\Ymax + R\,\|\theta^\star\|\big)^2.
\]
\end{lemma}

\begin{proof}[Proof of Lemma~\ref{lem:one-point-theta-star}]
For each $q$, define the random variable
\[
Z_q:=(\langle \tilde x_i^q,\theta^\star\rangle-y_i^q)^2.
\]
Since $|y|\le \Ymax$ and $\|x\|\le R$ a.s., we have
\[
Z_q \in [0,\,(\Ymax+R\|\theta^\star\|)^2]=[0,b_\star].
\]
By Hoeffding's inequality, with probability at least $1-\kappa/2$,
\[
\bigg|\frac{1}{n}\sum_{q=1}^n Z_q - \mathbb{E}[Z]\bigg|
\le b_\star\sqrt{\frac{2\log(2/\kappa)}{n}}.
\]
\end{proof}

\begin{lemma}[Uniform deviation over $\{\|\theta\|\le \Bth\}$]
\label{lem:uniform-ball-expanded}
\label{lem:uniform-ball}
Fix any radius $\Bth>0$ and let $\kappa\in(0,1)$.
Define the squared‐loss class
\[
\cF_{\Bth}
\;:=\;
\Big\{\,f_\theta(x,y)\;=\;\big(y-\langle x,\theta\rangle\big)^2
\ \Big|\ \|\theta\|\le \Bth\Big\}.
\]
Then, with probability at least $1-\kappa/2$ (over the draw of the probing sample of size $n$),
\[
\sup_{\|\theta\|\le \Bth}\ \big|\Ltrue(\theta) - \Risk(\theta)\big|
\;\le\; \Lamu,
\]
where
\[
\Lamu \;:=\; \frac{4\,(\Ymax+\Bth R)\,\Bth R}{\sqrt{n}}
\;+\;
b_{\mathrm{u}}\,\sqrt{\frac{2\log(2/\kappa)}{n}},
\qquad
b_{\mathrm{u}}\;:=\;(\Ymax+\Bth R)^2.
\]
\end{lemma}

\begin{proof}
We apply Wainwright’s Theorem 4.10.

\textbf{Uniform bound b for $\cF_{\Bth}$.}
For any $\|\theta\|\le \Bth$ and any $(x,y)$ with $\|x\|\le R$, $|y|\le \Ymax$, we have
\[
\big|y - \langle x,\theta\rangle\big|
\;\le\;
|y| + |\langle x,\theta\rangle|
\;\le\;
\Ymax + \|x\|\,\|\theta\|
\;\le\;
\Ymax + \Bth R.
\]
Hence
\[
0 \ \le\ f_\theta(x,y)\;=\;\big(y-\langle x,\theta\rangle\big)^2
\ \le\ (\Ymax+\Bth R)^2\;=:\;b_{\mathrm{u}}.
\]
Thus $\cF_{\Bth}$ is $b_{\mathrm{u}}$–uniformly bounded.

\textbf{Bound $R_n(\cF_{\Bth})$ via contraction to the linear class.}
Fix a sample $S=\{(x_i,y_i)\}_{i=1}^n$ and let $\sigma_1,\dots,\sigma_n$ be i.i.d.\ Rademacher signs. The empirical Rademacher average of $\cF_{\Bth}$ on $S$ is
\[
\Rad_S(\cF_{\Bth})
\;:=\;
\E_\sigma\Big[\sup_{\|\theta\|\le\Bth}\ \frac{1}{n}\sum_{i=1}^n \sigma_i\big(y_i-\langle x_i,\theta\rangle\big)^2\Big].
\]
Define $u_{i,\theta}:=\langle x_i,\theta\rangle$ and, for each $i$, the recentered function
\[
\psi_i(u)
\;:=\;
\big(y_i-u\big)^2 - \big(y_i-0\big)^2
\;=\;
u^2 - 2\,y_i\,u,
\qquad \psi_i(0)=0.
\]
Since $\E[\sigma_i]=0$, the constant offsets $\{(y_i-0)^2\}$ vanish in the Rademacher average. Thus
\[
\Rad_S(\cF_{\Bth})
\;=\;
\E_\sigma\Big[\sup_{\|\theta\|\le\Bth}\ \frac{1}{n}\sum_{i=1}^n \sigma_i\,\psi_i\!\big(u_{i,\theta}\big)\Big].
\]

We now bound the Lipschitz constants of $\psi_i$ over the relevant range. For any $u,v\in[-\Bth R,\Bth R]$,
\[
\big|\psi_i(u)-\psi_i(v)\big|
=
\big|\, (u-v)\,\big[(u-y_i)+(v-y_i)\big]\,\big|
\ \le\
|u-v|\;\big(\,|u-y_i|+|v-y_i|\,\big)
\ \le\
2(\Ymax+\Bth R)\,|u-v|.
\]
Hence each $\psi_i$ is $L$–Lipschitz with
\[
L \;:=\; 2(\Ymax+\Bth R),
\qquad \psi_i(0)=0.
\]

By the Ledoux–Talagrand contraction inequality (applied elementwise to $\{\psi_i\}$ with common Lipschitz constant $L$ and $\psi_i(0)=0$), we have
\[
\Rad_S(\cF_{\Bth})
\;\le\;
L\cdot \E_\sigma\Big[\sup_{\|\theta\|\le\Bth}\ \frac{1}{n}\sum_{i=1}^n \sigma_i\,u_{i,\theta}\Big]
\;=\;
L\cdot \Rad_S(\cG_{\Bth}),
\]
where $\cG_{\Bth}:=\{(x,y)\mapsto \langle x,\theta\rangle:\ \|\theta\|\le\Bth\}$ is the linear class (note: dependence on $y$ disappears).

The empirical Rademacher average of $\cG_{\Bth}$ on $S$ is standard:
\[
\Rad_S(\cG_{\Bth})
=
\E_\sigma\Big[\sup_{\|\theta\|\le\Bth}\ \frac{1}{n}\sum_{i=1}^n \sigma_i\,\langle x_i,\theta\rangle\Big]
=
\frac{\Bth}{n}\,\E_\sigma\Big\|\sum_{i=1}^n \sigma_i x_i\Big\|
\ \le\
\frac{\Bth}{n}\,\sqrt{\E_\sigma\Big\|\sum_{i=1}^n \sigma_i x_i\Big\|^2}
=
\frac{\Bth}{n}\,\sqrt{\sum_{i=1}^n \|x_i\|^2}.
\]
Using $\|x_i\|\le R$, we conclude $\Rad_S(\cG_{\Bth})\le \Bth R/\sqrt{n}$. Combining,
\[
\Rad_S(\cF_{\Bth})
\ \le\ L\cdot \Rad_S(\cG_{\Bth})
\ \le\ 2(\Ymax+\Bth R)\;\frac{\Bth R}{\sqrt{n}}.
\]
Taking expectation over the sample (or simply noting that this bound holds for any $S$ satisfying $\|x_i\|\le R$) yields
\[
R_n(\cF_{\Bth})
\ \le\
2(\Ymax+\Bth R)\;\frac{\Bth R}{\sqrt{n}}.
\]

\textbf{Apply Theorem 4.10 and choose $\delta$.}
By Theorem 4.10, for any $\delta>0$,
\[
\sup_{\|\theta\|\le\Bth}\ \big|\Ltrue(\theta)-\Risk(\theta)\big|
\ \le\
2\,R_n(\cF_{\Bth}) + \delta
\ \le\
\frac{4(\Ymax+\Bth R)\,\Bth R}{\sqrt{n}} + \delta,
\]
with probability at least $1-\exp\!\big(-\tfrac{n\,\delta^2}{2 b_{\mathrm{u}}^2}\big)$.
Set $\delta\!=\!b_{\mathrm{u}}\,\sqrt{\tfrac{2\log(2/\kappa)}{n}}$ so that $\exp\!\big(-\tfrac{n\,\delta^2}{2 b_{\mathrm{u}}^2}\big)=\exp\!\big(-\log(2/\kappa)\big)=\kappa/2$.
The stated deviation bound $\Lamu$ follows.
  \end{proof}

\begin{lemma}[Concentration of $\Delta_n^2$ under Accurate Probing]
  \label{lem:deltaN-accurate-probing}
  Assume that $|\tilde y|\le \Ymaxprobe$ almost surely (e.g., $\Ymaxprobe \!=\! R\,\max_{j}\|\theta_j^0\|$ for linear peers at initialization).
  Then for any $\kappa\in(0,1)$, with probability at least $1-\kappa/2$,
  \[
  \Delta_n^2 \;\le\; B \;+\; C_0^2\,\sqrt{\frac{2\log(2/\kappa)}{n}},
  \qquad\text{where}\quad C_0 \;:=\; \Ymax + \Ymaxprobe.
  \]
      \end{lemma}

      \begin{proof}
  Let $W=(\tilde y - y)^2$. By the boundedness assumptions, $0\le W \le (\Ymax+\Ymaxprobe)^2=:U$ almost surely, and $\E[W]\le B$ by Accurate Probing.
  By Hoeffding’s inequality,
  \[
  \Pr\!\left(\frac{1}{n}\sum_{q=1}^n W_q - \E[W] \ge t\right)
  \;\le\; \exp\!\left(-\frac{2n t^2}{U^2}\right).
  \]
  Choosing $t = U\sqrt{\tfrac{2\log(2/\kappa)}{n}}$ yields
  $\Pr\!\big(\Delta_n^2 \ge \E[W] + U\sqrt{\tfrac{2\log(2/\kappa)}{n}}\big) \le \kappa/2$.
  Using $\E[W]\le B$ gives the stated bound with $C_0^2=U$.
  \end{proof}

\begin{lemma}[A priori norm bound and choice of $\Bth$]
\label{lem:norm-bound-Btheta}
For any probing learner $i\in U$,
\[
\|\tilde\theta_i\|
\;\le\;
\Bth
\quad\text{with}\quad
\Bth \;:=\; \sqrt{\frac{2(\Ymax^2+p\,\Ymaxprobe^2)}{\lambda\,p}}.
\]
\end{lemma}

\begin{proof}[Proof of Lemma~\ref{lem:norm-bound-Btheta}]
  At any stationary point $\tilde\Theta$, $\tilde\theta_i$ minimizes the convex function
  \[
  \hat\Phi_i^{\mathrm{probe}}(\theta;\tilde\Theta)
  \;:=\; \tau\alpha_i\,\mathbb{E}_{z\sim \cP_i}\big[\ell(z,\theta)\big]
  \;+\; (1-\tau)\,a_i(\tilde\Theta)\,\mathbb{E}_{z\sim \cD_i(\tilde\Theta)}\big[\ell(z,\theta)\big]
  \;+\; p\,\Lhat(\theta)
  \;+\; \frac{\lambda\,p}{2}\,\|\theta\|^2.
  \]
  By optimality of $\tilde\theta_i$ for $\hat\Phi_i^{\mathrm{probe}}(\cdot;\tilde\Theta)$,
\[
\hat\Phi_i^{\mathrm{probe}}(\tilde\theta_i;\tilde\Theta)\le \hat\Phi_i^{\mathrm{probe}}(0;\tilde\Theta).
\]
Using squared loss and $|y|\le \Ymax$, we have
\[
\Lhat(0)=\frac{1}{n}\sum_{q=1}^n (\tilde y_i^q)^2\le \Ymaxprobe^2.
\]
Hence,
\[
\frac{\lambda\,p}{2}\,\|\tilde\theta_i\|^2
\;\le\; \tau\alpha_i\,\Ymax^2
\;+\; (1-\tau)\,a_i(\tilde\Theta)\,\Ymax^2
\;+\; p\,\Lhat(0)
\;\le\; \Ymax^2 + p\,\Ymaxprobe^2.
\]
Thus
\[
\|\tilde\theta_i\|
\;\le\; \sqrt{\frac{2(\Ymax^2+p\,\Ymaxprobe^2)}{\lambda\,p}}
\;=\;\Bth.
\]
      \end{proof}

\section{Cross-Entropy Loss: Performance Guarantee}
\label{sec:ce_performance}

\paragraph{Notation.}
In this section, we provide a high-probability upper bound on the full-population
cross-entropy risk for a probing learner in MSGD-P at a stationary point. The proof
mirrors the squared-loss analysis, with key differences arising from the geometry of
the softmax and cross-entropy.

Throughout this section, let $K$ denote the number of classes. Let $\theta\equiv W\in\R^{K\times d}$
be the matrix of class-wise linear weights and define logits $z(x)=W x\in\R^K$ and
predicted probabilities $q_W(x)=\softmax(z(x))$. The multiclass cross-entropy loss is
$$
\ell_{\mathrm{CE}}((x,y),W)
\;=\;
- \sum_{c=1}^K y_c \,\log q_W(x)_c
\;=\;
- \log q_W(x)_Y,
$$
where $y\in\{e_1,\dots,e_K\}$ is one-hot and $Y$ is the true class index.
We assume $\|x\| \le R$ almost surely.

For a probing learner $i\in U$ (with probing weight $p>0$), we define its augmented
objective at $\Theta$ (as per MSGD-P) by
\begin{align*}
\Phi_i(W;\Theta)
\;=\;&\;
\tau\alpha_i\,\E_{(x,y)\sim\cP_i}\big[\ell_{\mathrm{CE}}((x,y),W)\big]
\\ &\;\;+\;
(1-\tau)\,a_i(\Theta)\,\E_{(x,y)\sim\cD_i(\Theta)}\big[\ell_{\mathrm{CE}}((x,y),W)\big]
\\ &\;\;+\;
p\,\widehat L_{\mathrm{probe}}(W)
\;+\;
\frac{\lambda\,p}{2}\,\|W\|_F^2,
\end{align*}
where
$$
\widehat L_{\mathrm{probe}}(W)
\;=\;
\frac{1}{n}\,\sum_{q=1}^n \ell_{\mathrm{CE}}\big((\tilde x_i^q,\tilde y_i^q),W\big)
\;=\;
-\frac{1}{n}\sum_{q=1}^n \sum_{c=1}^K \tilde y_{i,c}^q \log q_W(\tilde x_i^q)_c.
$$
Here $\{(\tilde x_i^q)\}_{q=1}^n$ are probing covariates drawn i.i.d.\ from $\cP_X$, and
$\{\tilde y_i^q\}_{q=1}^n$ are soft pseudo-labels. For analysis, let $y_i^q$ be the (hidden)
true labels of $\tilde x_i^q$, and define the empirical ``true'' probing CE:
$$
\widehat L_{\mathrm{true}}(W)
\;=\;
\frac{1}{n}\,\sum_{q=1}^n \ell_{\mathrm{CE}}\big((\tilde x_i^q,y_i^q),W\big)
\;=\;
-\frac{1}{n}\sum_{q=1}^n \log q_W(\tilde x_i^q)_{y_i^q}.
$$

Let $\theta^\star\in\arg\min_W \Risk(W)$ be a (finite-norm) population minimizer
of the unregularized CE risk $\Risk(W)=\E[\ell_{\mathrm{CE}}((x,y),W)]$, with
$\epsilon:=\Risk(\theta^\star)$ and $\|\theta^\star\|_F=:M_\star<\infty$.

\textbf{Aggregation Rule.}
we define peer logits $z_j(x, \Theta)=\theta_j^{0} x \in \R^K$
and the coordinatewise median of logits
$$
\tilde z_c(x, \Theta) := \operatorname{median}(z_{j,c}(x, \Theta) : j \in T_i(x)),\quad c\in[K].
$$
Then, we set the probing label as $\yAgg(x, \Theta)=\softmax(\tilde z(x, \Theta))$.

\subsection{Accurate Probing Assumption and Proofs}

\begin{table}[h]
      \centering
      \begin{tabular}{p{0.24\textwidth}p{0.18\textwidth}p{0.24\textwidth}}
      \toprule
      Scenario & $T_i(x)$ & $B$ \\
      \midrule
      Majority-good & $[m]$ & $\epsilon + 2 R r$ \\
      Market-leader & $\{j^\ast\}$ & $\xi$ \\
      Partial knowledge & $G$ & $\epsilon + 2 R r$ \\
      Preference-aware & $\{\pi(x)\}$ & $\epsilon$ \\
      \bottomrule
      \end{tabular}
      \vspace{0.5em}
      \caption{Probing accuracy parameters for cross-entropy loss.}
      \label{tab:accurate_probing_ce}
\end{table}

\begin{assumption}[Accurate Probing for CE]
\label{ass:accurate_probing_ce}
There exists $B_{\mathrm{CE}}\ge 0$ such that
$$
\E\big[\mathrm{CE}(y,\tilde y)\big]
\;=\;
\E\big[-\log \tilde y_Y\big]
    \;\le\;
B_{\mathrm{CE}},
$$
where the expectation is over $(x,y)\sim\cP$ and $\tilde y=\tilde y(x)$ produced by the robust aggregator.
\end{assumption}

We now present sufficient conditions ensuring Assumption~\ref{ass:accurate_probing_ce}.

\begin{lemma}[Sufficient conditions for Accurate Probing (CE)]
\label{lem:ce-accurate-probing-scenarios}
Assumption~\ref{ass:accurate_probing_ce} holds in each of the following scenarios:
\textbf{(i) Majority-good.}
Suppose strictly more than half of peers satisfy $\|\theta_j^0-\theta^\star\|_F\le r$.
Then
$$
B_{\mathrm{CE}}
\;=\;
\E\big[-\log \tilde y_Y\big]
\;\le\;
\epsilon \;+\; 2\,R\,r.
$$

\textbf{(ii) Market-leader.}
Suppose learner $i$ probes a single peer $j^\star$ with
$\E[\ell_{\mathrm{CE}}((x,y),\theta_{j^\star}^0)]\le \xi$. If
$\tilde y(x)=q_{\theta_{j^\star}^0}(x)$, then $B_{\mathrm{CE}}\le \xi$.

\textbf{(iii) Partial knowledge.}
If a known subset $G$ of peers satisfies $|G|>(m-1)/2$ and $\|\theta_j^0-\theta^\star\|_F\le r$ for all $j\in G$,
and the aggregator uses the coordinatewise median over $G$,
the same bound as in case~(i) holds.

\textbf{(iv) Preference-aware.}
Suppose each peer $j$ solves the ERM on its preference partition $\cP_j$:
$$
\bar\theta_j
\in
\arg\min_W\;\E_{(x,y)\sim\cP_j}\big[\ell_{\mathrm{CE}}((x,y),W)\big].
$$
If learner $i$ probes $\tilde y(x)=q_{\bar\theta_{\pi(x)}}(x)$, then
$$
B_{\mathrm{CE}}
\;\le\;
\sum_{j=1}^K \alpha_j\,\E_{\cP_j}\big[\ell_{\mathrm{CE}}((x,y),\bar\theta_j)\big]
\;\le\;
\epsilon.
$$
\end{lemma}
\begin{proof}
We treat each scenario in turn.

\textbf{(i) Majority-good.}
Fix any $x$ with $\|x\|\le R$ and write $z^\star(x)=\theta^\star x$.
      For any good peer $j$ (i.e., $\|\theta_j^0-\theta^\star\|_F\le r$) and any class $c$,
      $$
      |z_{j,c}(x) - z^\star_c(x)|
      \;=\;
      |(\theta_{j,c}^0 - \theta^\star_c)^\top x|
      \;\le\;
      \|\theta_{j,c}^0 - \theta^\star_c\|_2 \cdot \|x\|
      \;\le\;
      \|\theta_j^0-\theta^\star\|_F \cdot \|x\|
      \;\le\;
      r\,R.
      $$
      Since strictly more than half of the $\{z_{j,c}(x)\}_{j\ne i}$ lie in $[z^\star_c(x)\!-\!rR,\; z^\star_c(x)\!+\!rR]$, their coordinatewise median must also lie in this interval. Hence
      $$
      \|\tilde z(x) - z^\star(x)\|_\infty \;\le\; r\,R.
      $$
      Let $f(z)=-\log\softmax(z)_Y$. By Lemma~\ref{lem:ce-lipschitz-logits},
      $$
      -\log \tilde y_Y
      \;=\;
      f(\tilde z(x))
      \;\le\;
      f\big(z^\star(x)\big) + 2\,\|\tilde z(x) - z^\star(x)\|_\infty
      \;\le\;
      -\log q^\star_Y(x) + 2\,r\,R,
      $$
      where $q^\star=\softmax(z^\star)$. Taking expectations over $(x,y)\sim\cP$, we get
      $$
      B_{\mathrm{CE}}
      \;=\;
      \E\big[-\log \tilde y_Y\big]
      \;\le\;
      \E\big[-\log q^\star_Y(x)\big] + 2\,r\,R
      \;=\;
      \epsilon + 2\,r\,R.
      $$
\vspace{0.5em}

\textbf{(ii) Market-leader.}
Direct: $B_{\mathrm{CE}}=\E[-\log \tilde y_Y]=\E[\ell_{\mathrm{CE}}((x,y),\theta_{j^\star}^0)]\le \xi$.

\vspace{0.5em}

\textbf{(iii) Partial knowledge.}
The proof is identical to case~(i), restricting the median to $G$. Since $|G|>(m-1)/2$ and all learners in $G$ satisfy $\|\theta_j^0-\theta^\star\|_F\le r$, the coordinatewise median over $G$ satisfies the same bound.

\vspace{0.5em}

\textbf{(iv) Preference-aware.}
By optimality of $\bar\theta_j$ for the ERM objective,
$$
\E_{\cP_j}\big[\ell_{\mathrm{CE}}((x,y),\bar\theta_j)\big]
\;\le\;
\E_{\cP_j}\big[\ell_{\mathrm{CE}}((x,y),\theta^\star)\big].
$$
Summing over $j$ with weights $\alpha_j$ gives
$$
B_{\mathrm{CE}}
\;\le\; \sum_j \alpha_j \E_{\cP_j}\ell_{\mathrm{CE}}(\bar\theta_j)
\;\le\; \sum_j \alpha_j \E_{\cP_j}\ell_{\mathrm{CE}}(\theta^\star)
\;=\; \epsilon.
$$
\end{proof}

\begin{lemma}[CE is 2-Lipschitz in logits under $\|\cdot\|_\infty$]
      \label{lem:ce-lipschitz-logits}
      Fix a class $Y\in[K]$ and define $f(z):=-\log\softmax(z)_Y=-z_Y+\logsumexp(z)$ for $z\in\R^K$.
      Then for any $z,z'\in\R^K$,
      $$
      \big| f(z') - f(z) \big|
      \;\le\;
      2\,\|z'-z\|_\infty.
      $$
\end{lemma}
\begin{proof}
We have $\nabla f(z)=q(z)-e_Y$, where $q(z)=\softmax(z)$ and $e_Y$ is the one-hot at $Y$.
By the mean value inequality with dual norms $(\|\cdot\|_\infty,\|\cdot\|_1)$,
$$
| f(z') - f(z) |
\;\le\;
\sup_{\xi \in [z,z']} \|\nabla f(\xi)\|_1 \cdot \|z'-z\|_\infty
\;=\;
\sup_{\xi} \sum_{k=1}^K |q_k(\xi) - (e_Y)_k| \cdot \|z'-z\|_\infty.
$$
Since $y=e_Y$ is one-hot and $q\in\Delta^{K-1}$,
$
\sum_{k=1}^K |q_k - (e_Y)_k|
=
|q_Y-1| + \sum_{k\ne Y} q_k
=
(1-q_Y) + (1-q_Y)
\le 2.
$
Hence $|f(z')-f(z)|\le 2\,\|z'-z\|_\infty$.
\end{proof}

\subsection{Performance Bound}
We will use two probability floors:
$$
\gamma_\star \;:=\; \frac{1}{K}\,\exp\bigl(-2R\,\|\theta^\star\|_F\bigr),
\quad
\Gamma_\star \;:=\; \log\frac{1}{\gamma_\star}
\;=\;
\log K + 2 R \|\theta^\star\|_F,
$$
and
$$
B_\theta \;:=\; \sqrt{\frac{2(1+p)\log K}{\lambda\,p}},
\qquad
\gamma_B \;:=\; \frac{1}{K}\,\exp\bigl(-2R\,B_\theta\bigr),
\qquad
\Gamma_B \;:=\; \log\frac{1}{\gamma_B}
\;=\;
\log K + 2 R B_\theta.
$$

\begin{corollary}[Big-$O$ summary]
      \label{cor:ce-performance-bigO}
      Under the assumptions of Theorem~\ref{thm:ce-performance}, with probability at least $1-\kappa$,
      $$
      \Risk(\tilde\theta_i) \le O\!\Big(
      \Bigl(\frac{p+1}{p}\Bigr)\epsilon
      \;+\;
      \lambda \,\|\theta^\star\|_F^2
      \;+\;
      C_{\text{bias}}\,\sqrt{B_{\mathrm{CE}}}
      \;+\;
      \frac{(p+1)\,C_{\text{gen}}}{p}\,\sqrt{\tfrac{\log(1/\kappa)}{\lambda n}}
      \Big),
      $$
      where the constant $C_{\text{bias}} = R\left(\|\theta^\star\|_F+\sqrt{2\log K/{\lambda}}\right)$ and $C_{\text{gen}} = C_{\text{gen}} (R, \|\theta^\star\|, \log(K))$ is specified in the appendix.
\end{corollary}

\begin{theorem}[CE performance bound with probing]
\label{thm:ce-performance}
Let Assumptions~\ref{ass:bounded_support}, \ref{ass:learning_rate}, \ref{ass:loss_measure},
\ref{ass:bounded_iterates}, and \ref{ass:accurate_probing_ce} hold. Let $\tilde\Theta$ be a
stationary point of MSGD-P and let $\tilde\theta_i$ be the parameter for a probing learner
$i\in U$. Then for any $\kappa\in(0,1)$, with probability at least $1-\kappa$ over the probing
sample,
\begin{align*}
\Risk(\tilde\theta_i) \le\, &\Bigl(1+\frac{1}{p}\Bigr)\epsilon + \frac{\lambda}{2}\|\theta^\star\|_F^2
+ R(\|\theta^\star\|_F+B_\theta)\sqrt{2B_{\mathrm{CE}}} + 4R\sqrt{\frac{(1+p)K\log K}{\lambda\,p\,n}} \\
&+ \bigl(\Gamma_\star+\Gamma_B+R(\|\theta^\star\|_F+B_\theta)\bigr)\sqrt{\frac{\log(3/\kappa)}{2n}}.
\end{align*}
\end{theorem}

\begin{proof}
We work on the event $\mathcal{E}_{\mathrm{pinsker}}\cap\mathcal{E}_\star\cap\mathcal{E}_{\mathrm{u}}$ where:
\begin{itemize}
\item $\mathcal{E}_{\mathrm{pinsker}}$ is the event of Lemma~\ref{lem:ce-pinsker-conc} (probability $\ge 1-\kappa/3$),
\item $\mathcal{E}_\star$ is the event of Lemma~\ref{lem:ce-star-conc} (probability $\ge 1-\kappa/3$),
\item $\mathcal{E}_{\mathrm{u}}$ is the event of Lemma~\ref{lem:ce-uniform} (probability $\ge 1-\kappa/3$).
\end{itemize}
By a union bound,
$$
\mathbb{P}(\mathcal{E}_{\mathrm{pinsker}}\cap\mathcal{E}_\star\cap\mathcal{E}_{\mathrm{u}})\ge 1-\kappa.
$$

\medskip
\noindent\textit{Bound on $\widehat L_{\mathrm{probe}}(\tilde\theta_i)$.}
By Lemma~\ref{lem:ce-stationarity}, dividing by $p$,
$$
\widehat L_{\mathrm{probe}}(\tilde\theta_i)
\;\le\;
\frac{\epsilon}{p} \;+\; \widehat L_{\mathrm{probe}}(\theta^\star)
\;+\; \frac{\lambda}{2}\,\|\theta^\star\|_F^2.
$$
By Corollary~\ref{cor:ce-bridge-empirical} applied at $W=\theta^\star$ and Lemma~\ref{lem:ce-star-conc},
$$
\widehat L_{\mathrm{probe}}(\theta^\star)
\;\le\;
\widehat L_{\mathrm{true}}(\theta^\star)
\;+\;
R\,\|\theta^\star\|_F\,\Delta_{1,n}
\;\le\;
\epsilon
\;+\;
\Gamma_\star\,\sqrt{\frac{\log(3/\kappa)}{2n}}
\;+\;
R\,\|\theta^\star\|_F\,\Delta_{1,n}.
$$
Hence
$$
\widehat L_{\mathrm{probe}}(\tilde\theta_i)
\;\le\;
\frac{\epsilon}{p} + \epsilon
\;+\;
\frac{\lambda}{2}\,\|\theta^\star\|_F^2
\;+\;
\Gamma_\star\,\sqrt{\frac{\log(3/\kappa)}{2n}}
\;+\;
R\,\|\theta^\star\|_F\,\Delta_{1,n}.
$$

\medskip
\noindent\textit{Bound on $\widehat L_{\mathrm{true}}(\tilde\theta_i)$.}
By Corollary~\ref{cor:ce-bridge-empirical} at $W=\tilde\theta_i$,
$$
\widehat L_{\mathrm{true}}(\tilde\theta_i)
\;\le\;
\widehat L_{\mathrm{probe}}(\tilde\theta_i)
\;+\;
R\,\|\tilde\theta_i\|_F\,\Delta_{1,n}.
$$
Substituting the bound from the previous step,
$$
\widehat L_{\mathrm{true}}(\tilde\theta_i)
\;\le\;
\frac{\epsilon}{p} + \epsilon
\;+\;
\frac{\lambda}{2}\,\|\theta^\star\|_F^2
\;+\;
\Gamma_\star\,\sqrt{\frac{\log(3/\kappa)}{2n}}
\;+\;
R\,(\|\theta^\star\|_F+\|\tilde\theta_i\|_F)\,\Delta_{1,n}.
$$

On $\mathcal{E}_{\mathrm{pinsker}}$, by Lemma~\ref{lem:ce-pinsker-conc},
$$
\Delta_{1,n}
\;\le\;
\sqrt{2B_{\mathrm{CE}}}
\;+\;
\sqrt{\frac{\log(3/\kappa)}{2n}}.
$$
  Therefore,
$$
\widehat L_{\mathrm{true}}(\tilde\theta_i)
\;\le\;
\frac{\epsilon}{p} + \epsilon
\;+\;
\frac{\lambda}{2}\,\|\theta^\star\|_F^2
\;+\;
R\,(\|\theta^\star\|_F + B_\theta)\,\sqrt{2B_{\mathrm{CE}}}
\;+\;
\Bigl(\Gamma_\star + R\,(\|\theta^\star\|_F + B_\theta)\Bigr)\,\sqrt{\frac{\log(3/\kappa)}{2n}}.
$$

\medskip
\noindent\textit{Bound on $\Risk(\tilde\theta_i)$.}
By Lemma~\ref{lem:ce-uniform}, on $\mathcal{E}_{\mathrm{u}}$ and since
$\|\tilde\theta_i\|\le B_\theta$ by Lemma~\ref{lem:ce-norm-floor},
$$
\Risk(\tilde\theta_i)
\;\le\;
\widehat L_{\mathrm{true}}(\tilde\theta_i)
\;+\;
2\sqrt{2}\,R B_\theta \sqrt{\frac{K}{n}}
\;+\;
\Gamma_B\,\sqrt{\frac{\log(3/\kappa)}{2n}}.
$$
Combining with the previous bound and substituting $B_\theta=\sqrt{2(1+p)\log K/(\lambda\,p)}$
(which yields $2\sqrt{2}RB_\theta\sqrt{K/n} = 4R\sqrt{(1+p)(K\log K)/(\lambda\,p\,n)}$) gives the stated bound.
\end{proof}

\begin{lemma}[Stationarity bound]
\label{lem:ce-stationarity}
At a stationary point $\tilde\Theta$, for learner $i\in U$,
$$
p\,\widehat L_{\mathrm{probe}}(\tilde\theta_i)
\;\le\;
\epsilon \;+\; p\,\widehat L_{\mathrm{probe}}(\theta^\star)
\;+\; \frac{\lambda\,p}{2}\,\|\theta^\star\|_F^2.
$$
\end{lemma}

\begin{proof}[Proof of Lemma~\ref{lem:ce-stationarity}]
By optimality of $\tilde\theta_i$ for $\Phi_i(\cdot;\tilde\Theta)$,
$$
\Phi_i(\tilde\theta_i;\tilde\Theta)
\;\le\;
\Phi_i(\theta^\star;\tilde\Theta).
$$
Subtracting the two population CE terms at $W=\tilde\theta_i$ and $W=\theta^\star$,
and using that
$$
\tau\alpha_i\,\E_{\cP_i}\ell_{\mathrm{CE}}(\theta^\star)
\;+\;
(1-\tau)\,a_i(\tilde\Theta)\,\E_{\cD_i(\tilde\Theta)}\ell_{\mathrm{CE}}(\theta^\star)
\;\le\; \epsilon,
$$
we obtain the stated bound.
\end{proof}

\begin{lemma}[Norm bound and probability floor]
\label{lem:ce-norm-floor}
We have $\|\tilde\theta_i\|_F \le B_\theta = \sqrt{2(1+p)\log K/(\lambda\,p)}$. Consequently,
for any $x$,
$$
\min_{c\in[K]} q_{\tilde\theta_i}(x)_c
\;\ge\; \gamma_B \;=\; \frac{1}{K}\exp\bigl(-2R B_\theta\bigr),
\quad
\Gamma_B \;=\; \log\frac{1}{\gamma_B} \;=\; \log K + 2 R B_\theta.
$$
\end{lemma}

\begin{proof}[Proof of Lemma~\ref{lem:ce-norm-floor}]
Compare $\Phi_i(\tilde\theta_i;\tilde\Theta)$ to $\Phi_i(0;\tilde\Theta)$. For $W=0$,
$q_W$ is uniform and each CE term equals $\log K$. Thus
$$
\frac{\lambda\,p}{2}\,\|\tilde\theta_i\|_F^2
\;\le\;
(1+p)\,\log K
\quad\Rightarrow\quad
\|\tilde\theta_i\|_F \;\le\; \sqrt{\frac{2(1+p)\log K}{\lambda\,p}}.
$$
For the floor, write for any $c,c'$:
$$
|z_c - z_{c'}|
\;=\;
|(W_c-W_{c'})^\top x|
\;\le\; \|W_c-W_{c'}\|\,\|x\|
\;\le\; 2\,\|W\|_F\,\|x\|
\;\le\; 2 R \|W\|_F.
$$
Hence $q_W(x)_c \ge \frac{1}{K}\exp(-2R\|W\|_F)$.
\end{proof}

\begin{lemma}[Pseudo-label discrepancy]
\label{lem:ce-pinsker-conc}
Let $\Delta_{1,n}=\frac{1}{n}\sum_{q=1}^n \|\tilde y_i^q - y_i^q\|_1$. Under Assumption~\ref{ass:accurate_probing_ce}, for any $\kappa\in(0,1)$, with probability at least $1-\kappa/3$,
$$
\Delta_{1,n}
\;\le\;
\sqrt{2B_{\mathrm{CE}}}
\;+\;
\sqrt{\frac{\log(3/\kappa)}{2n}}.
$$
\end{lemma}

\begin{proof}[Proof of Lemma~\ref{lem:ce-pinsker-conc}]
For one-hot $y$, $\mathrm{CE}(y,\tilde y)=\mathrm{KL}(y\|\tilde y)$ and Pinsker gives
$\E\|y-\tilde y\|_1 \le \sqrt{2\,\E\mathrm{KL}(y\|\tilde y)} \le \sqrt{2 B_{\mathrm{CE}}}$.
Since $\|y-\tilde y\|_1\in[0,2]$, Hoeffding yields the stated bound.
\end{proof}

\begin{lemma}[Concentration at $\theta^\star$]
\label{lem:ce-star-conc}
With probability at least $1-\kappa/3$,
$$
\big|\widehat L_{\mathrm{true}}(\theta^\star)-\epsilon\big|
\;\le\;
\Gamma_\star\,\sqrt{\frac{\log(3/\kappa)}{2n}},
$$
where $\Gamma_\star=\log K + 2R\|\theta^\star\|_F$.
\end{lemma}

\begin{proof}[Proof of Lemma~\ref{lem:ce-star-conc}]
By Lemma~\ref{lem:ce-norm-floor} applied to $W=\theta^\star$,
$\min_c q_{\theta^\star}(x)_c \ge \gamma_\star$, hence
$\ell_{\mathrm{CE}}((x,y),\theta^\star)\in [0,\Gamma_\star]$.
Hoeffding's inequality gives the result.
\end{proof}

\begin{lemma}[Uniform convergence]
\label{lem:ce-uniform}
With probability at least $1-\kappa/3$,
$$
\sup_{\|W\|_F\le B_\theta}\,\big|\widehat L_{\mathrm{true}}(W) - \Risk(W)\big|
\;\le\;
2\sqrt{2}\,R\,B_\theta\,\sqrt{\frac{K}{n}}
\;+\;
\Gamma_B\,\sqrt{\frac{\log(3/\kappa)}{2n}}.
$$
\end{lemma}

\begin{proof}[Proof of Lemma~\ref{lem:ce-uniform}]
For any sample $\{(x_i,y_i)\}_{i=1}^n$ with $\|x_i\|\le R$, we bound the empirical
Rademacher complexity of
$$
\mathcal{F}_{B_\theta}
\;:=\;
\big\{(x,y)\mapsto \ell_{\mathrm{CE}}((x,y),W): \|W\|_F\le B_\theta\big\}.
$$
The function $f(z,y)=-\log\softmax(z)_Y$ has gradient $\nabla_z f = q - y$ with
$\|\nabla_z f\|_2=\|q-y\|_2\le \sqrt{2}$. Thus $f$ is $\sqrt{2}$-Lipschitz in $z$.
By the vector contraction lemma,
$$
\mathfrak{R}_n(\mathcal{F}_{B_\theta})
\;\le\;
\sqrt{2}\cdot
\E\bigg[
\sup_{\|W\|_F\le B_\theta}
\frac{1}{n}\sum_{i=1}^n \sum_{k=1}^K \sigma_{i,k}\,(W x_i)_k
\bigg],
$$
with $\sigma_{i,k}$ i.i.d.\ Rademacher. The inner supremum equals
$$
\frac{B_\theta}{n}\,\E\Big\|
\sum_{i=1}^n \sum_{k=1}^K \sigma_{i,k}\,e_k x_i^\top
\Big\|_F
\;\le\;
\frac{B_\theta}{n}\,\sqrt{
\E\Big\|
\sum_{i,k} \sigma_{i,k}\,e_k x_i^\top
\Big\|_F^2
}
\;=\;
\frac{B_\theta}{n}\,\sqrt{
\sum_{k=1}^K \sum_{i=1}^n \|x_i\|^2
}
\;\le\;
B_\theta R\,\sqrt{\frac{K}{n}}.
$$
Hence $\mathfrak{R}_n(\mathcal{F}_{B_\theta})\le \sqrt{2}\,B_\theta R \sqrt{K/n}$.
A standard symmetrization and bounded-difference argument (the per-sample loss range
over $\|W\|\le B_\theta$ is at most $\Gamma_B$ by Lemma~\ref{lem:ce-norm-floor})
yields, with prob.\ $\ge 1-\kappa/3$,
$$
\sup_{\|W\|_F\le B_\theta}\,\big|\widehat L_{\mathrm{true}}(W) - \Risk(W)\big|
\;\le\;
2\,\mathfrak{R}_n(\mathcal{F}_{B_\theta})
\;+\;
\Gamma_B\,\sqrt{\frac{\log(3/\kappa)}{2n}},
$$
which gives the stated bound.
\end{proof}

\begin{lemma}[Logits bridge for cross-entropy; no probability floors]
  \label{lem:ce-bridge}
  Let $W\in\R^{K\times d}$, $z(x)=W x\in\R^K$, and $q_W(x)=\softmax(z(x))$.
  For any $x\in\R^d$, any one-hot label $y\in\{e_1,\dots,e_K\}$, and any
  soft pseudo-label $\tilde y\in\Delta^{K-1}$ (i.e., $\tilde y_c\ge 0$, $\sum_c \tilde y_c=1$),
  the cross-entropy difference satisfies the exact identity
  $$
  \operatorname{CE}(\tilde y,q_W) - \operatorname{CE}(y,q_W)
  \;=\;
  -\,\langle \tilde y - y,\; z(x)\rangle.
  $$
  Consequently,
  $$
  \big|\operatorname{CE}(\tilde y,q_W) - \operatorname{CE}(y,q_W)\big|
  \;\le\;
  \|\tilde y - y\|_1 \cdot \|z(x)\|_\infty
  \;\le\;
  \|\tilde y - y\|_1 \cdot R\,\|W\|_F.
  $$
  \end{lemma}

  \begin{proof}
  Recall that $\operatorname{CE}(r,q)= - \sum_{c=1}^K r_c \log q_c$ for any probability vector $r$,
  and $q_c = \frac{e^{z_c}}{\sum_{j} e^{z_j}}$ with $z=W x$.
  Hence
  $$
  -\log q_c
  \;=\;
  \log\Big(\sum_{j=1}^K e^{z_j}\Big) - z_c
  \;=\;
  \logsumexp(z) - z_c.
  $$
  Therefore, for any $r\in\Delta^{K-1}$,
  $$
  \operatorname{CE}(r,q_W)
  \;=\;
  \sum_{c=1}^K r_c\,\big(\logsumexp(z) - z_c\big)
  \;=\;
  \logsumexp(z)\cdot \underbrace{\sum_{c=1}^K r_c}_{=\,1}
  \;-\;\sum_{c=1}^K r_c\,z_c
  \;=\;
  \logsumexp(z) \;-\; \langle r, z\rangle.
  $$
  Applying this twice with $r=\tilde y$ and $r=y$ gives
  $$
  \operatorname{CE}(\tilde y,q_W) - \operatorname{CE}(y,q_W)
  \;=\;
  \big(\logsumexp(z) - \langle \tilde y, z\rangle\big)
  \;-\;
  \big(\logsumexp(z) - \langle y, z\rangle\big)
  \;=\;
  -\,\langle \tilde y - y, z\rangle,
  $$
  which is the claimed identity.

  For the inequalities, use Hölder and the fact that $\|z\|_\infty=\max_c |z_c|$:
  $$
  \big|\operatorname{CE}(\tilde y,q_W) - \operatorname{CE}(y,q_W)\big|
  \;=\;
  |\langle \tilde y - y, z\rangle|
  \;\le\;
  \|\tilde y - y\|_1 \,\|z\|_\infty.
  $$
  Finally, since $z_c = W_c^\top x$ and $\|x\|\le R$,
  $$
  \|z\|_\infty
  \;=\;
  \max_c |W_c^\top x|
  \;\le\;
  \max_c \|W_c\|_2 \cdot \|x\|
  \;\le\;
  \Big(\max_c \|W_c\|_2\Big)\, R
  \;\le\;
  R\,\|W\|_F,
  $$
  because $\|W\|_F^2 = \sum_c \|W_c\|_2^2 \ge \max_c \|W_c\|_2^2$.
  Combining the bounds yields the result.
  \end{proof}

  \begin{lemma}[Empirical bridge on the probing batch]
  \label{cor:ce-bridge-empirical}
  On the probing dataset $\{(\tilde x_i^q,\tilde y_i^q,y_i^q)\}_{q=1}^n$, define
  $$
  \widehat L_{\mathrm{probe}}(W) \;=\; \frac{1}{n}\sum_{q=1}^n \operatorname{CE}(\tilde y_i^q, q_W(\tilde x_i^q)),
  \quad
  \widehat L_{\mathrm{true}}(W) \;=\; \frac{1}{n}\sum_{q=1}^n \operatorname{CE}(y_i^q, q_W(\tilde x_i^q)),
  \quad
  \Delta_{1,n} \;=\; \frac{1}{n}\sum_{q=1}^n \|\tilde y_i^q - y_i^q\|_1.
  $$
  Then for any $W$,
  $$
  \big| \widehat L_{\mathrm{probe}}(W) - \widehat L_{\mathrm{true}}(W) \big|
  \;\le\;
  R\,\|W\|_F \cdot \Delta_{1,n}.
  $$
  \end{lemma}

  \begin{proof}
  Apply Lemma~\ref{lem:ce-bridge} termwise and average.
  \end{proof}

\section{Experimental Details and
Additional Results}
\label{sec:app_numerical}

\subsection{Dataset and Preprocessing Details}

\paragraph{MovieLens-10M \citep{harper2015movielens}.}
This dataset contains 10 million movie ratings from 70k users across 10k movies,
providing a natural testbed for multi-learner competition in recommendation.
Following \citet{bose2023initializing} and \citet{Su2024-qw},
we extract $d=16$ dimensional user embeddings via matrix factorization
and retain ratings for the top 200 most-rated movies,
yielding a population of 69,474 users.
Each user's data consists of $z = (x, r)$ where $x \in \mathbb{R}^d$
is the embedding and $r$ contains their ratings.
Let $\Omega_x$ denote the set of movies rated by user $x$ with $|\Omega_x|$ movies.
Each learner fits a linear model $\theta \in \mathbb{R}^{d \times 200}$ using squared loss:
\begin{align*}
\ell(z; \theta) = \frac{1}{|\Omega_x|} \sum_{i \in \Omega_x} (\theta_i^\top x - r_i)^2.
\end{align*}

\paragraph{ACS Employment \citep{ding2021retiring}.}
We use the ACSEmployment task from \texttt{folktables},
where the goal is to predict employment status from demographic features.
The population consists of 38,221 individuals from the 2018 Alabama census (ages 16--90),
with $d=16$ features describing age, education, marital status, etc.
Each user's data is $z = (x, y)$ where $x \in \mathbb{R}^d$
(standardized to zero mean, unit variance) and $y \in \{0,1\}$.
Each learner uses logistic regression:
\begin{align*}
\ell(z; \theta) = -y \log(\sigma(\theta^\top x)) - (1-y) \log(1 - \sigma(\theta^\top x)),
\end{align*}
where $\sigma$ is the sigmoid function. The model predicts $\hat{y} = \mathbf{1}[\theta^\top x > 0]$.

\paragraph{Amazon Reviews 2023.}
We use the \texttt{McAuley-Lab/Amazon-Reviews-2023} corpus (via HuggingFace),
constructing a binary sentiment task from review text and star ratings.
We sample up to 30{,}000 reviews from nine product categories and define
labels by $y=\mathbf{1}[\text{rating}\ge 4]$ (so 1--3 stars
are negative, 4--5 stars are positive). Features are $d=384$ dimensional
sentence embeddings of the review text produced by \texttt{all-MiniLM-L6-v2},
with a stratified 95/5 train-test split.
Each learner uses the same logistic regression setup as Census.

\paragraph{User Preferences.}
For Census and MovieLens, we simulate a market with $m=5$ learners
and model inherent user preferences $\pi(z)$ via K-means clustering ($K=5$) on user features,
assigning each user's preferred platform based on their cluster membership.
For Amazon, we use category-based partitions ($m=9$ learners, one per product category):
each review is assigned to its category group, and this group index induces the preference partition $\{S_i\}_{i=1}^m$.
In all cases, the partition captures the intuition that users from different demographic
or behavioral segments may have systematic affinities for different platforms.

\paragraph{Train/test splits.}
Each experiment constructs one held-out test set at the start and uses the
remaining users as the shared training pool for all learners. Census uses a
fixed 99/1 split with random seed 0; Amazon uses a stratified 95/5 split over
sentiment labels; MovieLens uses a 90/10 user split. We do not use a validation
split because the experimental hyperparameters are fixed in advance rather
than tuned per run.
Dataset-specific hyperparameters are summarized in Table~\ref{tab:hyperparameters}.

\paragraph{Census neural experiment.}
The neural Census experiment in Figure~\ref{fig:census_nn} changes only the
learner family relative to the linear Census experiments. Each learner is a
two-layer ReLU MLP with architecture $d\!\to\!64\!\to\!1$, trained with SGD on
binary cross-entropy. The held-out test split, K-means preference clusters, and
cluster-induced learner rankings are reused unchanged. In the bad-outcome run,
learners are randomly initialized, $\tau=0.3$, and no learner probes. In the
good-outcome run, learners are initialized by partition pretraining on the users
who rank them first, $\tau=0.7$, and Learner 2 probes. Probing is offline:
pseudo-labels are collected once from the initialized peer selected by the
preference-aware rule $T_i(x)=\{\pi(x)\}$, with $\kappa=0$, and the probing
learner then trains on the weighted combination of organic and pseudo-labeled
losses. We use $T=1000$ MSGD iterations, three random seeds, and the probing
weights shown in Figure~\ref{fig:census_nn}.

\subsection{Additional Experiments}

In this section, we provide additional experiments, which
elaborate on \Cref{sec:numerical}.

\paragraph{Expt 1 for other scenarios}
We replicate Expt~1 from \Cref{sec:numerical} in the two globally-good settings from
\Cref{def:globally_good_scenarios}: market-leader and majority-good.
As in the preference-aware case, standard MSGD without probing
($p=0$, $\tau=0.3$) converges to equilibria with large full-population
gaps to the dashed black baseline (\Cref{fig:bad_market_leader,fig:bad_half_majority}).
In the market-leader setting (\Cref{fig:bad_market_leader}),
the most overspecialized learner is about $0.32$ below baseline on Census
($\approx 0.47$ vs $\approx 0.79$), and more than $2.6$ MSE above baseline on MovieLens
($>5.6$ vs $\approx 2.95$).
In the majority-good setting (\Cref{fig:bad_half_majority}),
two learners remain overspecialized with sizable baseline gaps
(Census around $0.22$--$0.25$ below baseline; MovieLens around $2.0$--$2.9$ above baseline).
These results show that poor global equilibria are not unique to the
preference-aware scenario.

\paragraph{Expt 2 for other scenarios}
We next replicate Expt~2 for the other scenarios.
In the market-leader scenario, Learner $4$ probes the known leader
(Learner $0$), and its final Census accuracy improves from about $0.55$
to about $0.75$, while its MovieLens loss drops from about $5.1$
to about $3.1$ as $p$ increases
(\Cref{fig:good_market_leader}).
Relative to the dashed baseline, this closes most of the initial gap
(roughly from $0.24$ to $0.04$ on Census, and from about $2.2$ to about $0.1$ on MovieLens).
In the majority-good scenario, where Learner $4$ probes via
median aggregation over all peers, we observe a similar pattern:
Census accuracy rises from about $0.52$ to about $0.77$, and MovieLens
loss decreases from about $4.7$ to about $3.1$
(\Cref{fig:good_half_majority}).
This again closes a large fraction of the baseline gap
(roughly from $0.27$ to $0.02$ on Census, and from about $1.7$ to about $0.1$ on MovieLens).
Well-performing learners change only modestly, indicating that probing
primarily benefits the underperforming learner and mitigates overspecialization.

\paragraph{Expt 4: Impact of noise in selection of probed labels}
Finally, we test robustness to noisy probing-source selection.
For each probe query $x$, the probing learner queries $\pi(x)$ with probability
$1-\kappa$, and with probability $\kappa$ it queries a random other learner.
Across Census, Amazon, and MovieLens (\Cref{fig:kappa}), increasing $\kappa$
causes only mild changes in the probing learner's final performance relative
to the low-noise case, while preserving strong gains from probing.
Thus, our method is robust to imperfect estimates of the ranking function.

\paragraph{Expt 5: What happens when multiple learners probe?}
We also evaluate a preference-aware setting where multiple learners probe
simultaneously. In \Cref{fig:multi_probe}, the triangle-marked learners
(Learners $2$ and $3$) both probe while the dashed black line indicates the
full-data baseline. As $p$ increases, Learner $2$ improves from about $0.60$
to about $0.78$, and Learner $3$ improves from about $0.66$ to about $0.79$.
Equivalently, their baseline gaps shrink from roughly $0.19$ and $0.13$ at
$p=0$ to about $0.01$ and near zero at $p=0.8$. The non-probing learners move
only slightly, indicating that simultaneous probing remains stable and still
helps underperforming learners recover most of the overspecialization gap.

\begin{figure}[t]
    \centering
    \includegraphics[width=0.5\textwidth]{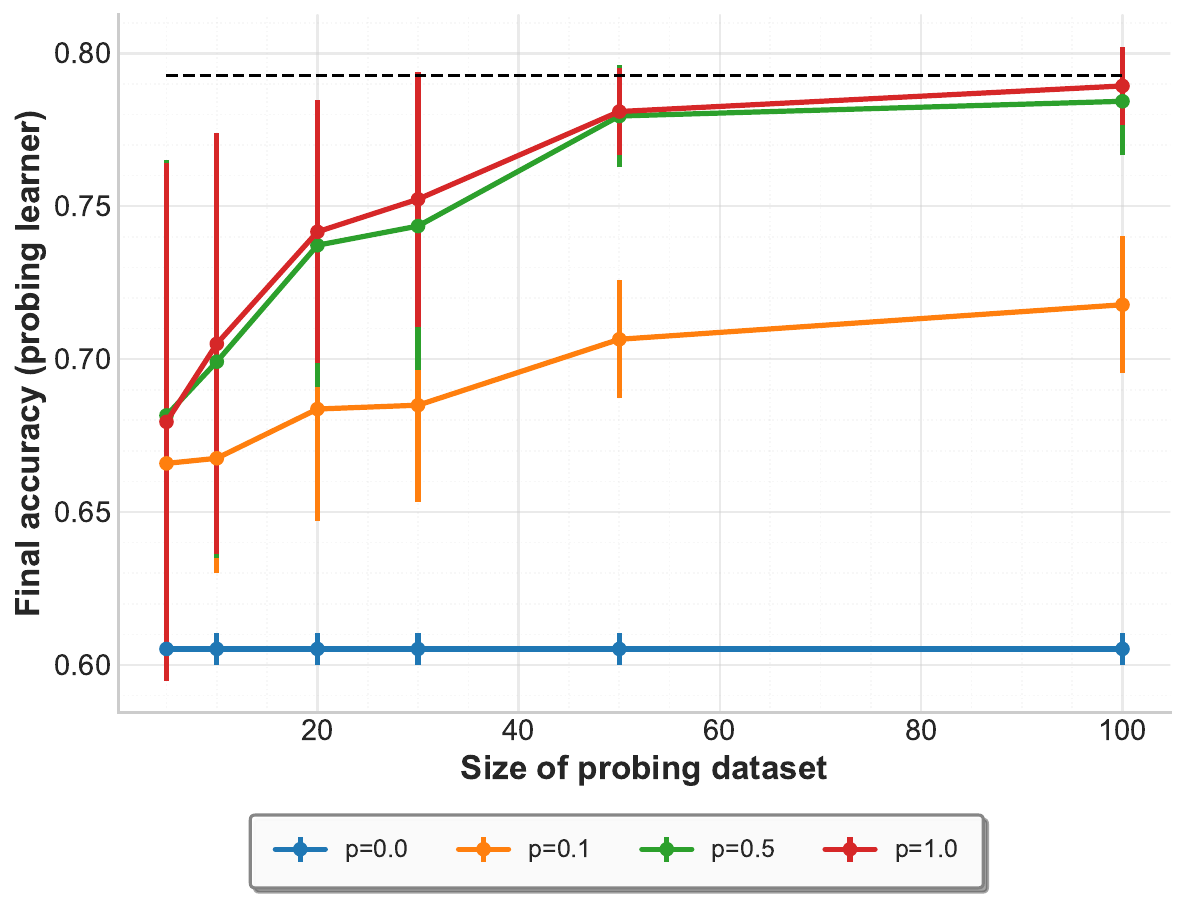}
    \caption{\textbf{Performance of probing learner on Census as a function of $n$.} Error bars show one standard deviation over 10 random seeds.}
    \label{fig:census_probe_vs_n}
\end{figure}

\begin{table}[t]
\centering
\small
\begin{tabular}{llccc}
\toprule
\textbf{Parameter} & \textbf{Description} & \textbf{Census} & \textbf{MovieLens} & \textbf{Amazon} \\
\midrule
$m$ & Number of learners & 5 & 5 & 9 \\
$T$ & Total rounds & 4000 & 4000 & 20000 \\
$\lambda$ & L2 regularization & $10^{-9}$ & $10^{-3}$ & $10^{-9}$ \\
$n$ & Offline probe dataset size & 100 & 1000 & 500 \\
\bottomrule
\end{tabular}
\caption{Hyperparameters by dataset. Shared values are merged across columns.}
\label{tab:hyperparameters}
\end{table}

\begin{figure}[t]
    \centering
    \includegraphics[width=\textwidth]{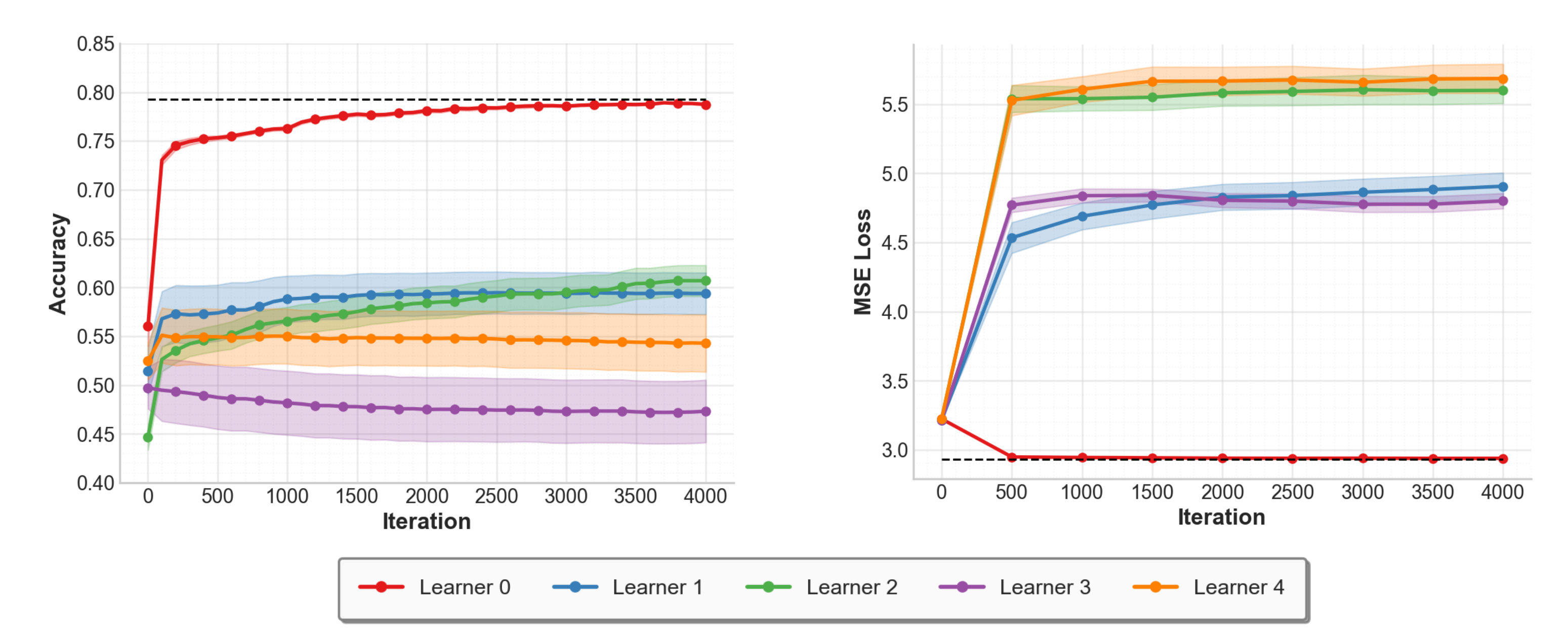}
    \caption{\textbf{MSGD full-population performance with random initialization
    (Market-leader scenario).}
     Left: Census test accuracy.
     Right: MovieLens test loss. Here $\tau = 0.3$.}
    \label{fig:bad_market_leader}
\end{figure}

\begin{figure}[t]
    \centering
    \includegraphics[width=\textwidth]{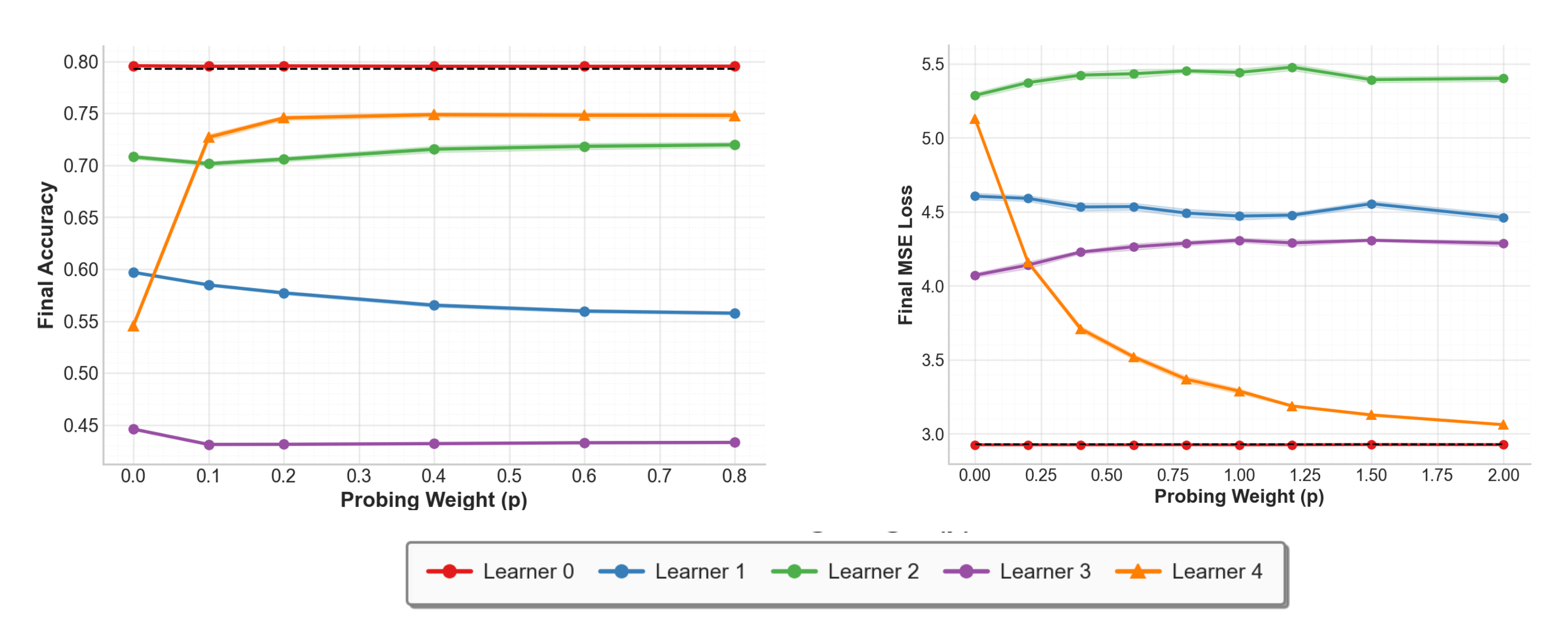}
    \caption{\textbf{Effect of probing on full-population performance (Market-leader scenario)} Left: Final accuracy vs probing weight $p$ on Census.
    Right: MovieLens final loss vs $p$. The triangle markers indicate Learner $4$
    probes the market leader, learner $0$. Here $\tau = 0.7$.}
    \label{fig:good_market_leader}
\end{figure}

\begin{figure}[t]
    \centering
    \includegraphics[width=\textwidth]{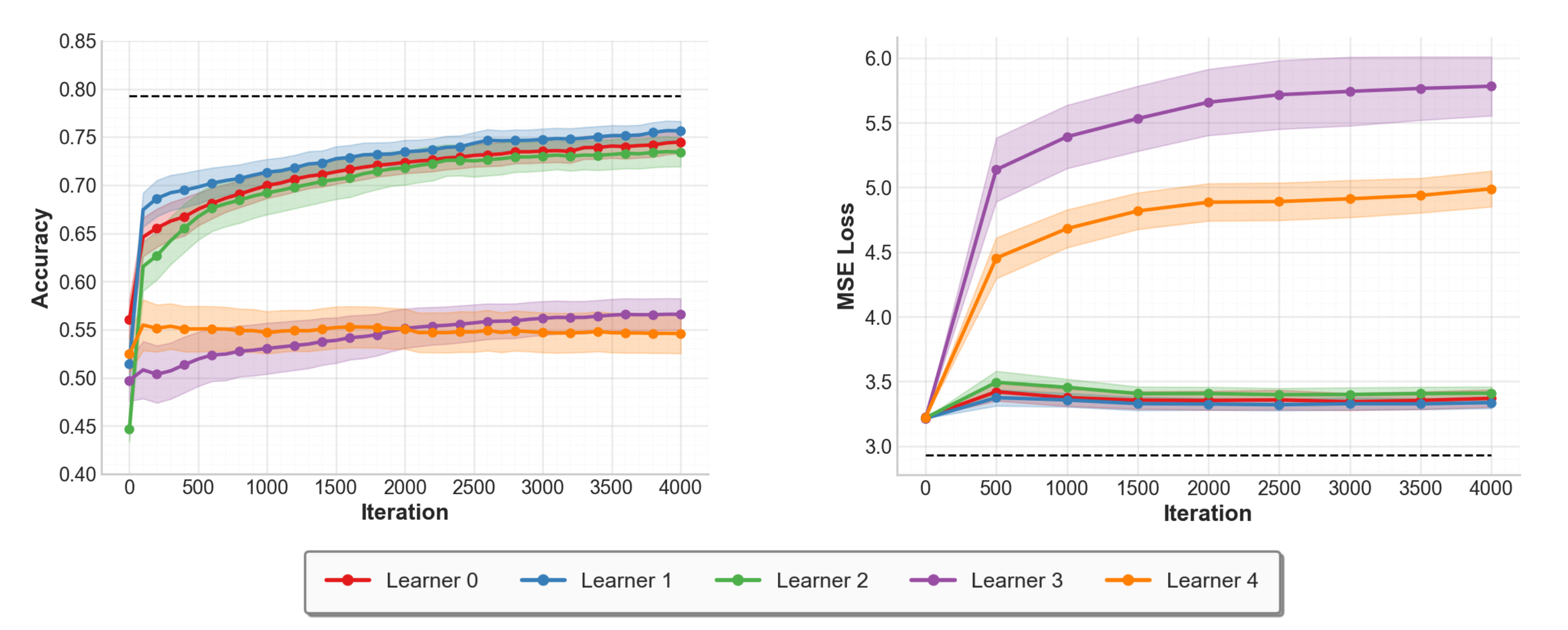}
    \caption{\textbf{MSGD full-population performance with random initialization
    (Majority good scenario).}
     Left: Census test accuracy. Right: MovieLens test loss. Here $\tau = 0.3$.}
    \label{fig:bad_half_majority}
\end{figure}

\begin{figure}[t]
    \centering
    \includegraphics[width=\textwidth]{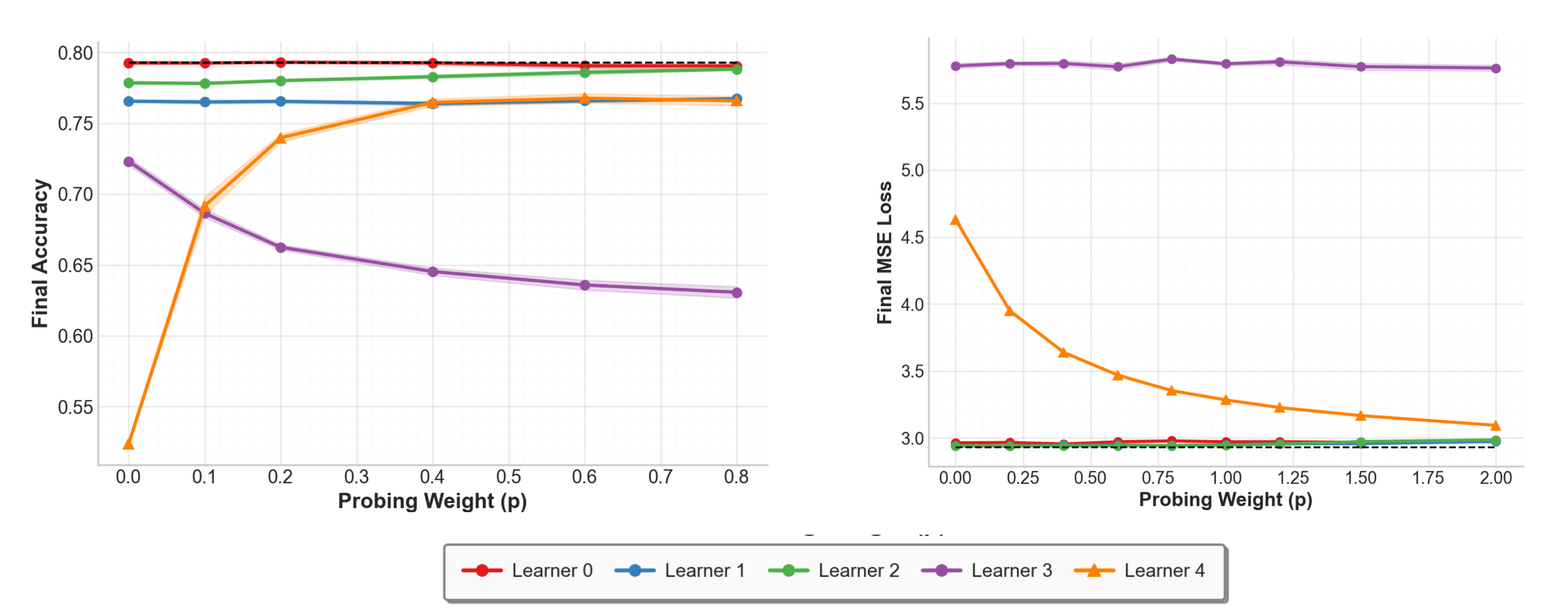}
    \caption{\textbf{Effect of probing on full-population performance (Majority Good Scenario)}
    Left: Final accuracy vs probing weight $p$ on Census.
    Right: MovieLens final loss vs $p$.
    Triangle markers indicate Learner $4$
    is probing via median aggregation over all peers. Here $\tau = 0.7$.}
    \label{fig:good_half_majority}
\end{figure}

\begin{figure*}[t]
    \centering
    \includegraphics[width=0.6\textwidth]{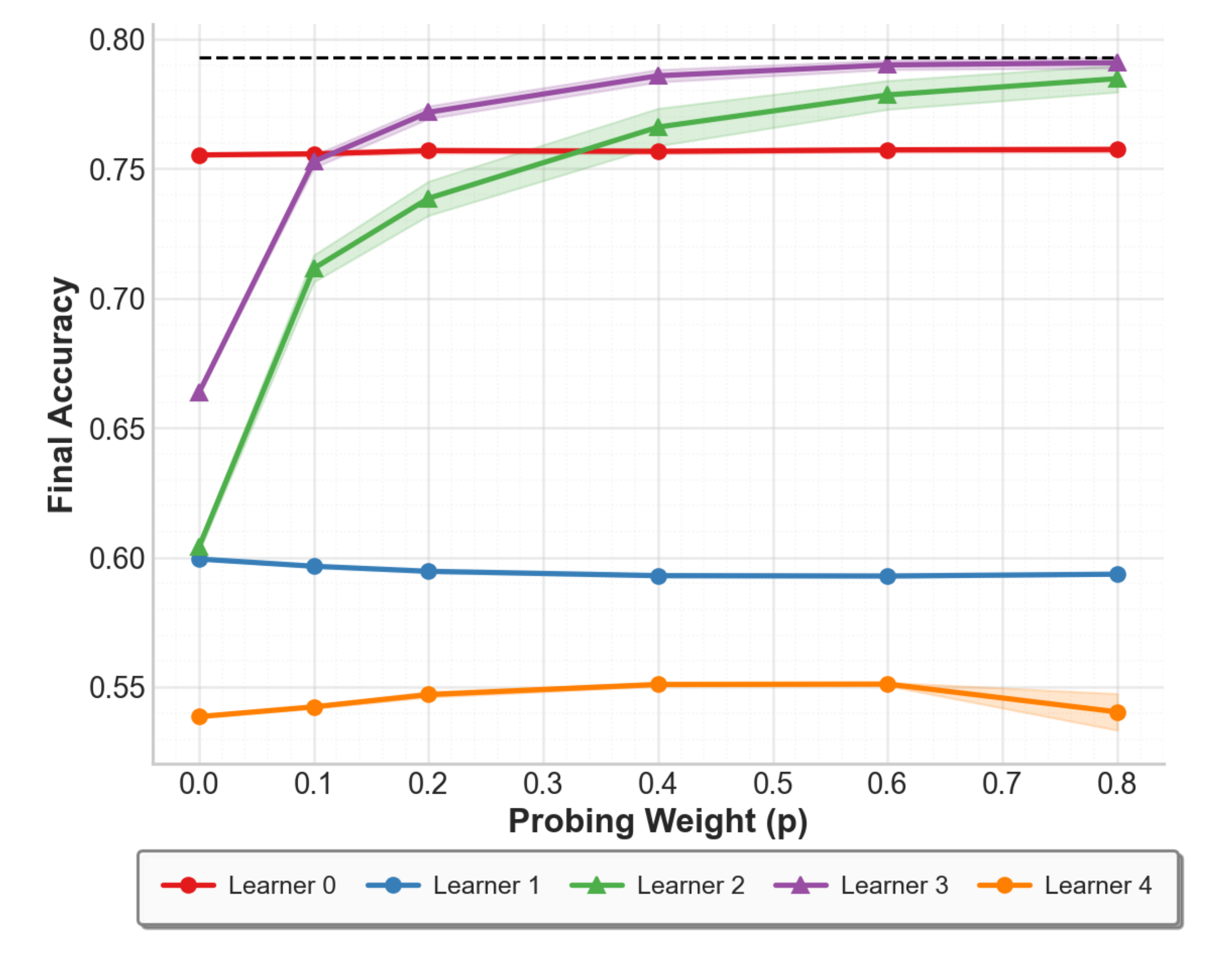}
    \caption{\textbf{Effect of probing on full-population performance when multiple learners probe (Preference-aware scenario).}
    Census final accuracy vs probing weight $p$.
    Triangle markers indicate the probing learners (Learners $2$ and $3$).
    The dashed black line denotes the full-data baseline.}
    \label{fig:multi_probe}
\end{figure*}

\end{document}